\documentclass[twoside]{article}

\usepackage[accepted]{aistats2024}
\usepackage{graphicx} 
\usepackage{amsmath}
\usepackage{mathtools}
\usepackage{hyperref}
\usepackage{algorithm}
\usepackage{algpseudocode}
\usepackage{amssymb}
\usepackage{amsthm}
\usepackage{color}
\usepackage{physics}
\usepackage{multirow}

\usepackage{listings}

\usepackage[round]{natbib}

\newtheorem{theorem}{Theorem}

\newtheorem{corollary}{Corollary}

\newcommand{\lafamilymethod}{Riemannian Laplace Approximation}
\newcommand{\lafamilyabbrv}{RLA}

\newcommand{\euclaabbrv}{ELA}
\newcommand{\bergaminabbrv}{RLA-B}
\newcommand{\logbergaminabbrv}{RLA-BLog}
\newcommand{\ourabbrv}{RLA-F}
\newcommand{\ourinvabbrv}{RLA-F}

\DeclareMathOperator{\btheta}{\boldsymbol{\theta}}
\DeclareMathOperator{\bTheta}{\boldsymbol{\Theta}}
\DeclareMathOperator{\bpsi}{\boldsymbol{\psi}}
\DeclareMathOperator{\bPsi}{\boldsymbol{\Psi}}
\DeclareMathOperator{\bJ}{\boldsymbol{J}}
\DeclareMathOperator{\bL}{\boldsymbol{L}}
\DeclareMathOperator{\bbv}{\boldsymbol{v}}
\DeclareMathOperator{\bbu}{\boldsymbol{u}}
\DeclareMathOperator{\bw}{\boldsymbol{w}}
\DeclareMathOperator{\bby}{\boldsymbol{y}}
\DeclareMathOperator{\bG}{\boldsymbol{G}}

\DeclareMathOperator{\bmu}{\boldsymbol{\mu}}

\DeclareMathOperator{\bGamma}{\boldsymbol{\Gamma}}
\DeclareMathOperator{\cM}{\mathcal{M}}
\DeclareMathOperator{\cP}{\mathcal{P}}
\DeclareMathOperator{\bSigma}{\boldsymbol{\Sigma}}
\DeclareMathOperator{\bzero}{\boldsymbol{0}}
\DeclareMathOperator{\ba}{\boldsymbol{a}}
\DeclareMathOperator{\N}{\mathcal{N}}
\DeclareMathOperator{\by}{\boldsymbol{y}}
\DeclareMathOperator{\bx}{\boldsymbol{x}}
\DeclareMathOperator{\bY}{\boldsymbol{Y}}
\DeclarePairedDelimiterX{\inp}[2]{\langle}{\rangle}{#1, #2}
\DeclareMathOperator{\bX}{\boldsymbol{X}}
\DeclareMathOperator{\bLambda}{\boldsymbol{\Lambda}}
\DeclareMathOperator{\bI}{\boldsymbol{I}}
\DeclareMathOperator{\bV}{\boldsymbol{V}}
\DeclareMathOperator{\bM}{\boldsymbol{M}}

\DeclareMathOperator{\E}{\mathbb{E}}
\DeclareMathOperator{\A}{\boldsymbol{A}}
\DeclareMathOperator{\bS}{\boldsymbol{S}}

\DeclareMathOperator{\Real}{\mathbb{R}}
\DeclareMathOperator{\bbf}{\boldsymbol{f}}
\DeclareMathOperator{\bphi}{\boldsymbol{\phi}}
\DeclareMathOperator{\bPhi}{\boldsymbol{\Phi}}
\DeclareMathOperator{\cN}{\mathcal{N}}
\DeclareMathOperator{\bones}{\boldsymbol{1}}
\DeclareMathOperator{\bb}{\boldsymbol{b}}
\DeclareMathOperator{\bgamma}{\boldsymbol{\gamma}}
\DeclareMathOperator{\bp}{\boldsymbol{p}}
\DeclareMathOperator{\bq}{\boldsymbol{q}}

\DeclareMathOperator{\B}{\boldsymbol{B}}

\begin{document}

%

%
\runningauthor{Hanlin Yu, Marcelo Hartmann, Bernardo Williams, Mark Girolami, Arto Klami}

\twocolumn[

\aistatstitle{Riemannian Laplace Approximation with the Fisher Metric}



\aistatsauthor{ Hanlin Yu \And Marcelo Hartmann \And Bernardo Williams }

\aistatsaddress{
Department of Computer Science\\ University of Helsinki
\And 
Department of Computer Science\\ University of Helsinki
\And
Department of Computer Science\\ University of Helsinki
}

\aistatsauthor{ Mark Girolami \And Arto Klami }

\aistatsaddress{
Department of Engineering\\ University of Cambridge\\ and The Alan Turing Institute
\And
Department of Computer Science\\ University of Helsinki
}

]

\begin{abstract}
  Laplace's method approximates a target density with a Gaussian distribution at its mode. It is computationally efficient and asymptotically exact for Bayesian inference due to the Bernstein-von Mises theorem, but for complex targets and finite-data posteriors it is often too crude an approximation. A recent generalization of the Laplace Approximation transforms the Gaussian approximation according to a chosen Riemannian geometry providing a richer approximation family, while still retaining computational efficiency. However, as shown here, its properties depend heavily on the chosen metric, indeed the metric adopted in previous work results in approximations that are overly narrow as well as being biased even at the limit of infinite data. We correct this shortcoming by developing the approximation family further, deriving two alternative variants that are exact at the limit of infinite data, extending the theoretical analysis of the method, and demonstrating practical improvements in a range of experiments.
\end{abstract}

\begin{figure}[t]
    \centering
    \includegraphics[width=0.45\columnwidth]{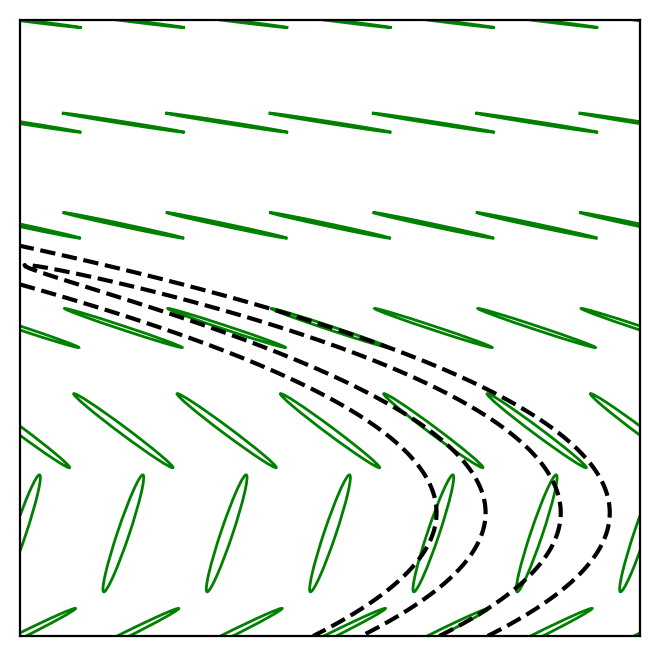}
    \includegraphics[width=0.45\columnwidth]{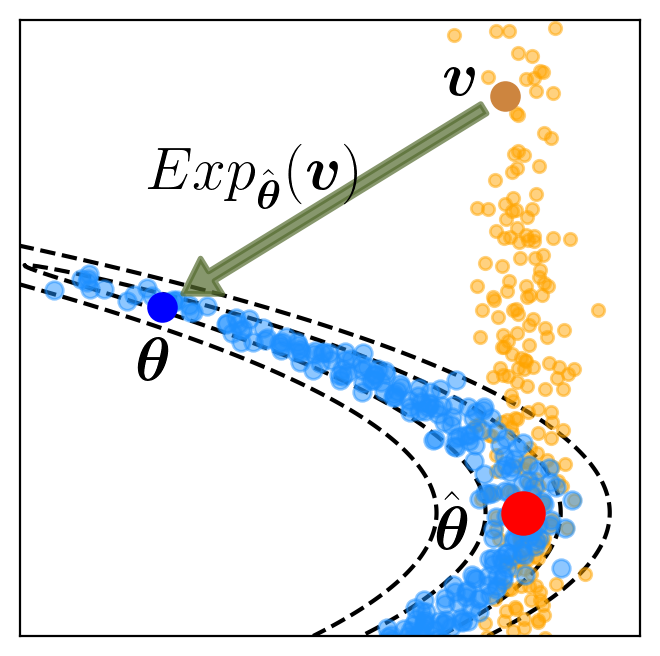}
    
    \caption{Left: The Fisher metric (green) captures the local curvature of the target density (black). Right: Samples from a Gaussian distribution (orange) are deterministically mapped using the Fisher metric to provide a flexible approximation (blue).}
    \label{fig:banana-geodesics}
\end{figure}

\section{INTRODUCTION}

Functional distributional approximations offer a computationally attractive alternative to Markov Chain Monte Carlo (MCMC) methods for statistical (posterior in the Bayesian setting) inference of potentially large models, building on variational inference (VI) \citep{Blei2017}, expectation propagation (EP) \citep{minka:2001}, or Laplace approximation (LA) \citep{Tierney1986, rue:2009}. Typically the methods need to make a compromise between computational efficiency and accuracy, and often the flexibility is achieved by non-linear transformations modelled as neural networks that need to be separately trained for a given model, as in amortized VI \citep{Margossian2023} or normalizing flows \citep{Papamakarios2021}. We focus on perhaps the oldest approximation strategy, the Laplace Approximation \citep{Tierney1986}, and study how to increase the flexibility of the approximation without resorting to trainable transformations. 

Laplace Approximation fits a Gaussian distribution at the mode of the target density and is used in machine learning from non-conjugate Gaussian processes \citep{Rasmussen} to deep learning \citep{Daxberger2021}. It is computationally efficient, only requiring the maximum a posteriori (MAP) estimate and the Hessian of the log-posterior that determine the mean and covariance of the approximation. Despite its simplicity, it is asymptotically exact at the limit of infinite data, due to the Bernstein-von Mises theorem stating that the posterior converges to a Gaussian when the probabilistic model has identifiable parameterisation \citep{Vandervaart1998}. Obviously, as has been shown many times in the literature, the disadvantage is that the approximation family is very limited for finite-data posteriors.

We seek to create an approximation that retains the advantages but increases the flexibility, and do so by relying on the tools of Riemannian geometry \citep{doCarmo1992}. We change the metric of the parameter space in such a manner that after this modification the approximation more accurately characterises the target. Concepts of Riemannian geometry have been broadly considered in the approximate inference literature, with Riemannian extensions of MCMC \citep{Girolami2011}, stochastic-gradient MCMC \citep{Patterson2013,Yu2023} and VI \citep{Frank2021}. Recently, \citet{Bergamin2023} introduced a Riemannian generalization of the Laplace Approximation. The core idea (illustrated conceptually in Figure~\ref{fig:banana-geodesics}, using the new metric from this paper) is to transform samples from a Gaussian distribution using numerical integrators to follow geodesic paths induced by a chosen geometry, which can be carried out in parallel. Some of the concepts required were already studied by \citet{Hauberg2018}, where Laplace approximation on a Riemannian manifold was considered.

We build on \citet{Bergamin2023}, expanding both theoretical and practical understanding of the approximation. They used the Riemannian metric \citet{Hartmann2022} proposed for Riemann Manifold MCMC, which is computationally attractive due to the efficient \emph{exponential map} for transforming the samples. However, as shown here, the approximation is not asymptotically exact and is also biased in practical cases. We resolve this theoretical and practical limitation in two alternative ways. We first show that the method of \citet{Bergamin2023} can be fixed by incorporating a \emph{logarithmic map} to correct the bias, resulting in an asymptotically exact but somewhat computationally heavy and unstable approximation. A more practical solution is to replace the metric itself. We explain how a metric based on the Fisher Information Matrix (FIM) gives an approximation that is exact for targets that are diffeomorphisms of a Gaussian, requires fewer function evaluations in numerical integration, and is superior in how it performs in a range of tasks.

\section{PRELIMINARIES}

\subsection{Laplace Approximation}

Let $\btheta \in \bTheta$ denote the $D$-dimensional vector in 
the parameter space $\bTheta$. Denote the log-posterior distribution given observed data set $\by$ $=$ $\{ y_n \}_{n = 1}^N$ as $\ell_{\by}(\btheta)$ $:=$ $\log \pi(\btheta|\by)$. Laplace's method expands a second-order Taylor series of $\ell_{\by}(\btheta)$ at the mode (MAP) where the first-order term disappears. This is equivalent to using a Gaussian approximation
\begin{equation*}
\pi(\btheta|\by) \approx \N \big( \btheta | \hat{\btheta}, -\nabla^{2}\ell_{\by}(\hat{\btheta})^{-1} \big),
\end{equation*}
where $\hat{\btheta}$ is the MAP estimate and $\nabla^2 \ell_{\by}(\hat{\btheta})$ is the Hessian matrix of the log-posterior evaluated at $\hat{\btheta}$. We refer to this as \textit{\euclaabbrv}, for Euclidean LA.

\subsection{Riemannian geometry} \label{sec:riemgeom}

\paragraph{Key concepts.}
A \emph{manifold} is a space more general than an Euclidean space but that locally acts like a Euclidean space. A \emph{differentiable manifold} is a manifold that can be represented with a system of coordinates such that it will allow us to extend notions of derivatives.
A \emph{Riemannian manifold} is a differentiable manifold endowed with a metric $g$ and this metric will, intrinsically, generalise the local notions of angles, distances and derivatives on the manifold.

The metric $g$ is formally defined as a function $g_{\btheta} : T_{\btheta}\bTheta \times T_{\btheta}\bTheta \rightarrow \mathbb{R} $ where $T_{\btheta}\bTheta$ denotes the \emph{tangent space} at $\btheta$,
given by $T_{\btheta}\bTheta = \{ \tfrac{\mathrm{d}}{\mathrm{d}t} \boldsymbol{c}(t)|_{t = 0} = \bbv \textrm{, and} \ \boldsymbol{c}(0) = \btheta \}$. The $\bbv$s are known as the \emph{tangent vector}s. The function $g$ is like the dot-product in Euclidean spaces, but takes a more general form as a mapping $(\bbv, \bbu) \mapsto \inp{\bbv}{\bG(\btheta)\bbu}$ where $\bG(\btheta)$ is a symmetric positive-definite matrix called \emph{metric tensor} that collects the coefficients of the metric as $g\big( \tfrac{\partial}{\partial_{\theta_i}}, \tfrac{\partial}{\partial_{\theta_j}} \big) = G(\btheta)_{ij}$. We are free to choose the metric $g$ to encode desired properties of the manifold.

\paragraph{Gradients and Hessians in Riemannian geometry.}
In a given basis, the \emph{Riemannian gradient} of $\ell_{\by}(\btheta)$ is given by $\textrm{grad} \ \ell_{\by}(\btheta) = \bG(\btheta)^{-1} \nabla \ell_{\by}(\btheta)$, and the \emph{Riemannian Hessian} is $\textrm{Hess} \ \ell_{\by}(\btheta)[\bbv] = \nabla_{\bbv} \textrm{grad} \ \ell_{\by}(\btheta) = \bG^{-1}\left(\nabla^2 \ell_{\by}(\btheta) - \sum_k \bGamma(\btheta)^k \nabla \ell_{\by} (\btheta)_k\right) \bbv$. 
Here $ \bGamma(\btheta)^k = \{ \bGamma(\btheta)^k \}_{i, j}  $ are the \emph{Christoffel symbols} of the second kind in matrix forms. 
They can also be expressed using the metric if the connection is chosen to be the Levi-Civita connection \citep{doCarmo1992} as
\begin{equation*}
    \Gamma^{k}_{ij} = \frac{1}{2}G^{kl}\left(\partial_{i}G_{jl} + \partial_{j}G_{il} - \partial_{l}G_{ij} \right),
\end{equation*}
where $\partial_i = \tfrac{\partial}{\partial {\theta_i}}$. Following differential geometric convention, we use the Einstein summation. Note that in this convention $G^{kl}$ refers to elements of the \emph{inverse} of the matrix $\bG(\btheta)$.

\paragraph{Geodesic paths.}
Another important notion for the development of this work is the \emph{exponential map} and its inverse known as the \emph{logarithmic map}. Formally, the exponential-map is defined as a function $\mathrm{Exp}_{\btheta} : T_{\btheta}\bTheta \rightarrow \bTheta$ that takes a vector on the tangent space of $\bTheta$ and maps it back onto $\bTheta$ \citep{Boumal2023}.

Intuitively, the exponential map is a function that returns the final position after following the shortest path at $\btheta$ in the direction of $\bbv$ for unit time. These paths are called geodesics and they solve the geodesic equation given by \citep{doCarmo1992}
\begin{equation}
    \dv[2]{\theta^{k}}{t} + \dv{\theta^{i}}{t}\dv{\theta^{j}}{t}\Gamma^{k}_{ij} = 0,\; k=1,2,\dots ,D,
\label{eq:geodesics}
\end{equation}
where $\dv{\theta^{i}}{t}=v^{i}$.
The inverse $\mathrm{Log}_{\btheta} : \bTheta \rightarrow T_{\btheta}\bTheta$ of the above is defined as $\mathrm{Log}_{\btheta}(\tilde\btheta) = \arg\min_{\bbv\in T_{\btheta}\bTheta}g_{\btheta}(\bbv, \bbv)$ subject to $\mathrm{Exp}_{\btheta}(\bbv) = \tilde\btheta$ \citep[see][]{Boumal2023}. It gives the $\bbv$ for which the exponential map
starting from $\btheta$
maps to $\tilde\btheta$.

\section{RIEMANNIAN LA}

Here we outline a general recipe for \lafamilymethod\ (\lafamilyabbrv). The formulation matches the one recently proposed by \citet{Bergamin2023}, though here presented in a more general form where e.g. the choice of the metric is still left open. We will later show how the metric influences both the theoretical and practical properties of the method.

The Bernstein-von Mises Theorem states that, under certain regularity conditions, the posterior at the limit of infinite data becomes Gaussian with inverse of the FIM as the covariance \citep{Vandervaart1998}. This makes \euclaabbrv\ exact at the limit. We want a Riemannian generalization of \euclaabbrv\ to retain this property but to adapt for the local curvature of the target for non-Gaussian targets and finite-data posteriors.

\subsection{Principle and algorithm}

The Riemannian extension of LA is based on a Taylor series expansion on the tangent space $T_{\btheta}\bTheta$ at a given point $\btheta \in \bTheta$.
\citet{Bergamin2023} used here the particular tangent space $T_{\hat{\btheta}}\bTheta$, where $\hat{\btheta}$ is the MAP estimate. 
When the parameter space of the model $\bTheta$ is endowed with a Riemannian geometry, the Taylor approximation of the function $\bTheta \ni \btheta \mapsto \ell_{\by}(\btheta)$ at $\hat{\btheta}$ assumes a more general formulation \citep[see][]{Boumal2023, Bergamin2023} 
\begin{align}\label{eq:riemlap}
    \ell_{\by}\big(\mathrm{Exp}_{\hat{\btheta}}(\bbv)\big) & \approx \ell_{\by}(\hat{\btheta}) + g_{\hat{\btheta}}\big(\textrm{grad} \ \ell_{\by}( \hat{\btheta}) , \bbv \big) \\ 
    & + \frac{1}{2} g_{\hat{\btheta}} \big(\textrm{Hess} \ \ell_{\by}(\hat{\btheta})[\bbv], \bbv \big).\nonumber
\end{align}
\begin{algorithm}[t]
    \caption{\bergaminabbrv, producing $N$ samples $\btheta^{[n]}$.}
        \begin{algorithmic}
            \State Obtain MAP estimate $\hat{\btheta}$
            \State Set $\bSigma = (-\nabla^{2}\ell_{\by}(\hat{\btheta}))^{-1}$
                \For{$n \leftarrow 1, \dots, N$}
                    \State Obtain velocity $\bbv^{[n]} \sim \mathcal{N} (\bzero, \bSigma)$
                    \State $\btheta^{[n]} = \mathrm{Exp}_{\hat{\btheta}}(\bbv^{[n]})$
                \EndFor
        \end{algorithmic}
\label{alg::BergaminLap}
\end{algorithm}

Generally we cannot evaluate the approximation analytically, but can obtain samples from it following Algorithm~\ref{alg::BergaminLap} that comprises of three distinct steps, where the last two are repeated for each sample (and can be computed in parallel):
\begin{enumerate}
    \itemsep0.2em 
    \item Find $\hat{\btheta}$ where the approximation is placed;
    \item Sample initial velocity $\bbv\in T_{\hat{\btheta}}\bTheta$ from a Gaussian distribution with suitably chosen covariance $\bSigma$;
    \item Solve exponential map at $\hat{\btheta}$ with velocity $\bbv$ to obtain the sample $\btheta$.
\end{enumerate}

This constructs a wrapped Gaussian \citep{DeBortoli2022} on the manifold. 
The primary design choice is the metric $g$ that encodes the intrinsic geometry of the problem, but as will be explained later there is freedom also in the choices for the first two steps.
Once $g$ is given, the geodesic equation \eqref{eq:geodesics} can be formulated as an initial value problem of an ordinary system of differential equations (ODE) and solved relatively efficiently with numerical integration.

\paragraph{Euclidean Laplace Approximation.} Consider the special case of $g$ being the Euclidean inner product. Then $\bG(\btheta) = \bI_D$ and the exponential map becomes $\mathrm{Exp}_{\hat{\btheta}}(\bbv)$ $=$ $\hat{\btheta} + \bbv$, a unit step in the direction of the velocity $\bbv$. If velocities are sampled using $\bSigma = -(\nabla^2 \ell_{\by}(\hat{\btheta}))^{-1}$ we get the classical \euclaabbrv.

\paragraph{Method of \citet{Bergamin2023}.}
In Equation~\eqref{eq:riemlap}, \citet{Bergamin2023} sets the metric tensor as $\bG(\btheta) = \bI_D + \nabla \ell_{\by}(\btheta) \nabla \ell_{\by}(\btheta)^\top$, which is a special case of the metric \citet{Hartmann2022} used for geometric MCMC and coined as the \emph{Monge metric}.
This metric is defined solely based on gradient information, and is a way to introduce curvature information using the first order derivative of the target density function. 
It enables fast computation of Christoffel symbols for the exponential map.
We denote this specific instance of \lafamilyabbrv\ as 
\textit{\bergaminabbrv} after the initial letter of the first author, to distinguish the specific variant from the general family.
Even though \bergaminabbrv \ was shown to work well in a range of tasks and scales up for large problems, we will later observe that it does not always improve over \euclaabbrv. 

More importantly, the method is not exact even for Gaussian targets and hence not for infinite-data posteriors, but instead underestimates the uncertainty. This is seen by analysis of the exponential map.
The solutions of the geodesic equation keep the norm
$\norm{\bbv(t)}^2_{\bG(\btheta(t))}$ constant \citep{Lee2018}, which in this metric expands to $\norm{\bbv(t)}^{2}+\langle \bbv(t) ,\nabla \ell_{\by}(\btheta(t)) \rangle ^{2}$.
At the MAP estimate we have $\nabla\ell_{\bby}(\hat{\btheta})=0$ and hence $\bG(\hat{\btheta}) = \bI_D$. Denote the norm of the initial velocity by $\norm{\bbv(0)}$. For every other $t$ we have $\norm{\bbv(t)} < \norm{\bbv(0)}$ because $\langle \bbv(t) ,\nabla \ell_{\by}(\btheta(t)) \rangle ^{2} > 0$ except in the rare cases when the gradient is null or exactly orthogonal to velocity. Since the exponential map integrates for unit time, the distance of each sample to $\hat{\btheta}$ (in the Euclidean sense) is upper-bounded by and almost always smaller than for \euclaabbrv, for which the distance is $\norm{\bbv(0)}$.
Figure~\ref{fig:e-and-b} shows that \bergaminabbrv \ does not recover even a two dimensional isotropic Gaussian. Moreover, as shown in Figure~\ref{fig:distances}, the bias of \bergaminabbrv \ becomes larger as $D$ grows from $1$ to $10$ for isotropic Gaussians; see Section~\ref{sec:sup-bias-rlab} in the Supplement for further remarks.

\begin{figure}[t]
    \centering
    \begin{tabular}{cc}
        \includegraphics[width=0.45\columnwidth]{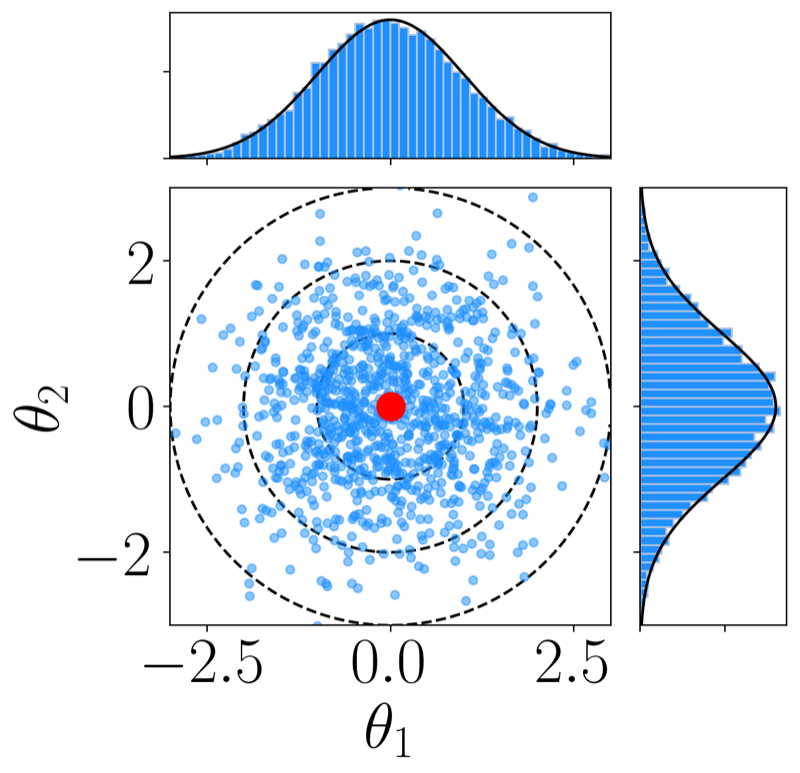} & \includegraphics[width=0.45\columnwidth]{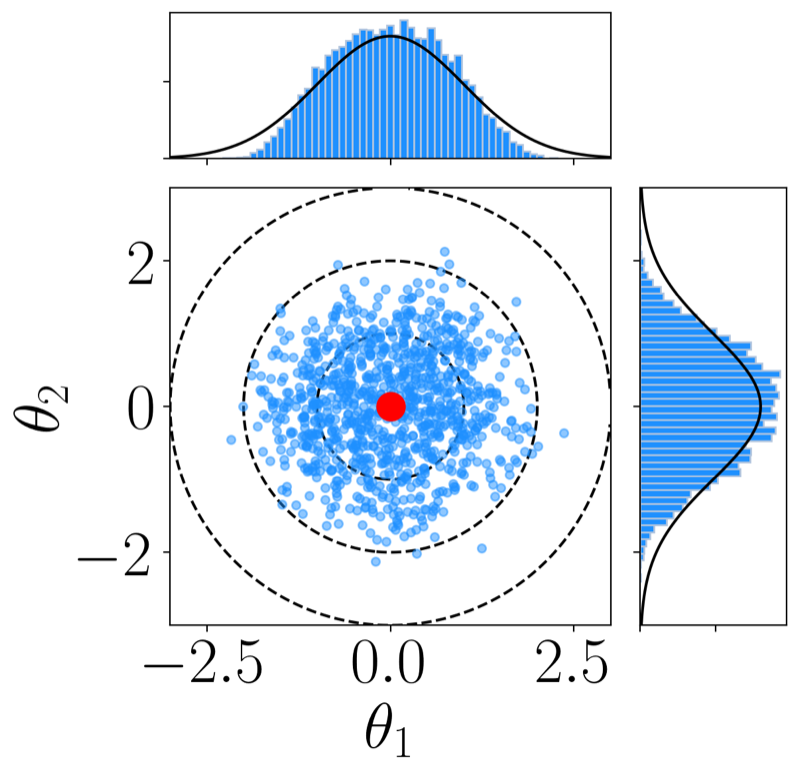}
    \end{tabular}
    
    \caption{Euclidean LA (left) is exact for a Gaussian target, but  \bergaminabbrv\ (right) is biased. Lines are the true contours and marginals, with samples and histograms characterizing the approximation.}
    \label{fig:e-and-b}
\end{figure}

\begin{figure}
    \centering
    \includegraphics[width=0.3\textwidth]{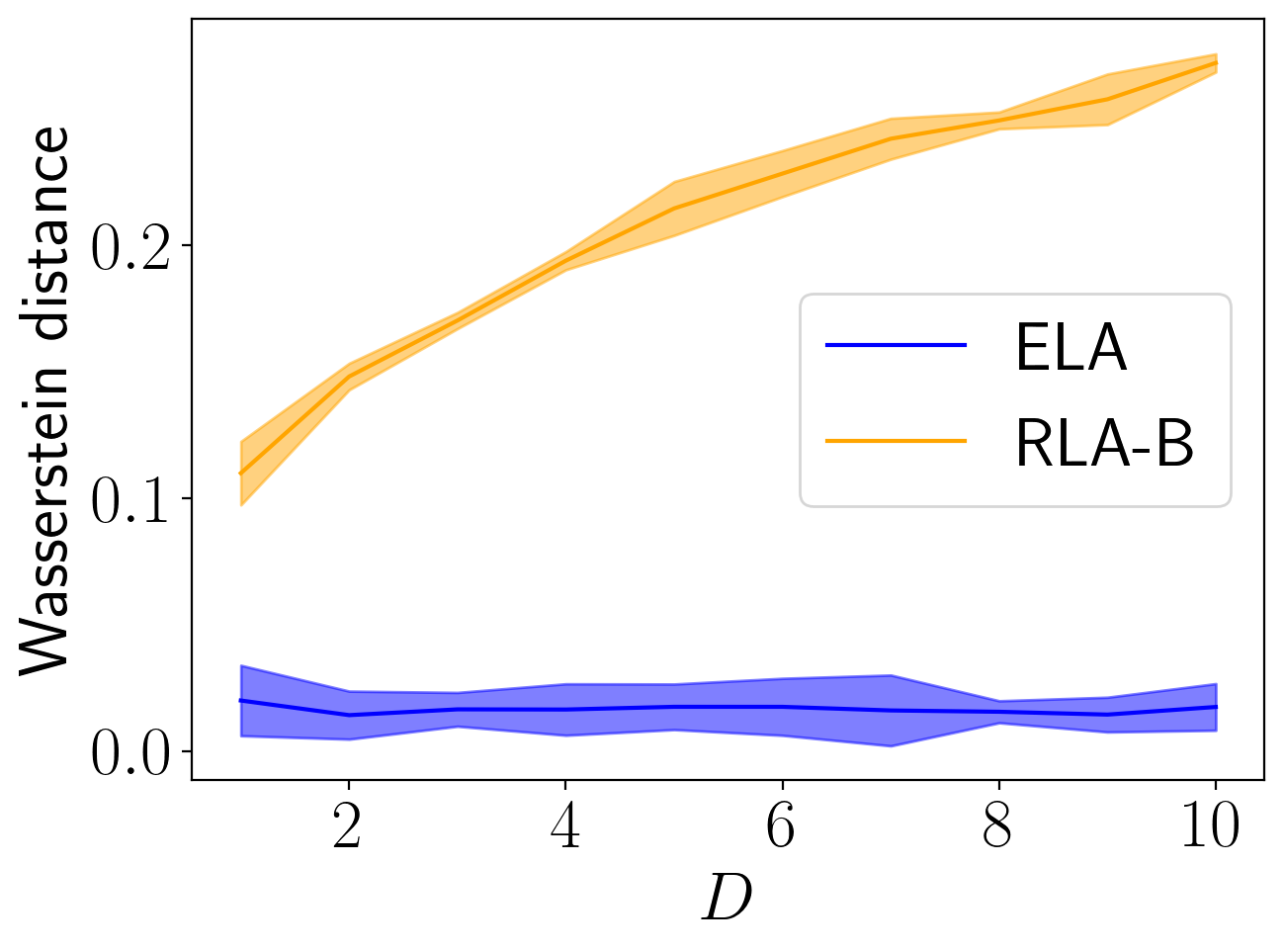}
    \caption{Wasserstein distances from approximate samples to true samples for isotropic Gaussians of varying $D$, computed for the first dimension. The lines are means and $\text{means} \pm 2.0$ times $\text{standard deviations (stds)}$ computed over $5$ runs.}
    \label{fig:distances}
\end{figure}

\section{IMPROVING RIEMANNIAN LA}

Next we describe two alternative asymptotically valid \lafamilyabbrv \ methods, and discuss the choice of the starting point and covariance of velocities.

\subsection{\bergaminabbrv \ with logarithmic map}
\label{sec:log-riem}

We can make \bergaminabbrv \ exact for Gaussian targets by employing a logarithmic map to solve for the initial velocity $\bbv$ such that the result of the exponential map retains the correct Euclidean distance to $\hat{\btheta}$.

Consider two manifolds $\cP$ and $\cM$. $\cP = (\bTheta, g_{\cP})$ has the metric tensor 
\begin{align*}
    \bG_{\cP}(\btheta) & = I + \nabla \log \mathcal{N}(\btheta| \hat{\btheta}, \bSigma) \nabla \log \mathcal{N}(\btheta| \hat{\btheta}, \bSigma)^\top
\end{align*}
where
$\bSigma = (-\nabla^2 \ell_{\by}(\hat{\btheta}))^{-1}$, and $\cM = (\bTheta, g_{\cM})$ the metric tensor
$$
\bG_{\cM}(\btheta) = I + \nabla \ell_{\by}(\btheta) \nabla \ell_{\by}(\btheta)^\top.
$$
$T_{\hat \btheta}\cP$ and $T_{\hat \btheta}\cM$
coincide at $\hat{\btheta}$ as the metric tensors are identical, and hence we have a diffeomorphism
$$
\textrm{Exp}_{\cM, \hat{\btheta}} \circ \textrm{Log}_{\cP, \hat{\btheta}} : \cP  \rightarrow \cM.
$$
Suppose we sample from a Gaussian distribution as in ELA and consider the samples as points on manifold $\cP$. If $\cP$ is equal to $\cM$
the transformation given by $\textrm{Exp}_{\cM, \hat{\btheta}} \circ \textrm{Log}_{\cP, \hat{\btheta}}$ results in exact samples from the target associated with $\cM$.
For non-Gaussian targets 
the method is not exact, but we observed it can still help reducing the underestimation tendency of \bergaminabbrv. 

In Algorithm~\ref{alg::RiemLap} the exponential map is computed as before on manifold $\cM$, and the logarithmic map corresponds to a boundary value problem of the geodesic ODE on manifold $\cP$. It is solved numerically and benefits again from efficient computation of Christoffel symbols. The logarithmic map is in general more costly than the exponential map \citep{Arvanitidis2016}, but we observed it can help reducing the number of function evaluations during the exponential map and the overall cost is not necessarily higher.
We refer to this variant as \logbergaminabbrv.
\begin{algorithm}[t]
    \caption{\logbergaminabbrv, producing $N$ samples $\btheta^{[n]}$.
    The manifold $\cP$ uses the metric induced by Gaussian $\mathcal{N}(\tilde\btheta|\hat\btheta, \bSigma)$ and the manifold $\cM$  the metric induced by the target $\pi(\btheta|\bby)$.}
    \begin{algorithmic}
        \State Obtain MAP estimate $\hat{\btheta}$
        \State $\bSigma = (-\nabla^{2}\ell_{\by}(\hat{\btheta}))^{-1}$
        \For{$n \leftarrow 1\dots N$}
            \State Obtain base sample $\bar{\btheta}^{[n]} \sim \mathcal{N}
            (\bzero,\bSigma)$
            \State $\bbv^{[n]} = \text{Log}_{\cP, \hat{\btheta}}(\bar{\btheta}^{[n]})$
            \State $\btheta^{[n]} = \text{Exp}_{\cM, \hat{\btheta}}(\bbv^{[n]})$
        \EndFor
\end{algorithmic}
\label{alg::RiemLap}
\end{algorithm}

\subsection{Fisher Information Matrix as metric}
\label{sec:fim-riem}
 
The choices of \citet{Bergamin2023} were motivated by computational arguments and relied on extrinsic notions in differential geometry, whereas we turn the attention to intrinsic notions and the natural Riemannian metric tensor for probabilistic machine learning.

\subsection{Fisher metric}

From a statistical point of view the inverse of the FIM is the lower bound of the variance of unbiased estimators. At a first glance, it may not tell us much about possible underlying geometries on $\bTheta$. However, FIM transforms like a 2-covariant tensor for regular probabilistic models as the expected score function is null and it is symmetric and positive-definite by definition \citep[see][]{leh:2003, schervish:2012}. Therefore the FIM can act as the coefficients $\bG(\btheta)$ for the metric $g$ and the pair $(\bTheta, g)$ defines an abstract Riemannian manifold. 

We propose using a metric tensor
\begin{equation} 
    \bG(\btheta) 
    = \E_{\bY|\btheta}\left[ - \nabla^2 \log \pi(\bY | \btheta)  \right]- \nabla^2 \log \pi(\btheta),
 \label{eq:fisher-metric}
\end{equation}
which is the FIM plus the negative Hessian of the log-prior, accounting for both the likelihood and the prior.
This formulation has been used by \citet{Girolami2011} and \citet{Lan2015} as the metric for Riemann Manifold MCMC algorithms and by \citet{Hartmann2019} 
in natural gradient descent for Gaussian processes with non log-concave likelihoods.
For simplicity, we call it the \emph{Fisher metric} and the  resulting approximation \ourabbrv.

The same form can also be used as $\bSigma^{-1}$ when sampling the velocities $\bbv$, instead of the negative Hessian as in \euclaabbrv\ and \bergaminabbrv. The negative Hessian is not guaranteed to be positive definite (e.g. for a neural network), which causes numerical problems, whereas \eqref{eq:fisher-metric} is when the negative Hessian of the log prior is positive definite, e.g. when the prior is Gaussian. In addition, we will later show that the Fisher metric can coincide with the negative Hessian for a specific class of targets (Theorem~\ref{thm:invariant-exact}).

\subsubsection{Hausdorff MAP}
\label{sec:hausdorff-map}

The classical MAP estimate is not invariant under reparameterizations in Euclidean geometry \citep[see][]{ian:2005}. From the Riemannian perspective, a natural invariant alternative is the maximum value of the posterior density under the Hausdorff measure. When $\bTheta$ is endowed with $g$, the probability density function on the manifold is given by
\begin{equation*}
    \pi^{\bG}(\btheta) = \frac{\pi(\btheta)}{\sqrt{ \det \bG(\btheta)}};
\end{equation*}
see Theorem 3.2.5 by \citet{fed:1969}.

If we sample the velocities $\bbv$ using the Fisher metric, we can justify the usage of the Hausdorff MAP as a reparameterization (from the differential geometric viewpoint), where the Fisher is locally identity around the MAP. This makes the gradients zero and consequently the second-order Taylor series still corresponds to a Gaussian; see Section~\ref{sec:hausdorffdetails} in the Supplement for further discussions. We will later show that both Hausdorff and Euclidean MAP can provide good approximations.

\subsubsection{Theoretical basis}
\label{sec:theoretical-basis}

This variant is asymptotically exact for Gaussian targets, including posterior distributions of certain identifiable models at the limit of large data, as determined by the following theorems (see Section~\ref{sec:proofs} in the Supplement for proofs). Note that \citet{Knollmüller2020} also showed conclusions similar to Theorem~\ref{thm:fixed-metric} in the context of VI.

\begin{theorem}
    For Gaussian (or uniform) prior and Gaussian likelihood with fixed covariance, the Fisher metric is constant.
\label{thm:fixed-metric}
\end{theorem}

\begin{theorem}
When the posterior becomes a Gaussian in the limit of infinite data, \ourabbrv\ is exact when the likelihood has a parameterization where the observed Fisher coincides with the expected Fisher, e.g. for distributions in the exponential family.
\label{thm:converge-metric}
\end{theorem}

The key reasoning behind the above theorems is that, for constant metric (or for one that converges to a constant one), the geodesics become straight lines in the Euclidean sense. Consequently the approximation
becomes identical to \euclaabbrv \ and hence exact for the target that is (or converges to) a Gaussian.

We can also make a stronger statement (proof in Supplement) for specific types of targets and priors:
\begin{theorem}
    With an invariant prior, e.g. Jeffreys prior, 
    \ourinvabbrv\ with Hausdorff MAP is exact for probabilistic models whose target distributions are diffeomorphic with Gaussians, for which the negative Hessian at the Hausdorff MAP coincides with the Fisher metric.
\label{thm:invariant-exact}
\end{theorem}
Denote a diffeomorphic transformation $\btheta = \bphi(\bpsi)$ for a Gaussian distribution $\mathcal{N}(\bpsi|\bmu, \bS)$.
Since Gaussians are symmetric, observing how Riemannian metrics transform, define a metric for the space of $\btheta$ as 
\begin{equation*}
    \bG_{\bPsi} = \left(\pdv{\btheta}{\bpsi}\right)^{\top}\bG_{\bTheta}\pdv{\btheta}{\bpsi}.
\end{equation*}
With $(\bG_{\bPsi})_{i, j}$ $=$ $\E_{\bpsi}\big(-\partial^2_{\mu_i, \mu_j} \log \mathcal{N}(\bpsi| \bmu, \bS)\big) = \bS^{-1}_{i,j}$ inspired by the Fisher metric, we have
\begin{equation*}
    \bG_{\bTheta} = \left(\pdv{\btheta}{\bpsi}\right)^{-\top}\bS^{-1}\left(\pdv{\btheta}{\bpsi}\right)^{-1}.
\end{equation*}

\paragraph{Example.}
Consider the Squiggle distribution, used previously e.g. for evaluation of MCMC samplers \citep{Hartmann2022},
\begin{align*}
    \pi(\theta_{1},\theta_{2}|\bmu,\bS) &= \N(\bphi^{-1}_a (\theta_{1},\theta_{2})|\bmu ,\bS), \\
    \bphi^{-1}_a (\theta_{1},\theta_{2}) &= (\theta_{1}, \theta_{2}+\sin(a\theta_{1})).
\end{align*}
This can be seen as a probability density as
\begin{align*}
    \pi(\bmu|\btheta) & = \N(\bmu | \bphi^{-1}_a(\btheta), \bS) \ \textrm{and} \ \pi_J(\btheta) = \sqrt{\det \bG_{\bTheta}}
\end{align*}
where $\pi_J$ denotes the Jeffreys prior. 

The Jacobian to form the Fisher metric is given by
\begin{equation*}
\left(\pdv{\btheta}{\bpsi_a}\right)^{-1} = \begin{bmatrix}
    1 & 0 \\
    a \cos(a\theta_{1}) & 1
\end{bmatrix}
\end{equation*}
and for $\bmu=\bzero$ used here the Hausdorff MAP is also $\bzero$. As stated by Theorem~\ref{thm:invariant-exact}, \ourinvabbrv \ is exact for this target due to the diffeomorphism detailed above. Figure~\ref{fig:invariant-squiggle} empirically validates this.
\begin{figure}[t]
    \centering
    \begin{tabular}{cc}
        \includegraphics[width=0.45\columnwidth]{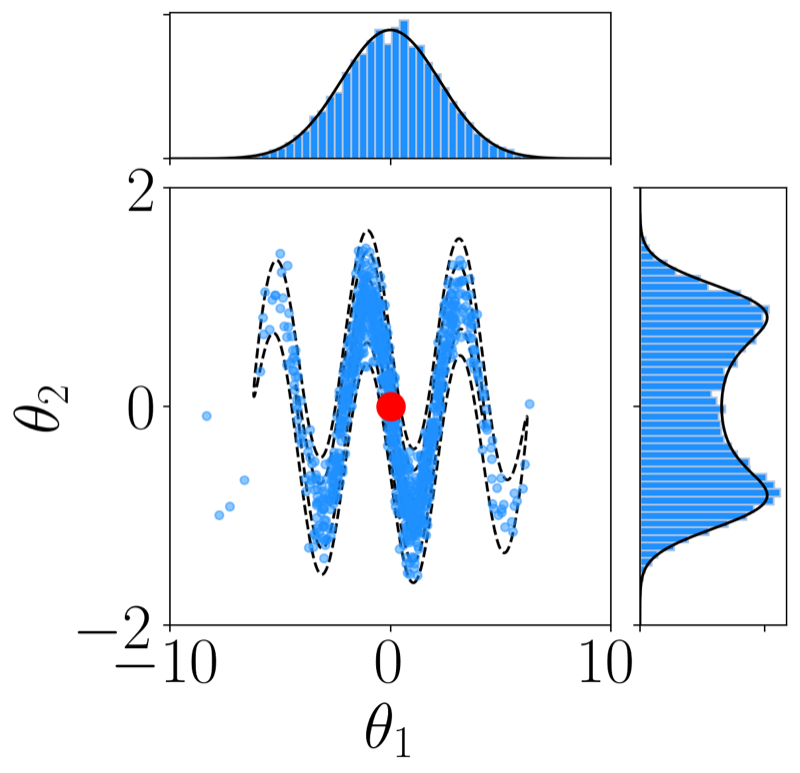} &
        \includegraphics[width=0.45\columnwidth]{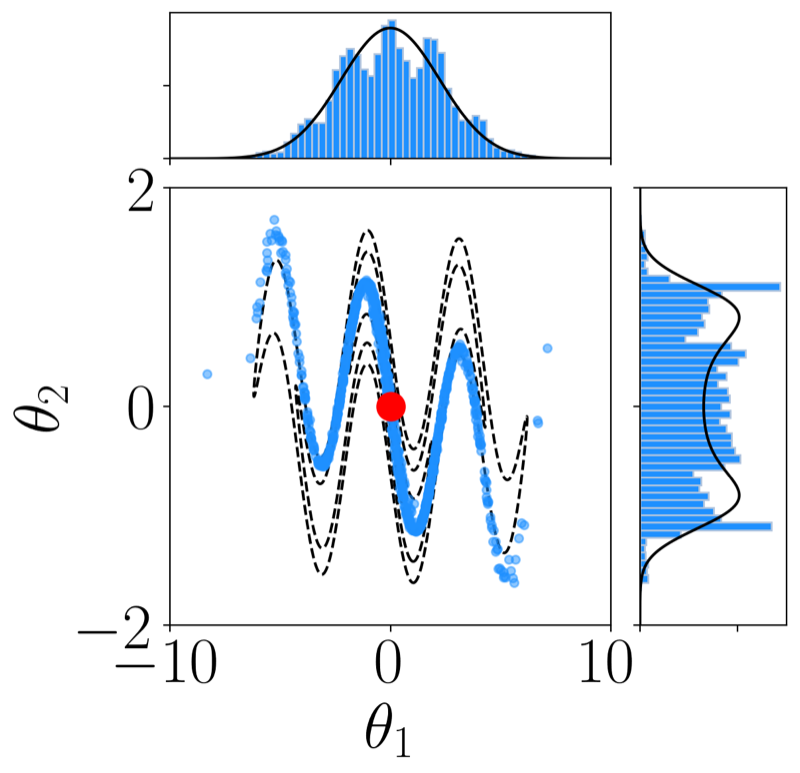}      
    \end{tabular}
    \caption{\ourinvabbrv \ (left) is exact (already for finite data) for the squiggle distribution with complex shape due to diffeomorphism with a Gaussian, whereas \bergaminabbrv \ (right) is too narrow everywhere and generally biased.}
    \label{fig:invariant-squiggle}
\end{figure}

\subsubsection{Computation of the Fisher metric}
\label{sec:comp-fisher}

\paragraph{FIM for neural networks (NNs).}
The Fisher metric is known for many probabilistic models of basic form, i.e. the likelihood part \citep[see][]{yang:96, cas:2001}. For instance, for exponential family the FIM always has closed-form expression \citep{cas:2001}. Building on these, we can easily form the Fisher metric for more complex models.
Let $\phi$ be the parameters of a model $\pi(y|\phi)$ in its basic form in a $P$-dimensional  parameter space $\Phi$. Denote the FIM for $\Phi$ as $\bG_{\bPhi}(\cdot)$.
Consider inputs $\{ \bx_n \}_{n = 1}^N$ with associated observations $\by = \{y_n \}_{n = 1}^N$ and a NN that maps $\mathcal{X}$ $\ni$ $\bx_n$ $\overset{f_{\btheta}}{\mapsto} \Phi$, where $\mathcal{X}$ is the space of inputs and $\btheta$ the vector of parameters. Let $\ell_{\by}(\btheta) = \sum_n^N \log \pi \big(y_n| f_{\btheta}(\bx_n))$ denote the likelihood function of $\btheta$ and define $\phi_n = f_{\btheta}(\bx_n) \in \bPhi \ \forall n$. Using the chain rule for the Hessian matrix of $\ell_{\by}(\btheta)$ w.r.t $\btheta$ the expression involving the score function vanishes in expectation. Therefore, the metric tensor on $\btheta$ reads
\begin{align*}
    \E_{\bY}(- & \nabla^2 \ell_{\bY}(\btheta)_{i, j})
    =  -\sum_n^N \E_{Y_n} \bigg( \partial^2_{i, j} \log \pi \big(Y_n| \underbrace{f_{\btheta}(\bx_n)}_{\phi_n} ) \bigg) \nonumber \\[-0.3cm]
    & = \sum_{n}^N \boldsymbol{J}_n^\top \bG_{\bPhi}(f_{\btheta}(\bx_n)) \boldsymbol{J}_n,
\end{align*}
where $\boldsymbol{J}_n$ $=$ $[\nabla (\boldsymbol{\phi}_n)_1 \cdots \nabla (\boldsymbol{\phi}_n)_ P]^\top$ is a $P \times D$ Jacobian matrix of the NN at the $n^{th}$ input. In other words, we can form the FIM for the whole network by transforming the FIM on its basic form with Jacobians that can be computed using standard automatic differentiation. Formally, this is a pullback metric from the parameter space $\bPhi$ to the NN parameter space $\bTheta$; see Section~\ref{sec:nn-pullback} in the Supplement for further details.

\paragraph{Computational cost.}
The core computational cost for \lafamilyabbrv \ comes from obtaining the acceleration given the position and velocity, which is needed in each step of the ODE integrator. We have
\begin{equation*}
\pdv[2]{\theta_{k}}{t} = -\frac{1}{2}G^{kl} \left[ \left(\partial_{i}G^{kl}+\partial_{j}G_{il}-\partial_{l}G_{ij}\right)v^{i}v^{j}\right],
\end{equation*}
which is a product of the inverse of the metric ($G^{kl}$) and a vector. 
For \ourabbrv, the computational cost is dominated by the inversion. The cost for drawing $N$ samples hence becomes $O(NTD^3)$, where $T$ is the number of evaluations during integration. \bergaminabbrv\ has lower cost due to avoiding direct inversion, but we will later see that with the Fisher metric we can often use considerably smaller $T$ that balances the difference. Also note that already standard \euclaabbrv\  has cost $O(D^3 + ND^2)$ due to inversion of Hessian and multivariate sampling.
The sampling parallelizes over $N$ for all methods. 

\section{EXPERIMENTS}

Code for reproducing the experiments is available at \url{https://github.com/ksnxr/RLAF}.

\subsection{Experimental setup}

We evaluate three methods: \ourabbrv\ using the Fisher metric (Section~\ref{sec:fim-riem}), \bergaminabbrv\ as proposed by \citet{Bergamin2023}, and \logbergaminabbrv\ using the logarithmic map (Section~\ref{sec:log-riem}). In addition, we show results for standard \euclaabbrv\ and discuss the choice of the MAP estimate and covariance for sampling initial velocities when relevant.
We repeat the experiments $5$ times and report averaged results over the repetitions.

For evaluating the approximation accuracy, we generate $20,000$ samples using the NUTS sampler in Stan \citep{Stan2023} as the ground truth. For low-dimensional problems we directly compare the approximation with the posterior samples by computing the Wasserstein distance $\mathcal{W}$ similar to \citet{Zhang2022}.
For the neural network experiments we compare the model predictions instead, due to non-identifiability of the posterior, measured using mean squared error (MSE) and negative log-likelihood (NLL).

Following \citet{Bergamin2023}, we use a numerical ODE solver with adaptive step sizes. We report the actual number of integration steps $T$ (and in some case running time) for the exponential map as an indicator of the complexity of the metric. The numbers are for producing one sample since the algorithm parallelizes trivially over the samples.
For the NN experiment we use
SciPy \citep{Virtanen2020} for integration, and for other experiments Diffrax \citep{Kidger2021}. Additional details are provided in Supplement.

\begin{table*}
	\begin{center}
		\begin{tabular}{|l|l|l|l|l|l|l|l|}
			\hline
			\multicolumn{1}{|c|}{} & \multicolumn{1}{|c|}{\euclaabbrv} & \multicolumn{2}{|c|}{\bergaminabbrv} & \multicolumn{2}{|c|}{\logbergaminabbrv} & \multicolumn{2}{|c|}{\ourabbrv} \\
			\hline
			MAP & $\mathcal{W}$ & $\mathcal{W}$ & $T$ & $\mathcal{W}$ & $T$ & $\mathcal{W}$ & $T$ \\
			\hline
			Euclidean & [1.434, 0.01] & [0.811, 0.009] & 70.4 & [\textbf{0.788}, 0.013] & 70.4 & [0.791, 0.014] & \textbf{24.9} \\
			\hline
			Hausdorff & [1.386, 0.009] & [0.341, 0.006] & 51.3 & [0.208, 0.007] & 60.0 & [\textbf{0.143}, 0.009] & \textbf{32.7} \\
			\hline
		\end{tabular}
	\end{center}
	\caption{Banana distribution results as $[\text{mean},\text{std}]$. Bold font indicates the best method. $\mathcal{W}$ indicates Wasserstein distance to NUTS samples while $T$ indicates the average number of function evaluations for one sample. For all evaluation metrics smaller is better.}
	\label{tbl:banana}
\end{table*}

\subsection{Banana distribution}
\label{sec:banana}

\begin{figure}[t]
    \centering
    \begin{tabular}{ccc}
        \includegraphics[width=0.3\columnwidth]{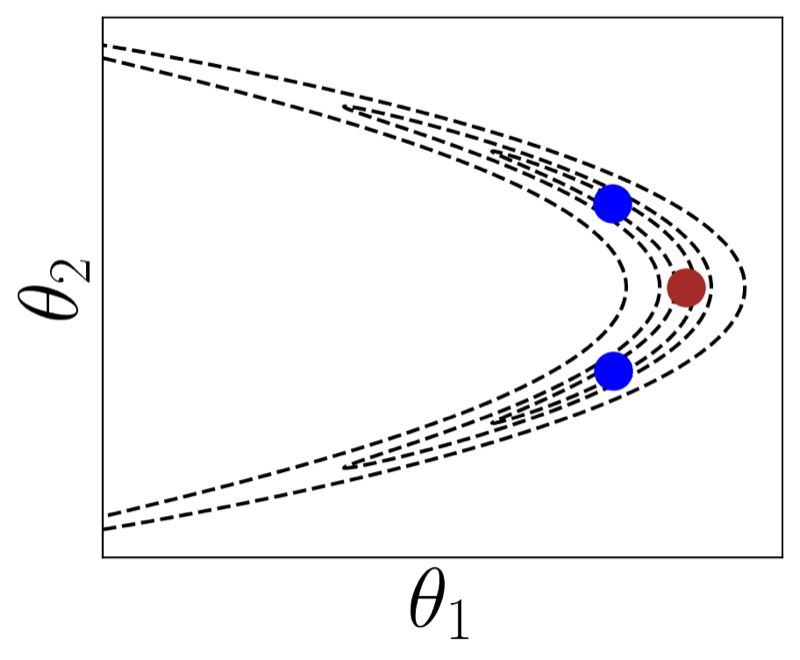} & 

        \includegraphics[width=0.3\columnwidth]{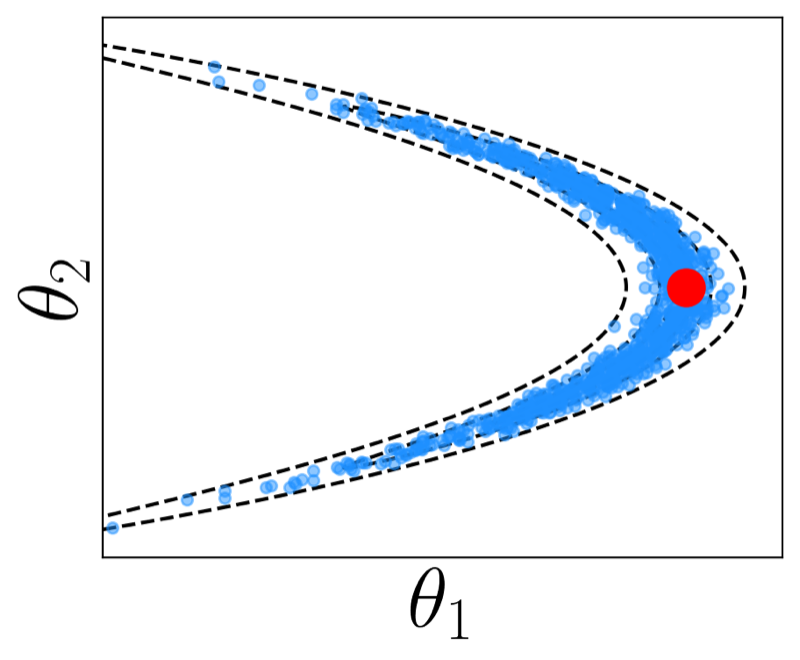} &
        \includegraphics[width=0.3\columnwidth]{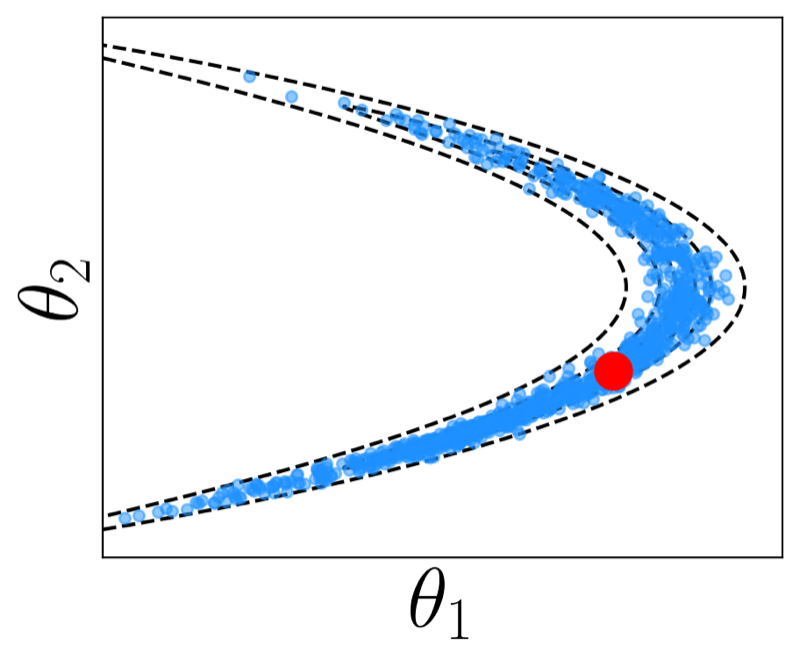}
        \end{tabular}
    \caption{Left: The banana distribution has two Euclidean MAPs (blue) but only one Hausdorff MAP (brown). \ourabbrv\ depends on the MAP choice, with Hausdorff (middle) being here superior to Euclidean (right).  
}
    \label{fig:banana-maps}
\end{figure}

We first compare the approximations in a 2D problem with difficult geometry, the banana distribution
\begin{align*}
    \btheta \overset{\textrm{ind.}}{\sim} \N(0,\sigma_{\btheta}^{2}) \ \textrm{and} \ Y_n \vert \btheta \sim \N(\theta_1 + \theta_2^2,\sigma_Y^{2}),
\end{align*}
where we set $\sigma_{\btheta}=\sigma_Y = 2$, $N = 100$. The samples are obtained using $\theta_{1} = 0.5$ and $\theta_{2}^{2} = 0.75$; the settings are inspired by \citet{Brofos2021}. The distribution has closed-form FIM (Section~\ref{sec:comp-fisher} also applies).

\begin{table*}
	\begin{center}
		\begin{tabular}{|l|l|l|l|l|l|l|l|l|}
			\hline
			 & & \multicolumn{1}{|c|}{\euclaabbrv} & \multicolumn{2}{|c|}{\bergaminabbrv} & \multicolumn{2}{|c|}{\logbergaminabbrv} & \multicolumn{2}{|c|}{\ourabbrv} \\
			\hline
			 & data & $\mathcal{W}$ & $\mathcal{W}$ & $T$ & $\mathcal{W}$ & $T$ & $\mathcal{W}$ & $T$ \\
			\hline
			\multirow{5}{1em}{\centering\rotatebox[origin=c]{90}{stand.}}&Ripl & [0.106, 0.004] & [0.236, 0.001] & 29.7 & [0.16, 0.065] & 45.2 & [\textbf{0.064}, 0.002] & \textbf{12.2} \\
			\cline{2-9}
			&Pima & [0.149, 0.0] & [0.274, 0.0] & 41.7 & [0.21, 0.007] & 78.5 & [\textbf{0.147}, 0.0] & \textbf{12.2} \\
			\cline{2-9}
			&Hear & [0.529, 0.001] & [0.649, 0.0] & 40.2 & [1.154, 0.189] & 71.6 & [\textbf{0.514}, 0.0] & \textbf{13.2} \\
			\cline{2-9}
			&Aust & [0.441, 0.001] & [0.522, 0.001] & 51.5 & [0.745, 0.203] & 79.1 & [\textbf{0.417}, 0.001] & \textbf{18.0} \\
			\cline{2-9}
			&Germ & [0.388, 0.001] & [0.431, 0.0] & 58.4 & [0.676, 0.1] & 104.0 & [\textbf{0.387}, 0.0] & \textbf{14.2} \\
			\hline
			\multirow{5}{1em}{\centering\rotatebox[origin=c]{90}{raw}}&Ripl & [0.437, 0.017] & [0.489, 0.012] & 30.2 & [0.554, 0.253] & 45.0 & [\textbf{0.247}, 0.012] & \textbf{12.7} \\
			\cline{2-9}
			&Pima & [0.21, 0.01] & [0.294, 0.004] & 5632.9 & [0.216, 0.014] & 1690.4 & [\textbf{0.112}, 0.008] & \textbf{15.4} \\
			\cline{2-9}
			&Hear & [0.842, 0.026] & [0.96, 0.012] & \textit{7868.2} & [0.898, 0.029] & \textit{3161.1} & [\textbf{0.644}, 0.012] & \textbf{17.9} \\
			\cline{2-9}
			&Aust & [0.454, 0.004] & [0.455, 0.01] & \textit{18513.8} & [0.467, 0.005] & \textit{12298.6} & [\textbf{0.378}, 0.005] & \textbf{18.0} \\
			\cline{2-9}
			&Germ & [0.846, 0.002] & [0.92, 0.003] & 3545.9 & [0.898, 0.003] & 2985.5 & [\textbf{0.823}, 0.001] & \textbf{17.8} \\
			\hline
		\end{tabular}
	\end{center}
	\caption{Logistic regression results as $[\text{mean}, \text{std}]$. Bold font indicates the best method. Italic font indicates the integrator reached maximum number of steps in at least one run. $\mathcal{W}$ indicates Wasserstein distance to NUTS samples while $T$ indicates the average number of function evaluations for one sample. For all evaluation metrics smaller is better.}
	\label{tbl:lr}
\end{table*}
The distribution has two modes symmetric across the x-axis with the same log-posterior. However, the Hausdorff MAP is unique and at the x-axis. 
Figure~\ref{fig:banana-maps} compares the resulting approximations with \ourabbrv \ for the different MAP estimates; with Euclidean MAP, we use Hessian precision for sampling the velocities, whereas with Hausdorff MAP we use the Fisher precision.
Table~\ref{tbl:banana} shows that all Riemannian methods clearly outperform classical \euclaabbrv\ for this target, and for all Riemannian methods using the Hausdorff MAP is considerably better. Both of the newly proposed methods outperform \bergaminabbrv.

\subsection{Bayesian logistic regression}
\label{sec:blr}

Following \citet{Girolami2011} and \citet{Lan2015}, we applied Bayesian logistic regression on five datasets (details in Supplement) using
the model
\begin{align*}
    Y_n \vert \btheta \sim \textrm{Bernoulli}(\sigma(\btheta^\top\bx_n)) \ \ \textrm{and} \ \ \btheta \overset{\textrm{ind.}}{\sim} \N(0,\alpha),
\end{align*}
where $\sigma : \mathcal{X} \rightarrow (0, 1)$ is the sigmoid function and $\alpha = 100$. The parameter space of the Bernoulli model in its basic form is $\bphi \in (0, 1) = \Phi$, 
whose FIM is known. Therefore the Fisher metric on the parameter space $\bTheta$ becomes $\bG$ $=$ $\bX^{\top}\bLambda\bX + \alpha^{-1}\bI$, 
where $\bX = [\bx_1 \cdots \bx_N]$ is the covariate matrix (or inputs) 
and $\bLambda$ is a diagonal matrix with elements $\bLambda_{n n} = \sigma\big(\boldsymbol{\theta}^\top\boldsymbol{x}_n\big) \big(1 - \sigma(\boldsymbol{\theta}^\top \boldsymbol{x}_n \big) \big)$. We use Euclidean MAP, and here the negative Hessian coincides with the Fisher metric \citep{Girolami2011} (so both choices for $\bSigma$ are the same).

Table~\ref{tbl:lr} reports approximation accuracies for two setups: With 
standardized (z-score)
inputs to make the geometry of the problem easier, and with the raw inputs where the metric also needs to handle potentially large scale differences.
For both cases, 
\ourabbrv\ is clearly the best on all data sets, and the approximations using the Monge metric are worse than \euclaabbrv.
The Fisher metric results in smoother integration surface, especially when not standardizing the inputs, with the Monge metric variants needing up to $1000$ times more evaluations and consequently also more time despite lower complexity (see Supplement for times).

\subsection{Neural network regression}
\label{sec:nn-regression}

Finally, we ran an 
experiment
similar to the one by \citet{Bergamin2023}, with data from \citet{Snelson2005}.
The task is a $1$D regression problem using an NN of size $1$-$10$-$1$ with $\tanh$ activation. The methods are implemented using \citet{Daxberger2021}, with Fisher precision (the default option for full LA), $\hat{\btheta}$ found by standard MAP training, and post hoc optimized prior precision and noise std.
MSE and NLL are computed using $500$ samples.
\begin{figure*}[t]
    \centering
    \begin{tabular}{cccc}
        \includegraphics[width=0.22\textwidth]{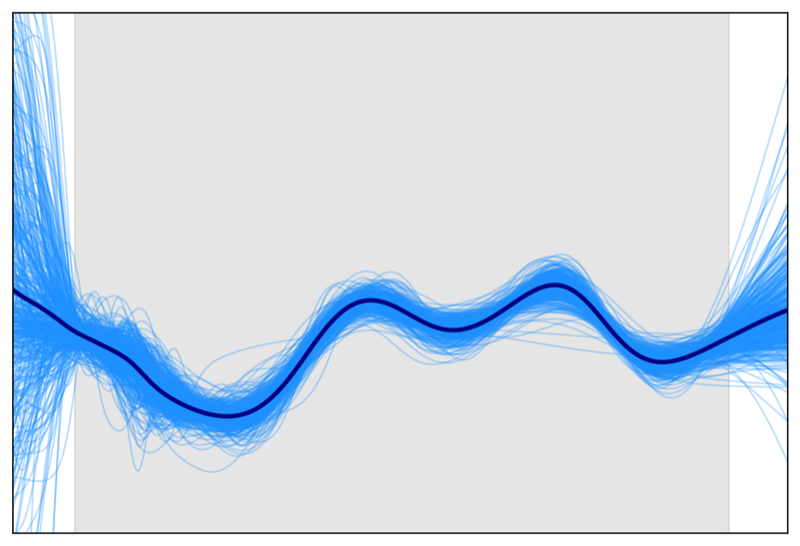} &
        \includegraphics[width=0.22\textwidth]{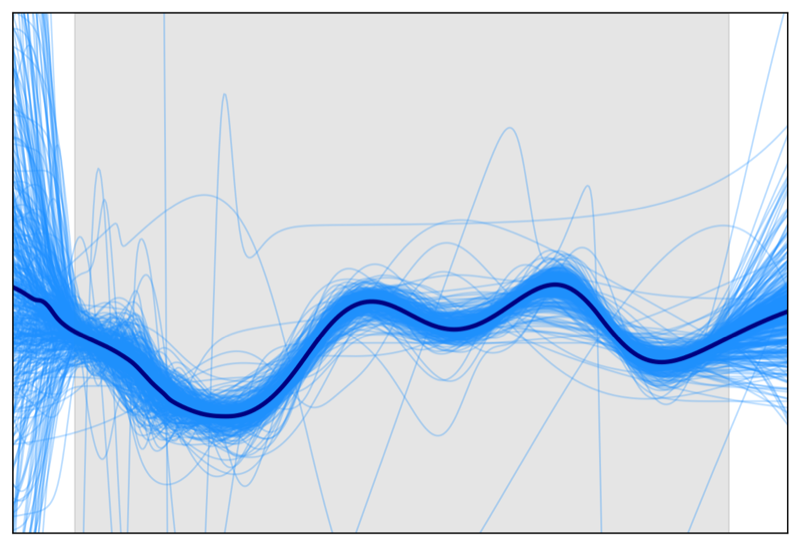} &
        \includegraphics[width=0.22\textwidth]{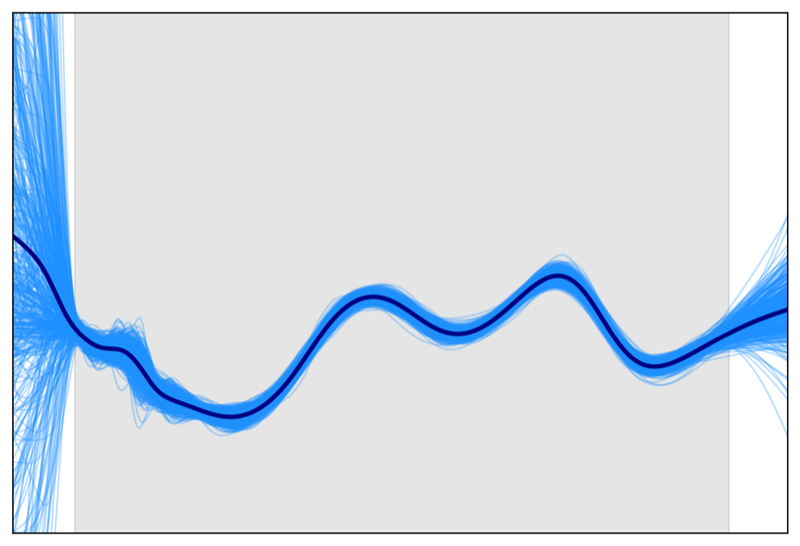} &
        \includegraphics[width=0.22\textwidth]{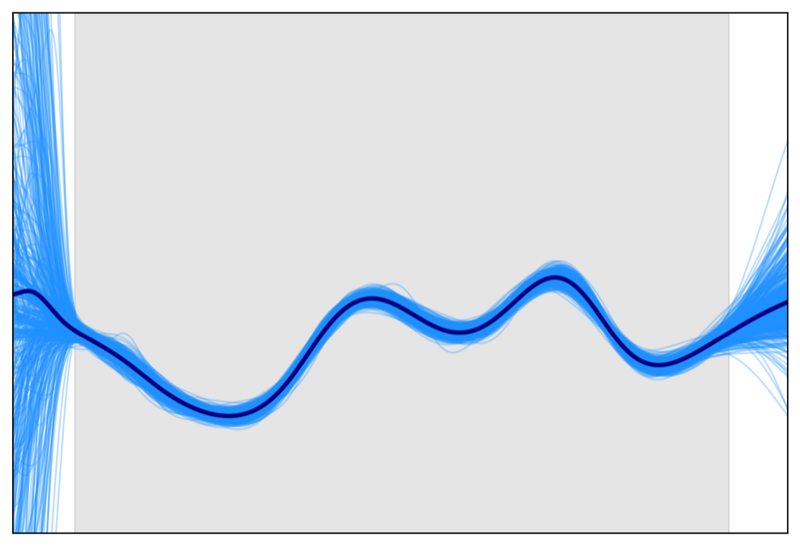} \\
        \includegraphics[width=0.22\textwidth]{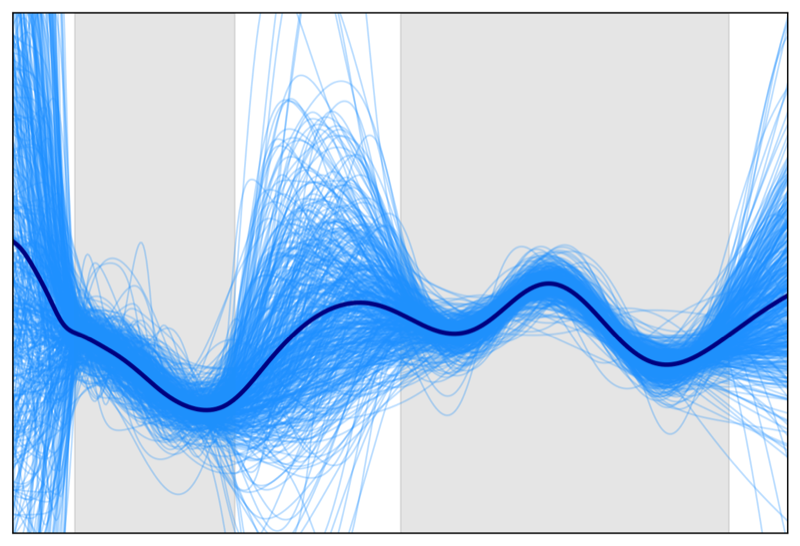} &
        \includegraphics[width=0.22\textwidth]{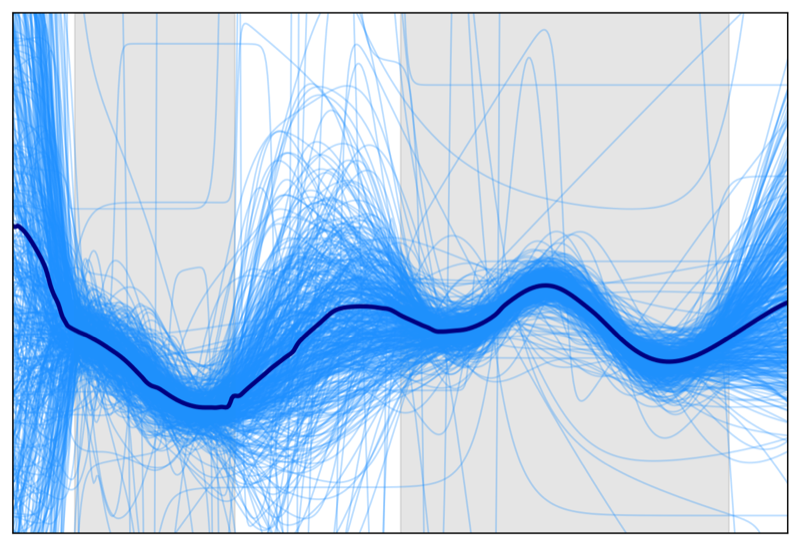} &
        \includegraphics[width=0.22\textwidth]{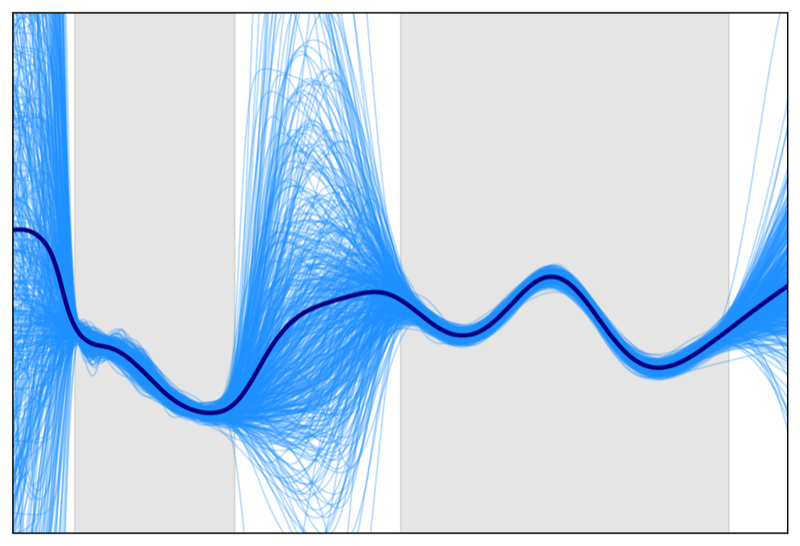} &
        \includegraphics[width=0.22\textwidth]{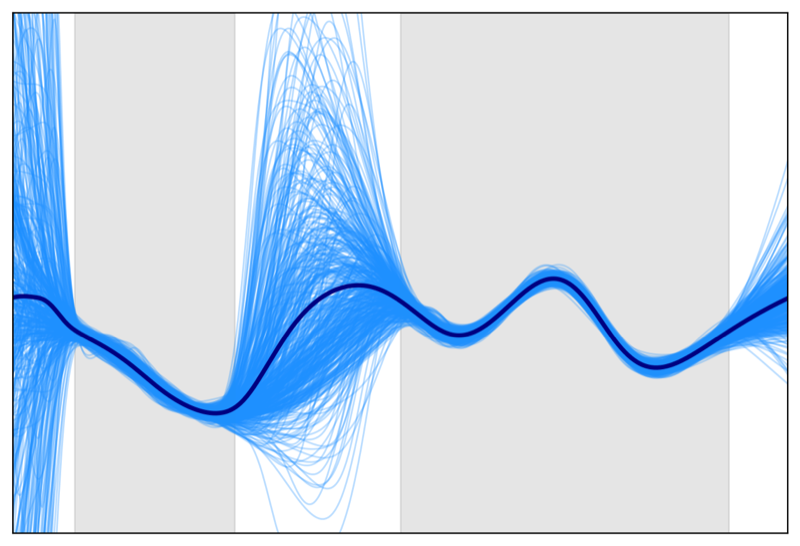}
    \end{tabular}
    
    \caption{NN regression with complete (top) and gap (bottom) training data. Methods from left to right:
    \bergaminabbrv; \logbergaminabbrv; \ourabbrv; NUTS. Gray shading denotes the part of $x$-axis with training data, dark line is the mean prediction, and blue lines are samples.}
    \label{fig:nn-reg}
\end{figure*}
Fisher metric and Christoffel symbols follow from Section~\ref{sec:comp-fisher} and are
\begin{align*}
    G_{kl} &= \frac{1}{\sigma^{2}}\sum_{n = 1}^N \partial_k f_{\btheta}(\bx_n) \partial_l f_{\btheta}(\bx_n) + \alpha^{-1}\delta^{k}_{l} \; \; \text{and} \\
    v^i v^j \Gamma^{k}_{ij} &= \frac{1}{\sigma^{2}}G^{kl}\sum_{n=1}^{N} \partial_l f_{\btheta}(\bx_n) \Big(v^i v^j \partial^2_{i, j} f_{\btheta}(\bx_n) \Big),
\end{align*}
as we use Gaussian prior with precision $\alpha$ on every $\btheta$ and Gaussian probabilistic model for $Y$; see \citet{geom:2014} and \citet{Song2018} for derivations.

Following \citet{Bergamin2023}, we run experiments with training samples covering a continuous area of inputs (\textit{complete}) and with a gap between $1.5$ and $3$ (\textit{gap}) to measure the in-between uncertainty quantification \citep{Foong2019}.
Table~\ref{tbl:nn-reg} quantifies the results and shows that \ourabbrv\ is in general the best approximation and for the \textit{complete} case matches also NUTS. 
As is well known, Euclidean LA does not work without linearization of the network \citep{Daxberger2021}. 
Figure~\ref{fig:nn-reg} illustrates the two scenarios, where \ourabbrv\ is effectively indistinguishable from NUTS, whereas both \bergaminabbrv\ and \logbergaminabbrv\ overestimate predictive variance and some of the posterior samples are clearly off.

\begin{table*}
	\begin{center}
		\begin{tabular}{|l|l|l|l|l|l|l|l|l|}
			\hline
			\multicolumn{1}{|c|}{} & \multicolumn{4}{|c|}{Complete} & \multicolumn{4}{|c|}{Gap} \\
			\hline
			method & MSE & NLL & $T$ & time & MSE & NLL & $T$ & time \\
			\hline
			\euclaabbrv & [2.235, 0.617] & [3.02, 0.064] & N/A & N/A & [1.365, 0.407] & [3.171, 0.075] & N/A & N/A \\
			\hline
			\bergaminabbrv & [0.2, 0.237] & [0.193, 0.003] & 6677.8 & 34.1 & [1.646, 2.722] & [1.031, 0.02] & 6573.2 & 34.0 \\
			\hline
			\logbergaminabbrv & [0.107, 0.054] & [0.201, 0.004] & 6674.7 & 37.7 & [\textbf{0.27}, 0.068] & [1.122, 0.029] & 4793.9 & 28.5 \\
			\hline
			\ourabbrv & [\textbf{0.072}, 0.0] & [\textbf{0.121}, 0.002] & \textbf{242.9} & \textbf{0.8} & [0.564, 0.066] & [1.063, 0.017] & \textbf{727.7} & \textbf{2.4} \\
			\hline
			NUTS & [0.073, 0.0] & [0.126, 0.002] & N/A & N/A & [0.394, 0.02] & [\textbf{0.872}, 0.029] & N/A & N/A \\
			\hline
		\end{tabular}
	\end{center}
	\caption{NN regression results as $[\text{mean},\text{std}]$. Bold font indicates the best method. $T$ indicates the average number of function evaluations for one sample while \textit{time} indicates the average time for one sample. For all evaluation metrics smaller is better.}
	\label{tbl:nn-reg}
\end{table*}

\section{DISCUSSION AND CONCLUSIONS}

\paragraph{Main observation.} Our key message is that the Riemannian Laplace Approximation proposed by \citet{Bergamin2023} is promising as practical and flexible approximation, but their specific variant is biased and the bias can be relatively large already in simple problems. The bias can be resolved by either slightly modifying the algorithm or by considering more suitable metrics, and the corrected approximation family has interesting theoretical properties. For instance, with the Fisher metric, it is exact for targets diffeomorphic with a Gaussian. This opens up new research directions for the study of metrics in inference tasks.

\paragraph{On Monge metric.}
One likely reason for the poor empirical performance of \bergaminabbrv\ is the (lack of) relative scaling of the identity and the outer product in the metric. \citet{Hartmann2022} and \citet{Yu2023} empirically showed that for MCMC the latter may need to be heavily down-weighted and even then it is not always better than Euclidean. \bergaminabbrv\ likely requires some form of scaling as well and downweighting the latter term would reduce the bias, but selecting the scaling remains an open problem. \citet{Hartmann2023} recently proposed a possible remedy that could be used as the basis for more accurate yet computationally efficient metric.

\paragraph{Recommendations.} 
For small-to-medium problems \ourabbrv\ is clearly the best choice, based on consistent best performance and stable computation. For large-scale problems, for instance in the common use case of NNs \citep{Daxberger2021}, \bergaminabbrv\ remains currently the best choice, despite the bias.

We showed how to compute the FIM for arbitrary NNs, but only experimented on tiny ones to facilitate direct comparisons against NUTS. 
We also experimented with NN regression problems of varying numbers of parameters $D$ and number of data points. Interestingly, \ourabbrv\ can be faster than \bergaminabbrv\ even for $D$ larger than $1000$; see Section~\ref{sec:scalability} in the Supplement for details.
We could naturally run \ourabbrv\ for still larger $D$ in reasonable time, but very large models necessarily call for a more efficient metric.
A promising direction for scalable and accurate method could build on 
scalable approximations for FIM \citep{Martens2015,George2018} with practical implementations as in \citet{Botev2023} and \citet{George2021}, or on exact subsampled Fisher \citep{Benzing2022}.

\subsubsection*{Acknowledgements}
HY, MH, BW and AK are supported by the Research Council of Finland Flagship programme: Finnish Center for Artificial Intelligence FCAI, and additionally by the grants 336019, 345811, 348952, 324852. MG is supported by EPSRC grants EP/T000414/1, EP/R018413/2, EP/P020720/2, EP/R034710/1, EP/R004889/1, and a Royal Academy of Engineering Research Chair. The authors acknowledge support from CSC – IT Center for Science, Finland, for computational resources.


\bibliography{references}

\begin{thebibliography}{}

\bibitem[Arvanitidis et~al., 2016]{Arvanitidis2016}
Arvanitidis, G., Hansen, L.~K., and Hauberg, S. (2016).
\newblock A {Locally} {Adaptive} {Normal} {Distribution}.
\newblock In {\em Advances in {Neural} {Information} {Processing} {Systems}},
  volume~29. Curran Associates, Inc.

\bibitem[Atkinson et~al., 2009]{Atkinson2009}
Atkinson, K., Han, W., and Stewart, D. (2009).
\newblock {\em Numerical {Solutions} of {Ordinary} {Differential} {Equations}}.
\newblock John Wiley \& Sons.

\bibitem[Benzing, 2022]{Benzing2022}
Benzing, F. (2022).
\newblock Gradient {Descent} on {Neurons} and its {Link} to {Approximate}
  {Second}-order {Optimization}.
\newblock In {\em Proceedings of the 39th {International} {Conference} on
  {Machine} {Learning}}, volume 162 of {\em Proceedings of {Machine} {Learning}
  {Research}}, pages 1817--1853. PMLR.

\bibitem[Bergamin et~al., 2023]{Bergamin2023}
Bergamin, F., Moreno-Muñoz, P., Hauberg, S.~r., and Arvanitidis, G. (2023).
\newblock Riemannian {Laplace} approximations for {Bayesian} neural networks.
\newblock In {\em Advances in {Neural} {Information} {Processing} {Systems}},
  volume~36, pages 31066--31095. Curran Associates, Inc.

\bibitem[Blei et~al., 2017]{Blei2017}
Blei, D.~M., Kucukelbir, A., and McAuliffe, J.~D. (2017).
\newblock Variational {Inference}: {A} {Review} for {Statisticians}.
\newblock {\em Journal of the American Statistical Association},
  112(518):859--877.
\newblock Publisher: Taylor \& Francis.

\bibitem[Botev and Martens, 2023]{Botev2023}
Botev, A. and Martens, J. (2023).
\newblock {KFAC}-{JAX}.

\bibitem[Boumal, 2023]{Boumal2023}
Boumal, N. (2023).
\newblock {\em An introduction to optimization on smooth manifolds}.
\newblock Cambridge University Press.

\bibitem[Bradbury et~al., 2018]{jax:2018}
Bradbury, J., Frostig, R., Hawkins, P., Johnson, M.~J., Leary, C., Maclaurin,
  D., Necula, G., Paszke, A., VanderPlas, J., Wanderman-Milne, S., and Zhang,
  Q. (2018).
\newblock {JAX}: composable transformations of {Python}+{NumPy} programs.

\bibitem[Brofos and Lederman, 2021]{Brofos2021}
Brofos, J.~A. and Lederman, R.~R. (2021).
\newblock On {Numerical} {Considerations} for {Riemannian} {Manifold}
  {Hamiltonian} {Monte} {Carlo}.

\bibitem[Calin and Udrişte, 2014]{geom:2014}
Calin, O. and Udrişte, C. (2014).
\newblock {\em Geometric Modeling in Probability and Statistics}.
\newblock Springer International Publishing, 1 edition.

\bibitem[Carmo, 1992]{doCarmo1992}
Carmo, M.~d. (1992).
\newblock {\em Riemannian {Geometry}}.
\newblock Birkhäuser, Boston.

\bibitem[Daxberger et~al., 2021]{Daxberger2021}
Daxberger, E., Kristiadi, A., Immer, A., Eschenhagen, R., Bauer, M., and
  Hennig, P. (2021).
\newblock Laplace {Redux} - {Effortless} {Bayesian} {Deep} {Learning}.
\newblock In {\em Advances in {Neural} {Information} {Processing} {Systems}},
  volume~34, pages 20089--20103. Curran Associates, Inc.

\bibitem[De~Bortoli et~al., 2022]{DeBortoli2022}
De~Bortoli, V., Mathieu, E., Hutchinson, M., Thornton, J., Teh, Y.~W., and
  Doucet, A. (2022).
\newblock Riemannian {Score}-{Based} {Generative} {Modelling}.
\newblock In {\em Advances in {Neural} {Information} {Processing} {Systems}},
  volume~35, pages 2406--2422. Curran Associates, Inc.

\bibitem[Dormand and Prince, 1980]{Dormand1980}
Dormand, J.~R. and Prince, P.~J. (1980).
\newblock A family of embedded {Runge}-{Kutta} formulae.
\newblock {\em Journal of Computational and Applied Mathematics}, 6(1):19--26.

\bibitem[Federer, 1969]{fed:1969}
Federer, H. (1969).
\newblock {\em Geometric Measure Theory}.
\newblock Springer-Verlag.

\bibitem[Feng et~al., 1992]{Feng1992}
Feng, C., S., A., K., S., M., S., and Henery, R. (1992).
\newblock Statlog {Project}.
\newblock Published: UCI Machine Learning Repository.

\bibitem[Flamary et~al., 2021]{Flamary2021}
Flamary, R., Courty, N., Gramfort, A., Alaya, M.~Z., Boisbunon, A., Chambon,
  S., Chapel, L., Corenflos, A., Fatras, K., Fournier, N., Gautheron, L.,
  Gayraud, N. T.~H., Janati, H., Rakotomamonjy, A., Redko, I., Rolet, A.,
  Schutz, A., Seguy, V., Sutherland, D.~J., Tavenard, R., Tong, A., and Vayer,
  T. (2021).
\newblock {POT}: {Python} {Optimal} {Transport}.
\newblock {\em Journal of Machine Learning Research}, 22(78):1--8.

\bibitem[Foong et~al., 2019]{Foong2019}
Foong, A. Y.~K., Li, Y., Hernández-Lobato, J.~M., and Turner, R.~E. (2019).
\newblock '{In}-{Between}' {Uncertainty} in {Bayesian} {Neural} {Networks}.

\bibitem[Frank et~al., 2021]{Frank2021}
Frank, P., Leike, R., and Enßlin, T.~A. (2021).
\newblock Geometric {Variational} {Inference}.
\newblock {\em Entropy}, 23(7).

\bibitem[George, 2021]{George2021}
George, T. (2021).
\newblock {NNGeometry}: {Easy} and {Fast} {Fisher} {Information} {Matrices} and
  {Neural} {Tangent} {Kernels} in {PyTorch}.

\bibitem[George et~al., 2018]{George2018}
George, T., Laurent, C., Bouthillier, X., Ballas, N., and Vincent, P. (2018).
\newblock Fast {Approximate} {Natural} {Gradient} {Descent} in a {Kronecker}
  {Factored} {Eigenbasis}.
\newblock In {\em Advances in {Neural} {Information} {Processing} {Systems}},
  volume~31. Curran Associates, Inc.

\bibitem[George~Casella, 2001]{cas:2001}
George~Casella, R. L.~B. (2001).
\newblock {\em Statistical Inference}.
\newblock Duxbury Press, 2° edition.

\bibitem[Girolami and Calderhead, 2011]{Girolami2011}
Girolami, M. and Calderhead, B. (2011).
\newblock Riemann manifold {Langevin} and {Hamiltonian} {Monte} {Carlo}
  methods.
\newblock {\em Journal of the Royal Statistical Society: Series B (Statistical
  Methodology)}, 73(2):123--214.

\bibitem[Greff et~al., 2017]{Greff2017}
Greff, K., Klein, A., Chovanec, M., Hutter, F., and Schmidhuber, J. (2017).
\newblock The {Sacred} {Infrastructure} for {Computational} {Research}.
\newblock In Huff, K., Lippa, D., Niederhut, D., and Pacer, M., editors, {\em
  Proceedings of the 16th {Python} in {Science} {Conference}}, pages 49 -- 56.

\bibitem[Grosse, 2022]{Grosse2022}
Grosse, R. (2022).
\newblock Chapter 3: {Metrics}.

\bibitem[Harris et~al., 2020]{Harris2020}
Harris, C.~R., Millman, K.~J., Walt, S. J. v.~d., Gommers, R., Virtanen, P.,
  Cournapeau, D., Wieser, E., Taylor, J., Berg, S., Smith, N.~J., Kern, R.,
  Picus, M., Hoyer, S., Kerkwijk, M. H.~v., Brett, M., Haldane, A., Río, J.
  F.~d., Wiebe, M., Peterson, P., Gérard-Marchant, P., Sheppard, K., Reddy,
  T., Weckesser, W., Abbasi, H., Gohlke, C., and Oliphant, T.~E. (2020).
\newblock Array programming with {NumPy}.
\newblock {\em Nature}, 585(7825):357--362.
\newblock Publisher: Springer Science and Business Media LLC.

\bibitem[Hartmann et~al., 2022]{Hartmann2022}
Hartmann, M., Girolami, M., and Klami, A. (2022).
\newblock Lagrangian manifold {Monte} {Carlo} on {Monge} patches.
\newblock In {\em Proceedings of {The} 25th {International} {Conference} on
  {Artificial} {Intelligence} and {Statistics}}, volume 151 of {\em Proceedings
  of {Machine} {Learning} {Research}}, pages 4764--4781. PMLR.

\bibitem[Hartmann and Vanhatalo, 2019]{Hartmann2019}
Hartmann, M. and Vanhatalo, J. (2019).
\newblock Laplace approximation and natural gradient for {Gaussian} process
  regression with heteroscedastic student-t model.
\newblock {\em Statistics and Computing}, 29(4):753--773.

\bibitem[Hartmann et~al., 2023]{Hartmann2023}
Hartmann, M., Williams, B., Yu, H., Girolami, M., Barp, A., and Klami, A.
  (2023).
\newblock Warped geometric information on the optimisation of euclidean
  functions.
\newblock {\em arXiv preprint arXiv:2308.08305}.

\bibitem[Hauberg, 2018]{Hauberg2018}
Hauberg, S. (2018).
\newblock Directional {Statistics} with the {Spherical} {Normal}
  {Distribution}.
\newblock In {\em 2018 21st {International} {Conference} on {Information}
  {Fusion} ({FUSION})}, pages 704--711.

\bibitem[Jermyn, 2005]{ian:2005}
Jermyn, I.~H. (2005).
\newblock Invariant {B}ayesian estimation on manifolds.
\newblock {\em The Annals of Statistics}, 33:583--605.

\bibitem[Kidger, 2021]{Kidger2021}
Kidger, P. (2021).
\newblock {\em On {Neural} {Differential} {Equations}}.
\newblock {PhD} {Thesis}, University of Oxford.

\bibitem[Kingma and Ba, 2015]{Kingma2015}
Kingma, D.~P. and Ba, J. (2015).
\newblock Adam: {A} {Method} for {Stochastic} {Optimization}.
\newblock In {\em 3rd {International} {Conference} on {Learning}
  {Representations}, {ICLR} 2015, {San} {Diego}, {CA}, {USA}, {May} 7-9, 2015,
  {Conference} {Track} {Proceedings}}.

\bibitem[Knollmüller and Enßlin, 2020]{Knollmüller2020}
Knollmüller, J. and Enßlin, T.~A. (2020).
\newblock Metric {Gaussian} {Variational} {Inference}.

\bibitem[Kristiadi et~al., 2023]{Kristiadi2023}
Kristiadi, A., Dangel, F., and Hennig, P. (2023).
\newblock The {Geometry} of {Neural} {Nets}' {Parameter} {Spaces} {Under}
  {Reparametrization}.

\bibitem[Kunstner et~al., 2019]{Kunstner2019}
Kunstner, F., Hennig, P., and Balles, L. (2019).
\newblock Limitations of the empirical {Fisher} approximation for natural
  gradient descent.
\newblock In {\em Advances in {Neural} {Information} {Processing} {Systems}},
  volume~32. Curran Associates, Inc.

\bibitem[Lan et~al., 2015]{Lan2015}
Lan, S., Stathopoulos, V., Shahbaba, B., and Girolami, M. (2015).
\newblock Markov {Chain} {Monte} {Carlo} {From} {Lagrangian} {Dynamics}.
\newblock {\em Journal of Computational and Graphical Statistics},
  24(2):357--378.

\bibitem[{Lee, John M.}, 2018]{Lee2018}
{Lee, John M.} (2018).
\newblock {\em Introduction to {Riemannian} {Manifolds}}.
\newblock Graduate {Texts} in {Mathematics}. Springer International Publishing
  AG, Cham, Switzerland, 2nd edition.

\bibitem[Lehmann, 2003]{leh:2003}
Lehmann, G.~C. (2003).
\newblock {\em Theory of Point Estimation}.
\newblock Springer texts in statistics. Springer, 2nd ed edition.

\bibitem[Magnusson et~al., 2022]{Magnusson2022}
Magnusson, M., Bürkner, P., and Vehtari, A. (2022).
\newblock posteriordb: a set of posteriors for {Bayesian} inference and
  probabilistic programming.

\bibitem[Margossian and Blei, 2023]{Margossian2023}
Margossian, C.~C. and Blei, D.~M. (2023).
\newblock Amortized {Variational} {Inference}: {When} and {Why}?

\bibitem[Martens, 2020]{Martens2020}
Martens, J. (2020).
\newblock New {Insights} and {Perspectives} on the {Natural} {Gradient}
  {Method}.
\newblock {\em Journal of Machine Learning Research}, 21(146):1--76.

\bibitem[Martens and Grosse, 2015]{Martens2015}
Martens, J. and Grosse, R. (2015).
\newblock Optimizing {Neural} {Networks} with {Kronecker}-factored
  {Approximate} {Curvature}.
\newblock In {\em Proceedings of the 32nd {International} {Conference} on
  {Machine} {Learning}}, volume~37 of {\em Proceedings of {Machine} {Learning}
  {Research}}, pages 2408--2417, Lille, France. PMLR.

\bibitem[Mathieu et~al., 2019]{Mathieu2019}
Mathieu, E., Le~Lan, C., Maddison, C.~J., Tomioka, R., and Teh, Y.~W. (2019).
\newblock Continuous {Hierarchical} {Representations} with {Poincaré}
  {Variational} {Auto}-{Encoders}.
\newblock In Wallach, H., Larochelle, H., Beygelzimer, A., Alché-Buc, F.~d.,
  Fox, E., and Garnett, R., editors, {\em Advances in {Neural} {Information}
  {Processing} {Systems}}, volume~32. Curran Associates, Inc.

\bibitem[Meurer et~al., 2017]{Meurer2017}
Meurer, A., Smith, C.~P., Paprocki, M., \v{C}ert\'ik, O., Kirpichev, S.~B.,
  Rocklin, M., Kumar, A., Ivanov, S., Moore, J.~K., Singh, S., Rathnayake, T.,
  Vig, S., Granger, B.~E., Muller, R.~P., Bonazzi, F., Gupta, H., Vats, S.,
  Johansson, F., Pedregosa, F., Curry, M.~J., Terrel, A.~R., Rou\v{c}ka, v.,
  Saboo, A., Fernando, I., Kulal, S., Cimrman, R., and Scopatz, A. (2017).
\newblock {SymPy}: symbolic computing in {Python}.
\newblock {\em PeerJ Computer Science}, 3:e103.

\bibitem[Minka, 2001]{minka:2001}
Minka, T.~P. (2001).
\newblock Expectation propagation for approximate bayesian inference.
\newblock In {\em Proceedings of the 17th Conference in Uncertainty in
  Artificial Intelligence}, UAI '01, page 362–369. Morgan Kaufmann Publishers
  Inc.

\bibitem[Murphy, 2023]{Murphy2023}
Murphy, K.~P. (2023).
\newblock {\em Probabilistic {Machine} {Learning}: {Advanced} {Topics}}.
\newblock MIT Press.

\bibitem[Neal, 2003]{Neal2003}
Neal, R.~M. (2003).
\newblock Slice sampling.
\newblock {\em The Annals of Statistics}, 31(3):705 -- 767.
\newblock Publisher: Institute of Mathematical Statistics.

\bibitem[{OpenAI}, 2023]{ChatGPT2023}
{OpenAI} (2023).
\newblock {ChatGPT}.

\bibitem[Papamakarios et~al., 2021]{Papamakarios2021}
Papamakarios, G., Nalisnick, E., Rezende, D.~J., Mohamed, S., and
  Lakshminarayanan, B. (2021).
\newblock Normalizing {Flows} for {Probabilistic} {Modeling} and {Inference}.
\newblock {\em Journal of Machine Learning Research}, 22(57):1--64.

\bibitem[Paszke et~al., 2019]{pyt:2019}
Paszke, A., Gross, S., Massa, F., Lerer, A., Bradbury, J., Chanan, G., Killeen,
  T., Lin, Z., Gimelshein, N., Antiga, L., Desmaison, A., Kopf, A., Yang, E.,
  DeVito, Z., Raison, M., Tejani, A., Chilamkurthy, S., Steiner, B., Fang, L.,
  Bai, J., and Chintala, S. (2019).
\newblock Pytorch: An imperative style, high-performance deep learning library.
\newblock In {\em Advances in Neural Information Processing Systems 32}, pages
  8024--8035. Curran Associates, Inc.

\bibitem[Patterson and Teh, 2013]{Patterson2013}
Patterson, S. and Teh, Y.~W. (2013).
\newblock Stochastic {Gradient} {Riemannian} {Langevin} {Dynamics} on the
  {Probability} {Simplex}.
\newblock In {\em Advances in {Neural} {Information} {Processing} {Systems}},
  volume~26. Curran Associates, Inc.

\bibitem[Petersen and Pedersen, 2012]{Peterson2012}
Petersen, K.~B. and Pedersen, M.~S. (2012).
\newblock The {Matrix} {Cookbook}.

\bibitem[Rasmussen and Williams, 2006]{Rasmussen}
Rasmussen, C.~E. and Williams, C. K.~I. (2006).
\newblock {\em Gaussian processes for machine learning}.
\newblock Adaptive computation and machine learning. MIT Press.

\bibitem[Ripley, 1994]{Ripley1994}
Ripley, B.~D. (1994).
\newblock Neural networks and related methods for classification (with
  discussion).
\newblock {\em Journal of the Royal Statistical Society series B}, pages
  409--456.

\bibitem[Ripley, 1996]{Ripley1996}
Ripley, B.~D. (1996).
\newblock {\em Pattern {Recognition} and {Neural} {Networks}}.
\newblock Cambridge University Press.

\bibitem[Rue et~al., 2009]{rue:2009}
Rue, H., Martino, S., and Chopin, N. (2009).
\newblock Approximate bayesian inference for latent gaussian models by using
  integrated nested laplace approximations.
\newblock {\em Journal Of The Royal Statistical Society}, 71:319--392.

\bibitem[Schervish, 2012]{schervish:2012}
Schervish, M. (2012).
\newblock {\em Theory of Statistics}.
\newblock Springer Series in Statistics. Springer New York.

\bibitem[Smith et~al., 1988]{Smith1988}
Smith, J.~W., Everhart, J.~E., Dickson, W.~C., Knowler, W.~C., and Johannes,
  R.~S. (1988).
\newblock Using the {ADAP} {Learning} {Algorithm} to {Forecast} the {Onset} of
  {Diabetes} {Mellitus}.
\newblock In {\em Proceedings of the {Annual} {Symposium} on {Computer}
  {Application} in {Medical} {Care}}, Orlando.

\bibitem[Snelson and Ghahramani, 2005]{Snelson2005}
Snelson, E. and Ghahramani, Z. (2005).
\newblock Sparse {Gaussian} {Processes} using {Pseudo}-inputs.
\newblock In {\em Advances in {Neural} {Information} {Processing} {Systems}},
  volume~18. MIT Press.

\bibitem[Song et~al., 2018]{Song2018}
Song, Y., Song, J., and Ermon, S. (2018).
\newblock Accelerating {Natural} {Gradient} with {Higher}-{Order} {Invariance}.
\newblock In {\em Proceedings of the 35th {International} {Conference} on
  {Machine} {Learning}}, volume~80 of {\em Proceedings of {Machine} {Learning}
  {Research}}, pages 4713--4722. PMLR.

\bibitem[{Stan Dev Team}, 2023]{Stan2023}
{Stan Dev Team} (2023).
\newblock {CmdStanPy}.

\bibitem[{Stan Development Team}, 2023]{StanRef2023}
{Stan Development Team} (2023).
\newblock Stan {Modeling} {Language} {Users} {Guide} and {Reference} {Manual}.

\bibitem[Tierney and Kadane, 1986]{Tierney1986}
Tierney, L. and Kadane, J.~B. (1986).
\newblock Accurate approximations for posterior moments and marginal densities.
\newblock {\em Journal of the American Statistical Association},
  81(393):82--86.

\bibitem[van~der Vaart, 1998]{Vandervaart1998}
van~der Vaart, A. (1998).
\newblock {\em Asymptotic {Statistics}}.
\newblock Cambridge University Press, Cambridge, UK.

\bibitem[Vershynin, 2018]{Vershynin2018}
Vershynin, R. (2018).
\newblock {\em High-{Dimensional} {Probability}: {An} {Introduction} with
  {Applications} in {Data} {Science}}.
\newblock Cambridge {Series} in {Statistical} and {Probabilistic}
  {Mathematics}. Cambridge University Press.

\bibitem[Virtanen et~al., 2020]{Virtanen2020}
Virtanen, P., Gommers, R., Oliphant, T.~E., Haberland, M., Reddy, T.,
  Cournapeau, D., Burovski, E., Peterson, P., Weckesser, W., Bright, J.,
  van~der Walt, S.~J., Brett, M., Wilson, J., Millman, K.~J., Mayorov, N.,
  Nelson, A. R.~J., Jones, E., Kern, R., Larson, E., Carey, C.~J., Polat, I.,
  Feng, Y., Moore, E.~W., VanderPlas, J., Laxalde, D., Perktold, J., Cimrman,
  R., Henriksen, I., Quintero, E.~A., Harris, C.~R., Archibald, A.~M., Ribeiro,
  A.~H., Pedregosa, F., van Mulbregt, P., and {SciPy 1.0 Contributors} (2020).
\newblock {SciPy} 1.0: {Fundamental} {Algorithms} for {Scientific} {Computing}
  in {Python}.
\newblock {\em Nature Methods}, 17:261--272.

\bibitem[Yang and Berger, 1996]{yang:96}
Yang, R. and Berger, J. (1996).
\newblock {\em A Catalog of Noninformative Priors}.
\newblock Discussion papers. Institute of Statistics and Decision Sciences,
  Duke University.

\bibitem[Yu et~al., 2023]{Yu2023}
Yu, H., Hartmann, M., Williams, B., and Klami, A. (2023).
\newblock Scalable {Stochastic} {Gradient} {Riemannian} {Langevin} {Dynamics}
  in {Non}-{Diagonal} {Metrics}.
\newblock {\em Transactions on Machine Learning Research}.

\bibitem[Zhang et~al., 2022]{Zhang2022}
Zhang, L., Carpenter, B., Gelman, A., and Vehtari, A. (2022).
\newblock Pathfinder: {Parallel} quasi-{Newton} variational inference.
\newblock {\em Journal of Machine Learning Research}, 23(306):1--49.

\end{thebibliography}





\section*{Checklist}

 \begin{enumerate}

 \item For all models and algorithms presented, check if you include:
 \begin{enumerate}
   \item A clear description of the mathematical setting, assumptions, algorithm, and/or model. [Yes]
   \item An analysis of the properties and complexity (time, space, sample size) of any algorithm. [Yes]
   \item (Optional) Anonymized source code, with specification of all dependencies, including external libraries. [Yes]
 \end{enumerate}

 \item For any theoretical claim, check if you include:
 \begin{enumerate}
   \item Statements of the full set of assumptions of all theoretical results. [Yes]
   \item Complete proofs of all theoretical results. [Yes]
   \item Clear explanations of any assumptions. [Yes]     
 \end{enumerate}

 \item For all figures and tables that present empirical results, check if you include:
 \begin{enumerate}
   \item The code, data, and instructions needed to reproduce the main experimental results (either in the supplemental material or as a URL). [Yes]
   \item All the training details (e.g., data splits, hyperparameters, how they were chosen). [Yes]
         \item A clear definition of the specific measure or statistics and error bars (e.g., with respect to the random seed after running experiments multiple times). [Yes]
         \item A description of the computing infrastructure used. (e.g., type of GPUs, internal cluster, or cloud provider). [Yes]
 \end{enumerate}

 \item If you are using existing assets (e.g., code, data, models) or curating/releasing new assets, check if you include:
 \begin{enumerate}
   \item Citations of the creator If your work uses existing assets. [Yes]
   \item The license information of the assets, if applicable. [Yes]
   \item New assets either in the supplemental material or as a URL, if applicable. [Not Applicable]
   \item Information about consent from data providers/curators. [Not Applicable]
   \item Discussion of sensible content if applicable, e.g., personally identifiable information or offensive content. [Not Applicable]
 \end{enumerate}

 \item If you used crowdsourcing or conducted research with human subjects, check if you include:
 \begin{enumerate}
   \item The full text of instructions given to participants and screenshots. [Not Applicable]
   \item Descriptions of potential participant risks, with links to Institutional Review Board (IRB) approvals if applicable. [Not Applicable]
   \item The estimated hourly wage paid to participants and the total amount spent on participant compensation. [Not Applicable]
 \end{enumerate}

 \end{enumerate}

\onecolumn
\aistatstitle{Riemannian Laplace Approximation with the Fisher Metric: \\ Supplementary Materials}

\setcounter{theorem}{0}

\section{CONTENTS}

This supplement complements the main article, by providing the proofs for the theorems, presenting details about the experiments as well as some additional empirical results, and expanding the theoretical basis. For clarity, we number the Sections, Equations \textit{etc.} with a range that does not overlap with the main paper.

Section~\ref{sec:proofs} provides the proofs of the theorems stated in the main paper. Section~\ref{sec:hausdorffdetails} justifies the usage of Hausdorff MAP along with Fisher precision. Section~\ref{sec:coverage} extends the theoretical discussion of the approximation in general, by providing a theorem characterising the coverage of \lafamilyabbrv. Section~\ref{sec:geodesic-ode} provides background on numerically solving ODEs and demonstrates how to formulate solving the exponential and logarithmic maps as ODE problems. Section~\ref{sec:auto-christoffel} demonstrates how to obtain the Christoffel symbols based on a differentiable expression of the Riemannian metric using JAX \citep{jax:2018}, and Section~\ref{sec:christoffel-models} provides additional details concerning the computation of the Fisher metric and the resulting Christoffel symbols for certain models. Section~\ref{sec:expdetails} provides all sorts of details for the empirical experiments reported in the main article. Finally, Section~\ref{sec:extraexperiments} provides additional experimental results, for instances visual illustrations that were omitted from the  main paper due to space constraints as well as demonstration of the methods for one more target distribution. 

Code for reproducing the experiments is available at \url{https://github.com/ksnxr/RLAF}.

\section{PROOFS}
\label{sec:proofs}
\subsection{Theorem 1}
\label{sec:proof1}

\begin{theorem}
    For Gaussian (or uniform) prior and Gaussian likelihood with fixed covariance, the Fisher metric is constant.
\end{theorem}
\begin{proof}
Consider the prior distribution on the parameters of interest. We consider two cases, the first where the prior is Gaussian with a fixed covariance and the second where the prior is uniform. For the first case, we have $\btheta \sim \N(\bmu,\bS)$ for some $\bmu$ and $\bS$. Therefore,
\begin{equation*}
    \log \pi(\btheta) = -\frac{k}{2}\log(2\pi) - \frac{1}{2}\log\det(\bS) - \frac{1}{2}(\btheta-\bmu)^{\top}\bS^{-1}(\btheta-\bmu),
\end{equation*}
and
\begin{equation*}
    \pdv[2]{\log \pi(\btheta)}{\btheta} = -\bS^{-1}.
\end{equation*}
The negative Hessian is therefore constant.
For the second case, where the prior is uniform, the negative Hessian is $0$ everywhere, thus is also constant.

Consider the Gaussian distribution $\bY|\btheta \sim \N(\btheta, \bSigma)$ for some fixed $\bSigma$. We have
\begin{equation*}
    \pdv[2]{\log \pi(\by|\btheta)}{\btheta} = -\bSigma^{-1},
\end{equation*}
which does not depend on $\by$, thus the FIM is constant. 

Therefore, the resulting Fisher metric is constant.
\end{proof}

\subsection{Theorem 2}

\begin{theorem}
When the posterior becomes a Gaussian in the limit of infinite data, \ourabbrv\ is exact when the likelihood has a parameterization where the observed Fisher coincides with the expected Fisher, e.g. for distributions in the exponential family.
\end{theorem}

\begin{proof}
In the limit of infinite data the contribution of the prior diminishes.

All distributions in the exponential family can be written using canonical parameters \citep{Murphy2023}, in which case we have
\begin{align*}
\log p(\bx|\btheta) &= \log h(\bx) + \btheta^{\top} T(\bx) - A(\btheta),\\
\nabla_{\btheta} \log p(\bx|\btheta) &= T(\bx) - \nabla_{\btheta}A(\btheta),\\
\nabla_{\btheta}^{2} \log p(\bx|\btheta) &= -\nabla_{\btheta}^{2}A(\btheta) = \E_{p(\bx|\btheta)}\left[\nabla_{\btheta}^{2} \log p(\bx|\btheta)\right],
\end{align*}
which implies that the observed Fisher information coincides with the expected Fisher information for all parameter values.

Consider the case where \ourabbrv\ is formed under the canonical parameters, in which case it is equivalent to using as metric the observed Fisher information. Since the observed Fisher converges to a fixed value, the resulting approximation becomes \euclaabbrv.

Due to the isometry between the manifold induced by different parameterizations, \ourabbrv\ is also exact when the likelihood is under different parameterizations; see Section~\ref{sec:proof3} for further discussions.
\end{proof}

\subsection{Theorem 3}
\label{sec:proof3}

\begin{theorem}
    With an invariant prior, e.g. Jeffreys prior, 
    \ourinvabbrv\ with Hausdorff MAP is exact for probabilistic models whose target distributions are diffeomorphic with Gaussians, for which the negative Hessian at the Hausdorff MAP coincides with the Fisher metric.
\end{theorem}

\begin{proof}
Consider first the simple case with no transformations, such that the probabilistic model reduces to the following form
\begin{align*}
\pi(\bmu | \bpsi) &= \N(\bmu | \bpsi, \bS), \\
\pi(\bpsi) &= \text{Jeffreys}(\bpsi).
\end{align*}
The FIM for $\bpsi$ is constant, given by $\bS^{-1}$. Therefore, The prior corresponds to a uniform distribution in the Euclidean sense, and the posterior is given by
\begin{equation*}
\pi(\bpsi | \bmu) = \N(\bpsi | \bmu, \bS).
\end{equation*}
It is clear that \ourinvabbrv\ is exact for this posterior.

Consider now a diffeomorphic transformation $\bphi: \bpsi \rightarrow \btheta$. It is well known that the Fisher Information Matrix automatically transforms according to the correct transformation rule of a Riemannian metric \citep{Martens2020,Kristiadi2023}. The distribution of the transformed parameters is given by $\pi(\btheta | \bmu) =  
    \pi(\bphi^{-1}(\btheta)  |\bmu)\left|\det \pdv{\bphi^{-1}}{\btheta}\right|$. It coincides with the posterior of the probabilistic model
\begin{align*}
\pi(\bmu | \btheta) &= \N(\bmu | \bphi^{-1}(\btheta), \bS), \\
\pi_{J}(\btheta) &= \sqrt{\det\bG_{\bTheta}(\btheta)},
\end{align*}
since Jeffreys prior accounts for the change of variable of the transformation,
\begin{align*}
    \pi_{J}(\btheta) =\sqrt{\det \bG_{\bTheta}(\btheta)} =  \left|\det \pdv{\bphi^{-1}}{\btheta}\right|\sqrt{\det\bG_{\bPsi}(\bphi^{-1}(\btheta))}.
\end{align*}
The parameters that give the maximum value under the Hausdorff measure $ {\pi(\btheta|\bmu)} (\det \bG_{\bTheta}(\btheta))^{-\tfrac{1}{2}} \propto \pi(\bmu|\bphi^{-1}(\btheta))$ are always given by the MLE estimate $\hat\btheta = \bphi(\bmu )$.

The argument for exactness is as follows.  
Consider the two manifolds $\cM: (\bPsi, \bG_{\bPsi})$ and $\cN: (\bTheta, \bG_{\bTheta})$. Since the corresponding metrics satisfy the transformation rule
\begin{equation}
\bG_{\bTheta} = \left(\pdv{\bpsi}{\btheta}\right)^{\top}\bG_{\bPsi} \pdv{\bpsi}{\btheta}, \label{eq::metric_transformation}
\end{equation}
there exists an isomorphism between the tangent spaces of $\cM$ and $\cN$. As a result, the exponential map on $\cM$ and $\cN$ and the tangent vectors at the center of the distribution transform correctly. Since \ourinvabbrv\ is exact for the distribution corresponding to $\cM$, it is also exact for the distribution corresponding to $\cN$.

Moreover, the negative Hessian at the Hausdorff MAP coincides with the Fisher metric. This can be seen from the transformation rule of Hessian matrix under a transformation given by $\btheta = \bphi(\bpsi)$
\begin{equation*}
\nabla_{\btheta}^{2}f = \left(\pdv{\bpsi}{\btheta}\right)^{\top}\nabla_{\bpsi}^{2}f \left(\pdv{\bpsi}{\btheta}\right) + [\nabla_{\btheta}^{2}\bphi^{-1}_{1}(\btheta),\dots ,\nabla_{\btheta}^{2}\bphi^{-1}_{n}(\btheta)]\left(\nabla_{\bpsi}f\otimes \bI_{n}\right),
\end{equation*}
because for the MAP estimate, $\nabla_{\bpsi}f = 0$.
\end{proof}

\section{USING HAUSDORFF MAP AND FISHER PRECISION}
\label{sec:hausdorffdetails}

In Section~\ref{sec:hausdorff-map}, we briefly discussed the choices of using Hausdorff MAP together with Fisher precision. Here we provide additional remarks on this.

We use notations inspired by \citet{Kristiadi2023}. Consider a reparametrization $\bphi : \bpsi \rightarrow \btheta$. Denote its Jacobian as $\bJ$, with
\begin{equation*}
J_{ij} = \frac{\partial\theta_{i}}{\partial\psi_j}.
\end{equation*}
As noted by \citet{Kristiadi2023}, the transformation rule of vector components is
\begin{equation*}
\bbv_{\bTheta} = \bJ \bbv_{\bPsi},
\end{equation*}
the transformation rule of covector components is
\begin{equation*}
\bw_{\bTheta} = \bJ^{-\top} \bw_{\bPsi},
\end{equation*}
and the transformation of the components of a Riemannian metric, which is a $(0,2)$ tensor, follows
\begin{equation*}
\bG_{\bTheta} = \bJ^{-\top}\bG_{\bPsi} \bJ^{-1}.
\end{equation*}

Recall that the FIM transforms as a Riemannian metric under one-to-one mappings \citep{Martens2020,Kristiadi2023}.
Therefore, with an invariant prior or ignoring the effect of the prior, the resulting metric transforms automatically.

Typical MAP estimates are not invariant under reparametrization. However, we can define an invariant MAP estimate by taking into account the Riemannian structure \citep{Kristiadi2023}. Specifically, instead of calculating the maximum of the PDF under Lebesgue measure, we calculate it under the Hausdorff measure. Denote the PDF under Lebesgue measure as $\pi_{\bTheta}(\btheta)$, and the PDF under Hausdorff measure as $\pi_{\bTheta}^{\bG}(\btheta)$. Then, using the Riemannian volume form \citep{Lee2018}, we have
\begin{equation*}
    \pi_{\bTheta}^{\bG}(\btheta) = \frac{\pi_{\bTheta}(\btheta)}{\sqrt{\det \bG_{\bTheta}(\btheta)}}.
\end{equation*}
\begin{theorem}
    The PDF under Hausdorff measure is invariant across reparametrization.
\end{theorem}
\begin{proof}
\begin{align*}
    \pi_{\bTheta}^{\bG}(\btheta) &= \frac{\pi_{\bTheta}(\btheta)}{\sqrt{\det\bG_{\bTheta}(\btheta)}} = \frac{\pi_{\bPsi}(\bphi^{-1}(\btheta))\left\vert{\det \bJ^{-1}(\btheta)}\right\vert}{\sqrt{ 
\det( \bJ^{-T}(\btheta)\bG_{\bPsi}(\bphi^{-1}(\btheta))\bJ^{-1}(\btheta))}}\\
    & = \frac{\pi_{\bPsi}(\bphi^{-1}(\btheta))}{\sqrt{\det\bG_{\Psi}(\bphi^{-1}(\btheta))}}= \pi_{\bPhi}^{\bG}(\bphi).
\end{align*}
\end{proof}
This naturally leads to the following corollary
\begin{corollary}
    The Hausdorff MAP is invariant across reparameterization.
    \label{prop:hausdorff-map}
\end{corollary}

The above result, along with the proof, can also be found in \citet{Kristiadi2023}. The following theorem shows that we can ignore the first order gradients in the Taylor series expansions as in Equation~\eqref{eq:riemlap} of the main paper
\begin{theorem}
Under normal coordinates, the Hausdorff MAP is a critical point under Lebesgue measure.
\end{theorem}
\begin{proof}
It is a standard result in Riemannian geometry that for a Riemannian manifold, at each tangent space there exists a normal coordinate \citep{Lee2018}, such that the components of the metric at the point form an identity matrix, and all first partial derivatives of the metric vanishes at the point \citep[see Proposition 5.24 from][]{Lee2018}. Therefore, consider such a transformation to the normal coordinate $\bphi: \bpsi\rightarrow\btheta$. $\hat{\btheta}=\bphi(\hat{\bpsi})$ is the Hausdorff MAP due to Theorem~\ref{prop:hausdorff-map}. 
Recall the definition of Hausdorff MAP; it is clear that $\pdv{\pi^{\bG}(\hat{\btheta})}{\theta_i} = 0$.

Moreover, since the first partial derivatives of the metric vanishes under a normal coordinate, recall $\partial(\det(\bX))=\det(\bX)\Tr(\bX^{-1}\partial\bX)$ \citep{Peterson2012}, we have $\pdv{\sqrt{\det\bG(\hat{\btheta})}}{\theta_i} = 0.$
Therefore,
\begin{equation*}
\pdv{\pi(\hat{\btheta})}{\theta_i} = \pdv{\pi^{\bG}(\hat{\btheta})\sqrt{\det \bG(\hat{\btheta})}}{\theta_i}= \pdv{\pi^{\bG}(\hat{\btheta})}{\theta_i}\sqrt{\det \bG(\hat{\btheta})}+\pi^{\bG}(\hat{\btheta})\pdv{\sqrt{\det \bG(\hat{\btheta})}}{\theta_i}=0.
\end{equation*}
\end{proof}
Under an invariant prior or ignoring the effect of prior, the Fisher metric transforms as a Riemannian metric, and \ourabbrv \ with Hausdorff MAP and Fisher precision can be interpreted as forming the approximate Gaussian with Fisher metric in place of the negative Hessian at $\hat{\btheta}$ under normal coordinates and transforming the approximation back to the current approximation.

The following theorem justifies using the Fisher metric as the precision of the Gaussian approximation from a differential geometry viewpoint; using the metric to form the precision has been explored before, e.g. in \citet{Mathieu2019}.
\begin{theorem}
When the precision of a Gaussian distribution is a Riemannian metric, the samples from the Gaussian follow the transformation rule of the tangent vectors.
\end{theorem}

\begin{proof}
The transformation of the inverse of the Riemannian metric follows
\begin{equation*}
\bG^{-1}_{\bTheta} = \bJ \bG^{-1}_{\bPsi} \bJ^{\top} = (\bJ\bL_{\bPsi})(\bJ\bL_{\bPsi})^{\top},
\end{equation*}
where $\bL$ indicates the decomposition of the inverse of the metric. We therefore have
\begin{equation*}
\bL_{\bTheta} = \bJ\bL_{\bPsi},
\end{equation*}
and the samples from the resulting Gaussian distribution transform in the same way as vector components.
\end{proof}

\section{COVERAGE OF RLA SAMPLES}
\label{sec:coverage}

An interesting insight is that, under mild assumptions, we can obtain samples over the entire $\Real^{D}$ space for any metric and any starting point.
\begin{theorem}
    For target distributions of unconstrained parameters, if the induced Riemannian manifold $(\bTheta, g)$ is connected, there always exists one length-minimizing geodesic starting from the MAP and passing through an arbitrary point $\btheta$ on the induced manifold.
\label{existence}
\end{theorem}
\begin{proof}
Under unconstrained parametrization, for a point $p$ on the connected Riemannian manifold $\bTheta$, the geodesics are defined on the entire tangent space. Therefore, using Lemma 6.18 from \citep{Lee2018}, for any two points $\bp, \bq \in \bTheta$, there exists one length-minimizing geodesic.
\end{proof}

\section{BIAS OF \bergaminabbrv}
\label{sec:sup-bias-rlab}

Consider a $D$ dimensional Gaussian distribution with mean $\bzero$ and covariance $\bI_{D}$. Isotropic multivariate Gaussian distributions are known to be rotation invariant \citep{Vershynin2018}, and the distribution is the same from all directions at the origin. Moreover, the distances from the samples to the origin follow a $\chi$ distribution by definition. Also hinted by \citet{Vershynin2018}, it is clear that we can represent samples from an isotropic multivariate Gaussian distribution of dimension $D$ such that the directions are sampled uniformly at random from a sphere, and the distances to the origin sampled from a $\chi$ distribution with degrees of freedom $D$. In this case, for \bergaminabbrv, an algorithm based on the Riemannian metric using the gradient information of the distribution, is also rotation invariant, and the geodesics travel in straight lines. As such, we can consider an arbitrary direction without losing generality.

Specifically, we consider the following case
\begin{align*}
\bbv = \begin{bmatrix}
v \\
\vdots \\
v
\end{bmatrix}_{D},
\bx = \begin{bmatrix}
x \\
\vdots \\
x
\end{bmatrix}_{D},
\nabla\ell = \begin{bmatrix}
-x \\
\vdots \\
-x
\end{bmatrix}_{D},
\bbv_{0} = \begin{bmatrix}
v_{0} \\
\vdots \\
v_{0}
\end{bmatrix}_{D},
\end{align*}
where a sample with initial velocity $\bbv_{0}$ reaches position $\bx$, at which point the gradient is given by $\nabla\ell$ and the sample has velocity $\bbv$, with every occurrence of the same letter denoting the same value, and $v$, $x$ and $v_{0}$ being non-negative. It is clear that all particles with $\bbv_{0}$ of the above form reach $\bx(t)$ at some time. In the following analysis, we assume that all the samples reach the point $\bx$ within time $1$ for the value of $v_{0}$ considered for simplicity; this generally holds for sufficiently small $x$. Recall that $\norm{\bbv(t)}^{2} + \langle\bbv(t),\nabla\ell\rangle^{2} = \norm{\bbv_{0}}^{2}$ due to the property of the geodesic,
we have
\begin{align*}
Dv^{2} + \left(Dvx\right)^{2} &= Dv_{0}^{2},\\
v^{2} + Dx^{2}v^{2} &= v_{0}^{2},\\
v^{2} &= \frac{v_{0}^{2}}{1+Dx^{2}}.
\end{align*}
As such, for fixed $v_{0}$ and $x$, with \bergaminabbrv, as $D$ increases, $v$ decreases. However, under the Euclidean metric, given $v_{0}$ and $x$, $v$ is constant for all $D$. This demonstrates that, for a fixed $v_{0}$, the final distance traveled by the sample decreases as $D$ increases, and, in a sense, the resulting algorithm leads to additional bias as $D$ increases.

Additionally, we can consider the expected value of $v$ at a given position $x$ under $v_{0}$. We can define the distribution of $v_{0}$ such that it induces the correct distribution of distances to the origin, in which case $\sqrt{D}v_{0}$ follows a $\chi$ distribution with degrees of freedom $D$. We have
\begin{equation*}
\E\left[v\right] = \E\left[\sqrt{\frac{v_{0}^{2}}{1+Dx^{2}}}\right] = \E\left[\frac{v_{0}}{\sqrt{1+Dx^{2}}}\right].
\end{equation*}
Consider a positive $x$. It is clear that $\frac{1}{\sqrt{1+Dx^{2}}}$ decreases as $D$ increases. For smaller $D$, the expected velocity as calculated above may increase as $D$ increases, since $\E\left[v_{0}\right]$ itself may increase. However, as $D$ approaches infinity, $\E\left[\sqrt{D}v_{0}\right]$ approaches $\sqrt{D-1}$, $\E\left[v_{0}\right]$ approaches $1$ and $\E\left[v\right]$ approaches $0$.

The above theoretical analysis complements the empirical evidence as shown in Figure~\ref{fig:distances}, which demonstrates that the bias, as measured by the Wasserstein distances from approximate samples to true samples computed for the first dimension, grows continuously as $D$ increases from $1$ to $10$.

\section{GEODESIC ODE}
\label{sec:geodesic-ode}

In this section, we provide some background on the numerical solutions of ODEs, while also demonstrating the formulations of exponential map and logarithmic map as ODE problems. Unless otherwise stated, the technical contents are based on \citet{Atkinson2009}.

Generally, an ODE system can be written in the form of
\begin{equation*}
\dv{\bx(t)}{t} = \bbf(t, \bx(t)),
\end{equation*}
where $\bx$ is an unknown function dependent on time $t$, and $\bbf$ is a function describing the change in $\bx$.

The geodesic equation induces the following ODE system
\begin{equation*}
\dv{[\btheta(t), \bbv(t)]}{t} = [\bbv(t), \ba(t)],
\end{equation*}
where $\ba(t)$ is given by the geodesic equation, such that
\begin{equation*}
a^{k}(t)=-v^{i}(t)v^{j}(t)\Gamma^{k}_{ij}(t).
\end{equation*}
It is a nonlinear second order ODE.

An initial value problem (IVP) solves for $\bx(b)$ given the following conditions
\begin{align*}
\dv{\bx(t)}{t} &= \bbf(t, \bx(t)),\quad t_{0}\leq t\leq b,\\
\bx(t_0) &= \bx_0.
\end{align*}
Probably the simplest solver is Euler's method, which divides the times into a discrete set of nodes $t_{0} < t_{1} < t_{2} < \dots < t_{N} \leq b$, and iteratively solves it as
\begin{equation*}
\bx(n+1) = \bx(n) + h \bbf(t_{n},\bx_{n}).
\end{equation*}
However, this is a naive method, and often leads to large integration errors. In practice, one commonly used solver is Dormand-Prince's $5 (4)$ method \citep{Dormand1980}; this is the default option in SciPy \citep{Virtanen2020}. It generally requires $6$ function evaluations per step.

A general nonlinear two-point boundary value problem (BVP) can be formulated as 
\begin{align*}
\bx''(t) &= \bbf(t,\bx(t), \bx'(t)),\quad a<t<b,\\
\A \begin{bmatrix}
    \bx(a) \\
    \bx'(a)
\end{bmatrix} &+ \B \begin{bmatrix}
    \bx(b) \\
    \bx'(b)
\end{bmatrix} = \begin{bmatrix}
    \bgamma_{1} \\
    \bgamma_{2}
\end{bmatrix},
\end{align*}
where $\A$ and $\B$ are square matrices, and the conditions on the second line are known as the boundary conditions.

The logarithmic map of the geodesic ODE can be formulated as a BVP, by observing that we can make the following choices
\begin{equation*}
\A = \begin{bmatrix}
    \bI & 0 \\
    0 & 0
\end{bmatrix},\; \B = \begin{bmatrix}
    0 & 0 \\
    \bI & 0
\end{bmatrix},\; \bgamma_{1} = \bx_a,\; \bgamma_{2} = \bx_b,
\end{equation*}
and can thus be solved by a BVP solver.

\section{CHRISTOFFEL SYMBOLS BASED ON EXPRESSION OF THE RIEMANNIAN METRIC}
\label{sec:auto-christoffel}

With modern autodiff frameworks, it is possible to directly calculate the Christoffel symbols for a given $\btheta$, given that we have a differentiable expression of the Riemannian metric. As noted by \citet{Song2018}, we only need the quantities in the form of
\begin{equation*}
\Gamma^{k}_{ij}v^{i}v^{j}.
\end{equation*}
We provide a JAX \citep{jax:2018} implementation to obtain the above quantities as below. \footnote{In all presented experiments \textit{jax.numpy.linalg.solve} is used instead of \textit{jax.scipy.linalg.solve}. However, the implementation in \textit{jax.scipy} could make use of the positive definite structure of the matrix, theoretically leading to faster inversions, and we thus present this implementation.}

\begin{lstlisting}
import jax
import jax.numpy as jnp
import jax.scipy as jsp


def christoffel_fn(g, theta, v):
    # Adapted based on ChatGPT
    d_g = jax.jacfwd(g)(theta)

    # Compute the Christoffel symbols
    partial_1 = jnp.einsum("jli,i,j->l", d_g, v, v)
    partial_2 = jnp.einsum("ilj,i,j->l", d_g, v, v)
    partial_3 = jnp.einsum("ijl,i,j->l", d_g, v, v)
    result = jsp.linalg.solve(g(theta), 0.5 * (partial_1 + partial_2 - partial_3),
    assume_a="pos")
    
    return result
\end{lstlisting}

In the above code, \textit{theta} and \textit{v} are the current position and velocity, and \textit{g} is a differentiable function that returns the Riemannian metric given the position. It is possible to analytically derive the Christoffel symbols given the expression of the metric. We observe that using analytical expressions can lead to more efficient and numerically stable results. We show some analytical derivations in Section~\ref{sec:christoffel-models}. However, we observed that for small scale problems, using the numerical results are often fast and accurate enough.

\section{FISHER METRIC AND CHRISTOFFEL SYMBOL FOR SPECIFIC MODELS}
\label{sec:christoffel-models}

\subsection{Logistic regression}

In Section~\ref{sec:blr}, we considered Bayesian logistic regression. Here we present some derivations on the Christoffel symbols of Fisher metric for that problem.

For logistic regression, the partial derivatives of the Fisher metric are given by \citep{Girolami2011}
\begin{equation*}
\pdv{\bG(\btheta)}{\theta_{i}} = \bX^{\top}\bLambda \bV^{i} \bX,
\end{equation*}
where, using $s_{k}$ to denote $s(\bX_{k}\btheta)$, we have $\Lambda_{kk} = s_{k}(1-s_{k})$ and $V^{i}_{kk} = (1-2s_{k})X_{ni}$.

One interesting property is the following \citep{Lan2015}
\begin{equation*}
\pdv{(\bG(\btheta))_{jl}}{\theta_{i}} = \pdv{(\bG(\btheta))_{il}}{\theta_{j}} = \pdv{(\bG(\btheta))_{ij}}{\theta_{l}}.
\end{equation*}

In order to see that, note that
\begin{align*}
\pdv{(\bG(\btheta))_{jl}}{\theta_{i}} &= (\bX^{\top}\bLambda\bV^{i}\bX)_{jl}= \sum_{k}X_{kj}(\bLambda\bV^{i})_{kk}X_{kl}= \sum_{k}X_{kj}s_{k}(1-s_{k})(1-2s_{k})X_{ki}X_{kl}\\
&= \pdv{(\bG(\btheta))_{il}}{\theta_{j}}= \pdv{(\bG(\btheta))_{ij}}{\theta_{l}}.
\end{align*}

Therefore, the resulting Christoffel symbols have a relatively simple form, given by
\begin{equation*}
\Gamma^{m}_{ij} = \frac{1}{2}G^{ml}\left(\partial_{i}G_{jl} + \partial_{j}G_{il} - \partial_{l}G_{ij} \right) = \frac{1}{2}G^{ml}\partial_{i}G_{jl}.
\end{equation*}

\subsection{Neural Networks (NN)}
\label{sec:nn-pullback}

Here we expand the discussions on FIM for NNs in Section~\ref{sec:comp-fisher} of the main paper.

Recall the probabilistic model $\pi(\by|\bphi)$ in its basic form. Given an NN that maps $\mathcal{X}$ $\overset{f_{\btheta}}{\mapsto} \bPhi$ where $\bPhi$ is a $P$-dimensional parameter space with given Fisher metric, we can define a Fisher metric on the parameter space of the NN as follows. Using the chain rule for the Hessian matrix of $\ell_{\by}(\btheta)$ w.r.t $\btheta$, observe that the extra expression involving the score function will be zero in expectation \citep[Definition 2.78, p. 111]{schervish:2012}. Therefore, the metric tensor w.r.t $\btheta$ reads 
\begin{align*}
    \E(-\nabla^2 \ell_{\bY}(\btheta)_{i, j})
    & = \E \left( -\sum_n^N \partial^2_{i, j} \log \pi \big(Y_n| \bphi_n ) \right) \nonumber \\
    & = \sum_n^N \sum_p^P \sum_q^P \E \left( - \partial^2_{q, p} \log \pi \big(Y_n| \bphi_n) \right) \partial_i(\bphi_n)_p\partial_j(\bphi_n)_q.\nonumber \\
    & = \sum_n^N \sum_p^P \sum_q^P \bG_{q, p}(\bphi_n)  \partial_i(\bphi_n)_p\partial_j(\bphi_n)_q \nonumber \\
    & = \sum_n^N \sum_p^P \sum_q^P \bG_{q, p}(f_{\btheta}(\bx_n))  \partial_i(f_{\btheta}(\bx_n))_p\partial_j(f_{\btheta}(\bx_n))_q,
\end{align*}
where the expectation taken w.r.t $\bY$ is over the Hessian of $\ell_{\bY}(\btheta)$ on $\bphi$ (in between 2nd and 3rd step) and that comes for granted once the Fisher metric is known in the basic form of the probabilistic model $\pi(\by|\bphi)$. 
In the last passage we see that $\partial_i(f_{\btheta}(\bx_n))_p$ is the Euclidean derivative with respect to $\btheta_i$ of the $p^{th}$ output of the neural network $f$ evaluated at the $n^{th}$ input. Then using matrix notations we write, 
\begin{align*}
    \bG(\btheta) = \sum_{n}^N \boldsymbol{J}_n^\top \bG_{\bPhi}(f_{\btheta}(\bx_n)) \boldsymbol{J}_n,
\end{align*}
where $\boldsymbol{J}_n$ $=$ $[\nabla (\boldsymbol{\bphi}_n)_1 \cdots \nabla (\boldsymbol{\bphi}_n)_ P]^\top$ denotes the $P \times D$ Jacobian matrix of the NN at the $n^{th}$ input. 

Now, it is clear from the above form that its structure resembles that of a pullback metric in Riemannian manifolds. To see this fix the inputs (or covariates) in the input space and look instead to the output of the NN as a map $\bTheta \overset{h_n}{\rightarrow} \bPhi$ with $h_n(\btheta) = f_{\btheta}(\bx_n) \ \forall n$. Once we have endowed a metric $g_{\bPhi}$ on $\bPhi$ and defined a map for every $n$, the pullback metric on $\bTheta$ given all the inputs can be formally written as the sum of all input information related to $\btheta$. That is, 
\begin{align*}
    g_{\bTheta}(\bbu, \bbv) = \sum_n^N g_{\bPhi} \big( \mathrm{d}h_n(\btheta)[\bbu], \mathrm{d}h_n(\btheta)[\bbv] \big),
\end{align*}
where $\mathrm{d}h_n(\btheta)[\cdot] : T_{\btheta}\bTheta \rightarrow T_{\bphi_n}\bPhi$ is the differential at a point $\btheta \ \forall n$. Since $\bphi_n = f_{\btheta}(\bx_n)$ can be seen as a function of $\btheta$, the differential takes the form $\mathrm{d}h_n(\btheta)[\bbv] = \boldsymbol{J}_n \bbv$. Plugging this into the expression above and using that the matrix of coefficients of the metric $g_{\bPhi}$ is $\bG_{\bphi}$, we obtain the metric tensor above; similar observation can be found in e.g. \citet{Grosse2022}. Also note that because the function $h_n$ may not be one-to-one for all $n$, $\bx_n \in \mathcal{X}$, the metric tensor $\bG(\btheta)$ may not satisfy the definition of a pullback \citep[see][]{doCarmo1992}. 

\section{EXPERIMENTAL DETAILS}
\label{sec:expdetails}

In this section, we describe the computation environment, some implementation details, the employed methods to find the MAP estimate, to obtain the NUTS samples, to integrate the ODEs and to calculate the Wasserstein distances, and describe the basic setups for several experiments.

\subsection{Computation environment}

Apart from some preliminary ones, all experiments are carried out on the computer cluster using CPUs with type \textit{Intel(R) Xeon(R) Gold 6230 CPU @ 2.10GHz}. When we benchmark the running times of the algorithms, we always use two cores for each task. The detailed configurations differ based on the nature of the jobs.

\subsection{Implementation details}

Sacred \citep{Greff2017} is used to manage the experiments. Code for experiments other than neural network ones are implemented mainly using JAX \citep{jax:2018}, benefiting from fast computations of gradients and other quantities and easy vectorizations. Code for neural network experiments are written mainly in PyTorch \citep{pyt:2019}, using \textit{torch.func} to obtain the gradients and hessian vector products, etc, because we build on the existing library provided by \citet{Daxberger2021}. We always use the default data types of the libraries (\textit{float64} in NumPy \citep{Harris2020} and \textit{float32} in PyTorch and JAX) when initializing the variables. Note that in experiments other than neural network ones, the random numbers are generated by NumPy; In neural network experiments, they are generated by PyTorch.

\citet{Bergamin2023} formed the metric as identity plus the outer product of gradients of negative log-posterior, while \citet{Hartmann2022} formed the metric as identity plus a scaled version of the outer product of gradients of log-posterior. We note that, when the scaling of Monge metric is $1$, these are strictly mathematically equivalent formulations. In experiments other than neural network ones, similar to Monge metric \citep{Hartmann2022}, we form the metric as identity plus the outer product of gradients of log-posterior; for neural network experiments, we form the metric similar to \citet{Bergamin2023}.

We used SymPy \citep{Meurer2017} to verify some derivations, and ChatGPT \citep{ChatGPT2023} as an assistance for writing code, scientific discussions etc.

\subsection{Finding the MAP estimate}

For experiments other than neural network experiments, unless there is an analytical solution of the MAP estimate, we find the MAP estimate by collecting $20$ runs of BFGS optimization with random initialization using the implementation from \citet{Virtanen2020} which yield negative log-posterior values other than plus infinity or NaN, with maximum number of iterations $1e6$. We choose the run that yields the smallest negative log-posterior with \textit{numpy.argmin}.

For neural network experiments, we find the MAP estimate using $lr=1e-2$, $weight\_decay=1e-5$ for $20000$ epochs with the Adam \citep{Kingma2015} optimizer. The prior precision and noise std are optimized post hoc following standard procedure from \citet{Daxberger2021}. During the optimization process, we fix the random seeds, in order to have consistent prior precision and noise std for NUTS.

\subsection{NUTS samples}

We use CmdStanPy \citep{Stan2023} to generate the NUTS samples. Following the default options for reference posterior in posteriordb \citep{Magnusson2022}, we run the NUTS sampler using $10$ chains, a thinning of $10$, a warm-up of $10000$ iterations  and $20000$ iterations per chain. We fix the random seeds for reproducibility, generating a total of $20000$ samples.

For sampling from banana distribution, based on preliminary runs, we set \textit{adapt\_delta} to be $0.95$. We otherwise keep the hyper parameters to be their default values.

\subsection{ODE integration}

For the exponential maps, other than neural network experiments, ODE integrations are carried out using Diffrax \citep{Kidger2021}. Inspired by the default settings of the function \textit{solve\_ivp} in SciPy \citep{Virtanen2020}, we use Dormand-Prince's $5 (4)$ method \citep{Dormand1980}. Diffrax reports the number of steps \textit{num\_steps} used by the integrator. In order to make the results approximately comparable to the number of function evaluations \textit{nfev} reported by SciPy, we multiply \textit{num\_steps} by $6$, the number of function evaluations per step of Dormand-Prince's $5 (4)$ method. We use adaptive step size with initial step size chosen by the algorithm, with $rtol=1e-3$ and $atol=1e-6$, matching the default option of SciPy. We set the maximum number of steps allowed to be $4096$ (approximately equivalent to $24576$ function evaluations), following the default configuration of Diffrax. For neural network experiments, following \citet{Bergamin2023}, we use SciPy's \textit{solve\_ivp} and use the default hyper parameters.

For the logarithmic maps, we implement it using analytical gradients and hessian vector products in NumPy.

In all experiments apart from logistic regression (in which we use analytical gradients), we use the gradients and hessian vector products calculated by the autodiff framework. 

We always use analytical expressions for the Fisher metric. For logistic regression and NN experiments, we use analytical expressions of the Christoffel symbols, while in other experiments the Christoffel symbols are solved numerically based on the metric.

When reporting the numbers of function evaluations, we report the numbers of function evaluations required to solve the exponential maps. We observe that in certain scenarios the number of integration steps can exceed the upper limit. In which case, we treat the number of steps for those samples exceeding the limit as the limit $4096$ while calculating $T$, and calculate the Wasserstein distance of the samples not exceeding the limit.

\subsection{Calculating the Wasserstein distance}

We use Python Optimal Transport \citep{Flamary2021} to calculate the Wasserstein $1$-distance, which can be formulated as \citep{Flamary2021}
\begin{align*}
\min_{\bgamma}\langle \bgamma ,\bM \rangle_{F}, \\
s.t.\quad \bgamma\bones = \ba, \\
\bgamma^{\top}\bones = \bb, \\
\bgamma \geq 0,
\end{align*}
where $\bgamma$ is the transportation plan, $\bM$ is the distance matrix induced by the Euclidean metric for the two sets of samples, and $\ba$ and $\bb$ are uniform numbers summing up to $1$ with the lengths the same as the numbers of respective samples. The maximum number of iterations is set to be $1e10$.

Wasserstein distance is a proper distance metric, and it becomes zero when the two distributions are the same.

\subsection{Squiggle}

The experiment presented in Section~\ref{sec:theoretical-basis} was carried out using the following hyper parameters on the squiggle distribution
\begin{equation*}
a = 1.5,\; \bS = \begin{bmatrix}
    5.0 & 0.0 \\
    0.0 & 0.05
\end{bmatrix}.
\end{equation*}

\subsection{Banana distribution}

Results for sampling from banana distribution were presented in Section~\ref{sec:banana}. Originally proposed by Luke Bornn and Julien Cornebise as part of the discussions on \citet{Girolami2011}, the Fisher metric for banana distribution can also be found in \citet{Brofos2021}, and is given by
\begin{equation*}
\bG(\theta_1, \theta_2) = \begin{bmatrix}
    \frac{1}{\sigma_{\btheta}^{2}}+\frac{N}{\sigma_{\by}^{2}} && \frac{2 N \theta_2}{\sigma_{\by}^{2}} \\
    \frac{2N \theta_2}{\sigma_{\by}^{2}} && \frac{1}{\sigma_{\btheta}^2}+\frac{4 N \theta_2^2}{\sigma_{\by}^{2}}
\end{bmatrix}.
\end{equation*}

\subsection{Bayesian logistic regression}

We report the details on the $5$ logistic regression datasets as used in Section~\ref{sec:blr}, covering the dimensionality of parameter of interest $D$ and the number of data points $N$.

\textit{Ripley} \citep{Ripley1994,Ripley1996}, abbreviated as \textit{Ripl}: $D=3$, $N=250$.

\textit{Pima} \citep{Smith1988}: $D=8$, $N=532$.

\textit{Heart} \citep{Feng1992}, abbreviated as \textit{Hear}: $D=14$, $N=270$.

\textit{Australian} \citep{Feng1992}, abbreviated as \textit{Aust}: $D=15$, $N=690$.

\textit{German} \citep{Feng1992}, abbreviated as \textit{Germ}: $D=25$, $N=1000$.

\subsection{Neural networks}

In Section~\ref{sec:nn-regression}, we performed experiments on NN regression tasks. Similar to \citet{Bergamin2023}, we use the dataset from \citet{Snelson2005}, which has $200$ data points with labels, with the $x$ coordinates distributed between $0$ and $6$.

For \textit{complete} training set, a test set of size $50$ is selected from the $200$ data points with labels with random seed $1$. As mentioned in the main paper, the \textit{gap} training set is constructed by using data points with labels with $x$ coordinates not between $1.5$ and $3.0$.

\section{ADDITIONAL EXPERIMENTAL RESULTS}
\label{sec:extraexperiments}

\subsection{Samples from banana distribution}

In Section~\ref{sec:banana}, we consider the problem of sampling from banana distribution. We reported the numerical qualities of the results in Table~\ref{tbl:banana}, and the different MAP estimates and the resulting approximations using \ourabbrv\ in Figure~\ref{fig:banana-maps}. Here we provide additional plots that show how all other variants work for the same problem.

In the Banana distribution, there are two Euclidean MAPs, with approximately the same log-posterior and approximately symmetric on the two sides of axis. 

The results of starting from one Euclidean MAP and the Hausdorff MAP are shown in Figure~\ref{fig:banana}. We report the results of sampling from one particular Euclidean MAP; starting from the other Euclidean MAP would simply (approximately) mirror the approximation.

\begin{figure}[t]
    \centering
    \begin{tabular}{cccc}
        \includegraphics[width=0.2\textwidth]{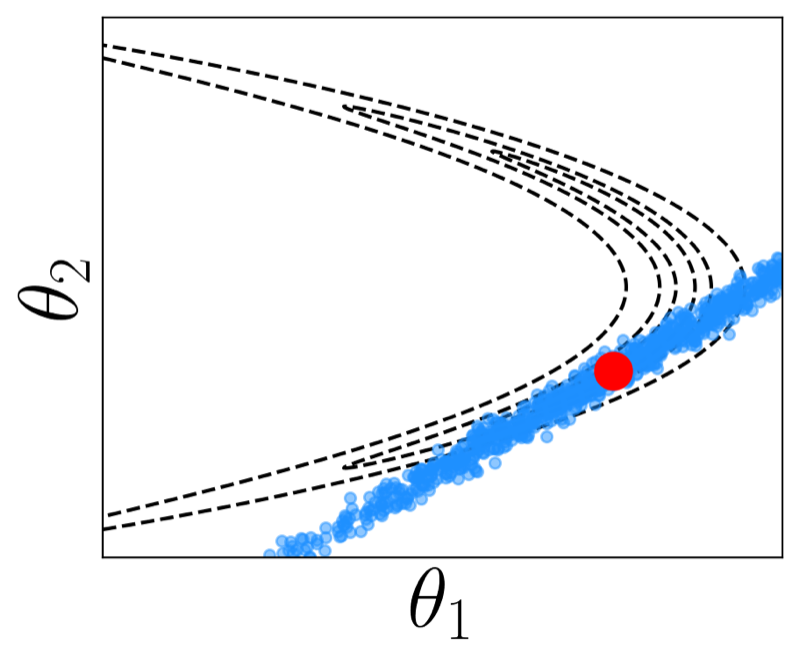} & \includegraphics[width=0.2\textwidth]{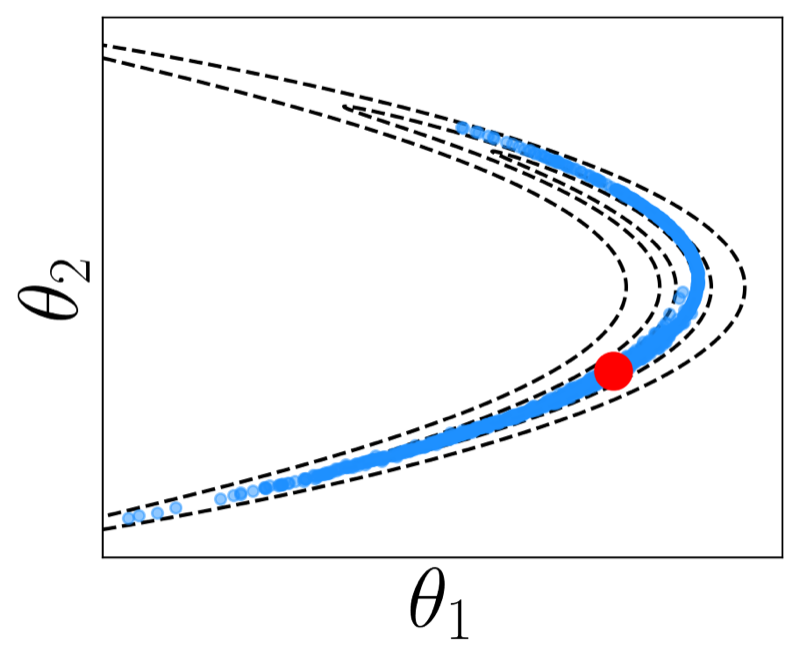} &
        \includegraphics[width=0.2\textwidth]{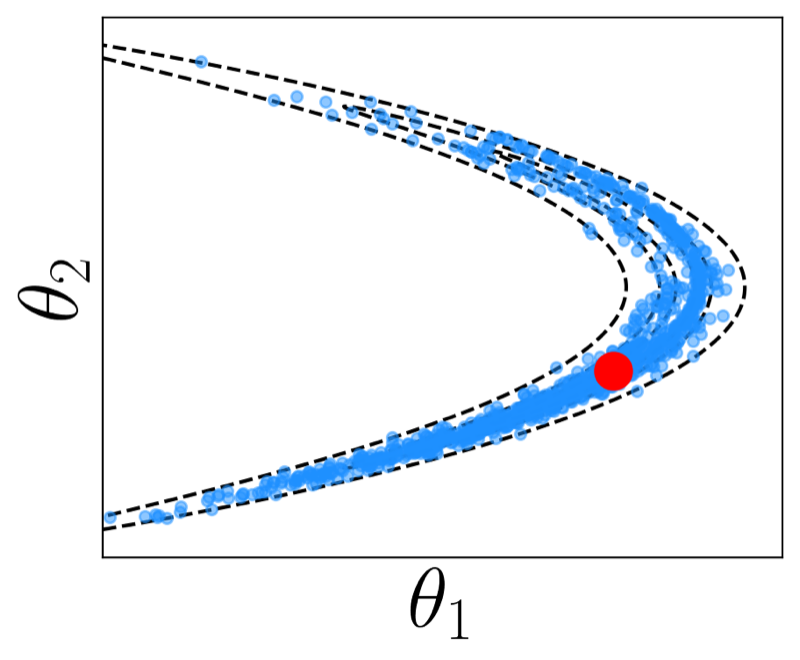} & \includegraphics[width=0.2\textwidth]{figs/h_banana_fisher.png} \\
        \includegraphics[width=0.2\textwidth]{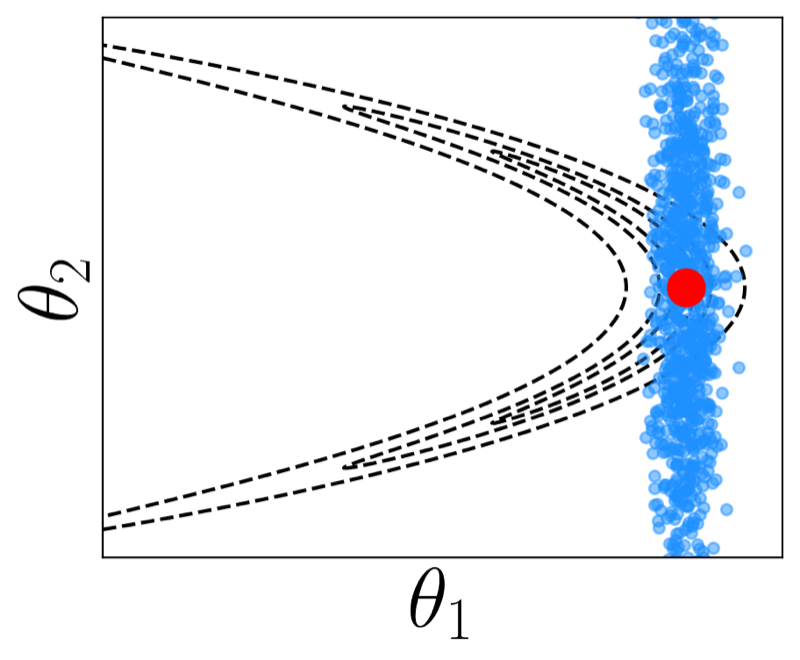} & \includegraphics[width=0.2\textwidth]{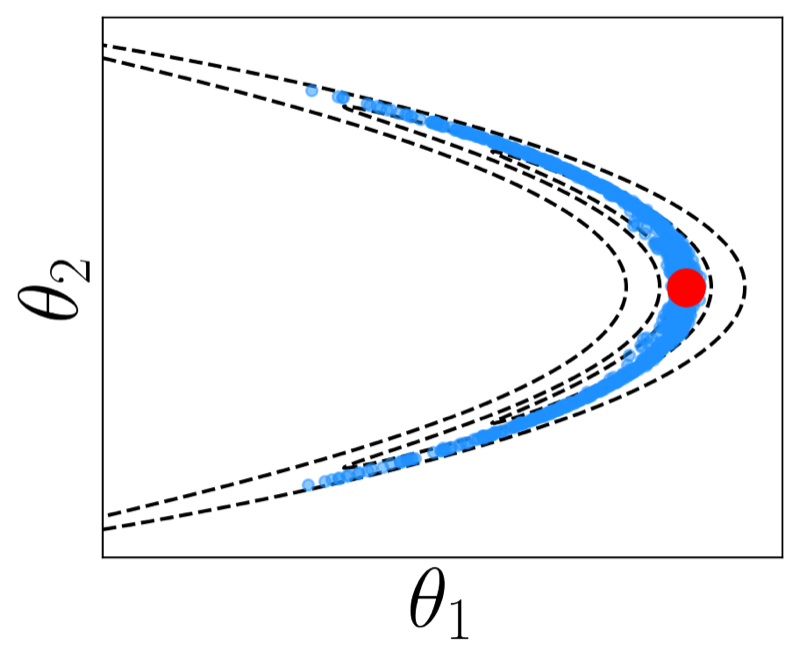} &
        \includegraphics[width=0.2\textwidth]{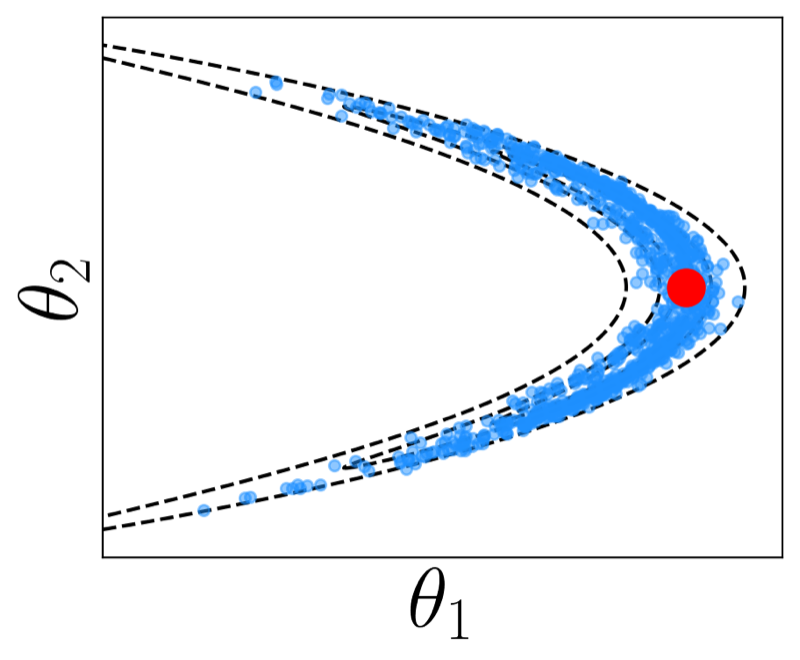} & \includegraphics[width=0.2\textwidth]{figs/f_banana_hausdorff_fisher.png}
    \end{tabular}
    
    \caption{The first row shows approximations starting from one of the Euclidean MAPs, and the second row shows approximations starting from the Hausdorff MAP. Each row from left to right: \euclaabbrv; \bergaminabbrv; \logbergaminabbrv; \ourabbrv}
    \label{fig:banana}
\end{figure}

Visually, \bergaminabbrv, \logbergaminabbrv \ and \ourabbrv \ all adapt to the local curvature. \bergaminabbrv \ is generally narrower, and is slightly in the wrong direction. \logbergaminabbrv \ alleviates these issues similar to \ourabbrv, and for all methods the Hausdorff MAP is clearly better as already indicated by the numerical results in the main paper.

\subsection{Squiggle}

In Section~\ref{sec:theoretical-basis}, \bergaminabbrv\ was demonstrated to struggle sampling from a squiggle distribution. Moreover, in the particular runs shown in Figure~\ref{fig:invariant-squiggle} which generated $10000$ samples respectively, \bergaminabbrv \ took on average $75$ function evaluations, while \ourabbrv \ only took $32$.

We note that its properties can vary. Figure~\ref{fig:difficult-squiggle} shows the results of \bergaminabbrv\ and \ourabbrv \ on another version of squiggle, given by
\begin{equation*}
a = 1.5,\; \bS = \begin{bmatrix}
    10.0 & 0.0 \\
    0.0 & 0.001
\end{bmatrix}.
\end{equation*}
The high density area is clearly narrower. While in this version of squiggle \bergaminabbrv \ seems to yield better results than the version as shown in the main paper, there still exists a non-negligible mismatch. Moreover, the resulting integration problem seems more difficult than using \ourinvabbrv: For this particular run, \ourinvabbrv \ only took $41$ function evaluations on average compared to $300$ for \bergaminabbrv; Moreover, there were at least one sample where \bergaminabbrv \ exceeded the maximum steps, while there were no such issues with \ourinvabbrv.

\begin{figure}[t]
    \centering
    \begin{tabular}{cc}
        \includegraphics[width=0.25\textwidth]{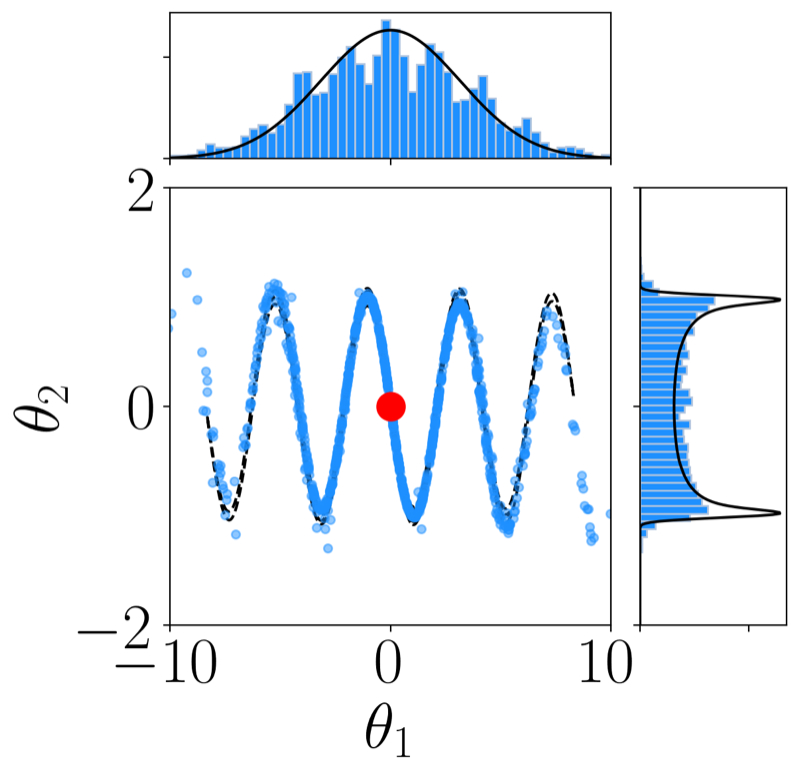} & 
        \includegraphics[width=0.25\textwidth]{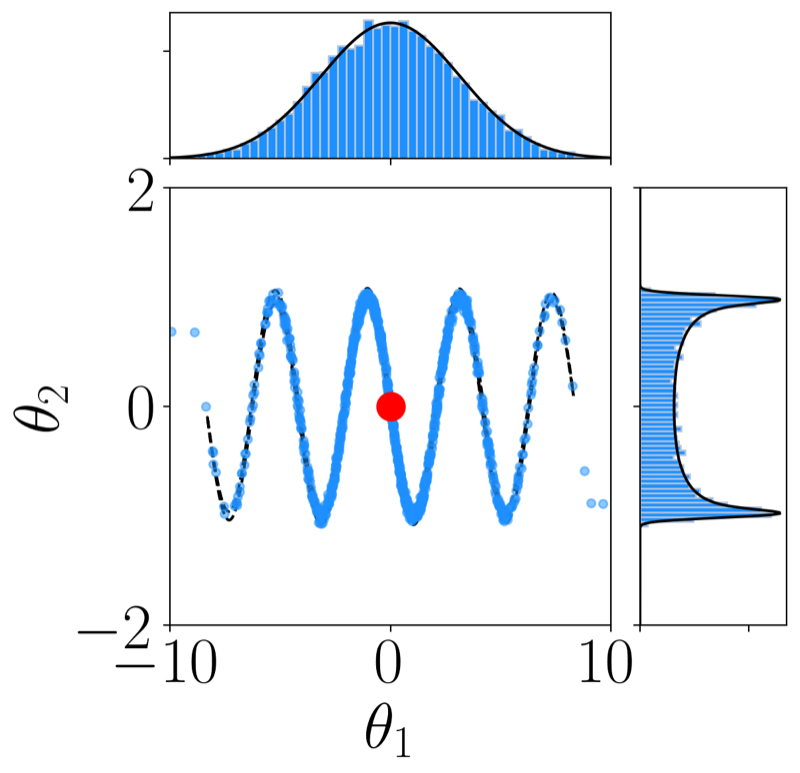}
    \end{tabular}
    \caption{Samples from \bergaminabbrv \ (left) and \ourabbrv \ (right) on another version of squiggle.}
    \label{fig:difficult-squiggle}
\end{figure}

\subsection{Funnel}

Here we demonstrate how the approximations work in one more case of a distribution of complex geometry that is diffeomorphic with a Gaussian. We consider the problem of sampling from the funnel distribution \citep{Neal2003}. It can be formulated as
\begin{equation*}
p(\btheta) = \N(\theta_{D} | 0, \sigma) \prod_{d=1}^{D-1}\N\left(\theta_{d}|0,\exp(\tfrac{1}{2}\theta_{D})\right).
\end{equation*}
Consider the two dimensional funnel. It admits a natural reparameterization, and can be formulated as a transformed version of multivariate Gaussian \citep{StanRef2023}, given by
\begin{align*}
\bpsi &\sim \N(\bzero , \bI),\\
\btheta &= \bphi(\bpsi)= \begin{bmatrix}
    \exp(\sigma \psi_{2} / 2) \psi_{1} \\
    \sigma \psi_{2}
\end{bmatrix}.
\end{align*}

As such, it is possible to obtain a Riemannian metric, following the general procedure in Section~\ref{sec:theoretical-basis} of the main paper. We have
\begin{align*}
\bphi^{-1}(\bzero) &= \bzero,\\
\left(\pdv{\btheta}{\bpsi}\right)^{-1} &= \begin{bmatrix}
    \exp(-\frac{1}{2}\theta_{2}) & -\frac{1}{2}\theta_{1}\exp(-\frac{1}{2}\theta_{2}) \\
    0 & \frac{1}{\sigma}
\end{bmatrix},\\
\bG_{\btheta} &= \left(\pdv{\btheta}{\bpsi}\right)^{-\top}\left(\pdv{\btheta}{\bpsi}\right)^{-1}.
\end{align*}

The samples obtained using \ourinvabbrv\  on a two dimensional funnel with $\sigma=3$ are shown in Figure~\ref{fig:funnel}.
Interestingly, all variants capture the marginal distribution of $\theta_2$ well even though many MCMC methods struggle to reach the narrow funnel. However, \euclaabbrv \ and \bergaminabbrv \  fail to expand in the direction of $\theta_1$ for larger values of $\theta_2$. \logbergaminabbrv \ alleviates the issue but is still not perfect, but \ourabbrv \ is exact.

\begin{figure}[t]
    \centering
    \begin{tabular}{cccc}
        \includegraphics[width=0.22\textwidth]{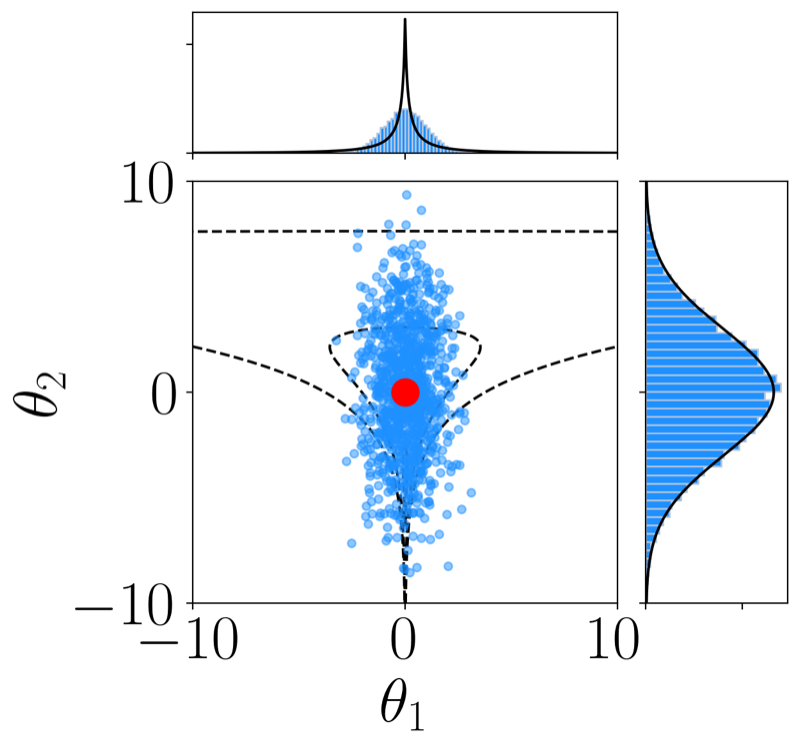} &
        \includegraphics[width=0.22\textwidth]{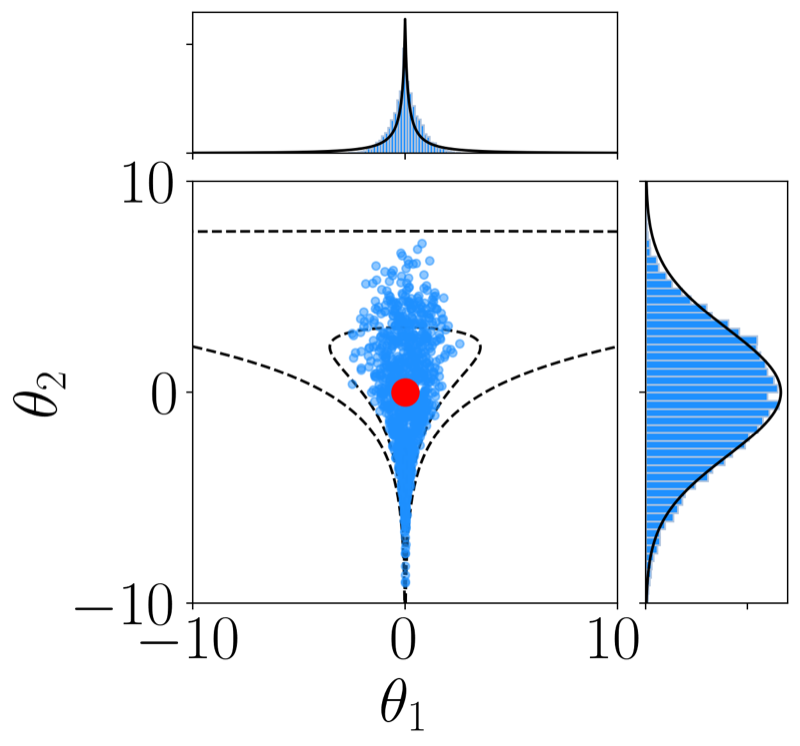} &
        \includegraphics[width=0.22\textwidth]{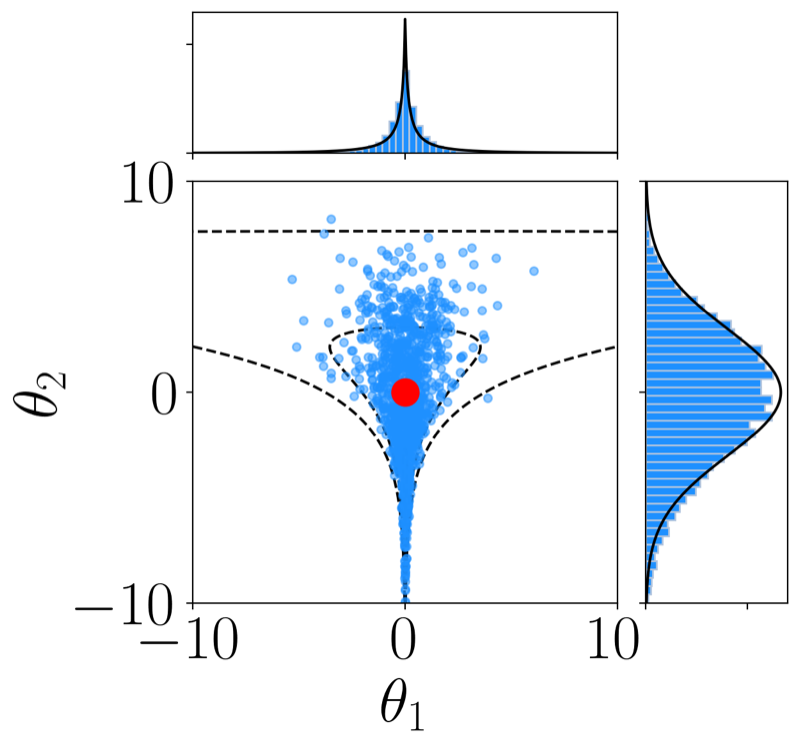} &
        \includegraphics[width=0.22\textwidth]{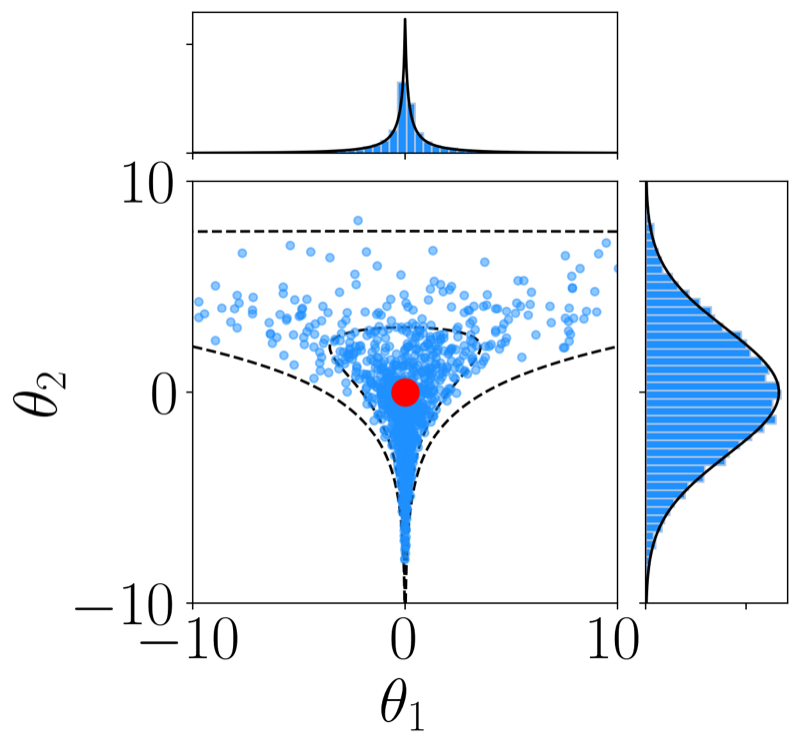}
    \end{tabular}
    
    \caption{Results on funnel. From left to right: \euclaabbrv; \bergaminabbrv; \logbergaminabbrv; \ourabbrv.}
    \label{fig:funnel}
\end{figure}

\subsection{Bayesian logistic regression}

In Section~\ref{sec:blr}, we considered Bayesian logistic regression. We observed that for \bergaminabbrv \ and \logbergaminabbrv, there can be huge differences in terms of number of function evaluations between standardized and raw inputs. We hypothesize that this is in part caused by the large differences in the scales of the raw inputs, and hence the gradient term in Monge metric, which is of quadratic nature, has very different scales across dimensions. This induces strong curvature and leads to challenging numerical integration problems.

We benchmark the time to obtain a single sample based on a sample from Gaussian distribution in seconds of different algorithms and report them in Table~\ref{tbl:lr-times}. We use a slightly optimized version of function for logarithmic map implemented in NumPy and functions for exponential maps are implemented in JAX. We use SciPy for integrations, averaging the obtained times over $2500$ samples. \ourabbrv\ consistently takes less time to obtain one sample, despite requiring direct matrix inversions.

\begin{table*}[b]
	\begin{center}
		\begin{tabular}{|l|l|l|l|l|l|l|}
			\hline
			\multicolumn{1}{|c|}{} & \multicolumn{3}{|c|}{stand.} & \multicolumn{3}{|c|}{raw} \\
			\hline
			data & \bergaminabbrv & \logbergaminabbrv & \ourabbrv & \bergaminabbrv & \logbergaminabbrv & \ourabbrv \\
			\hline
			Ripl & 0.011 & 0.048 & \textbf{0.006} & 0.011 & 0.045 & \textbf{0.006} \\
			\hline
			Pima & 0.016 & 0.26 & \textbf{0.006} & 1.989 & 1.183 & \textbf{0.006} \\
			\hline
			Hear & 0.015 & 0.638 & \textbf{0.007} & 2.33 & 2.683 & \textbf{0.006} \\
			\hline
			Aust & 0.02 & 0.734 & \textbf{0.01} & 6.183 & 4.632 & \textbf{0.008} \\
			\hline
			Germ & 0.026 & 2.215 & \textbf{0.014} & 1.435 & 3.4 & \textbf{0.015} \\
			\hline
		\end{tabular}
	\end{center}
	\caption{Time per sample in seconds for logistic regression. Bold font indicates the fastest method. Smaller is better.}
	\label{tbl:lr-times}
\end{table*}

\subsection{Neural networks}

In Section~\ref{sec:nn-regression}, we presented NN regression results. Here we present further visualizations. 
Figure~\ref{fig:nn-reg-ela} explicitly shows how ELA fails to generate meaningful samples for this problem.

\begin{figure*}[t]
    \centering
    \begin{tabular}{cc}
        \includegraphics[width=0.22\textwidth]{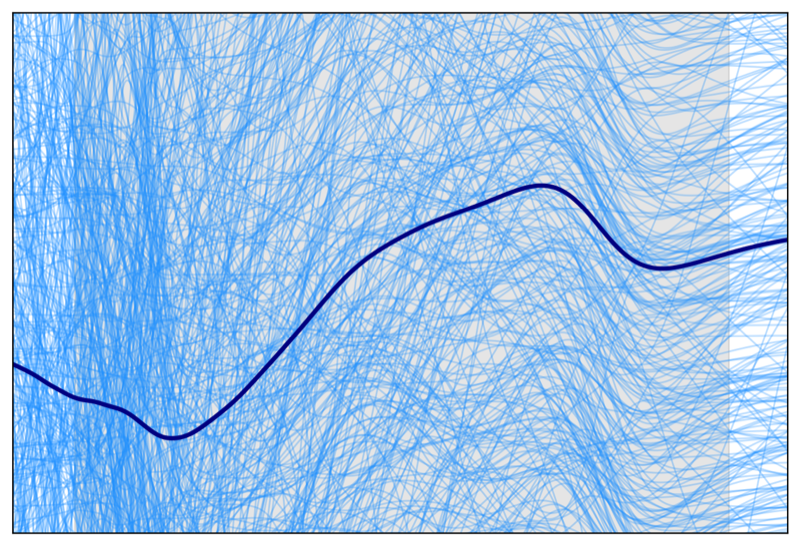} &
        \includegraphics[width=0.22\textwidth]{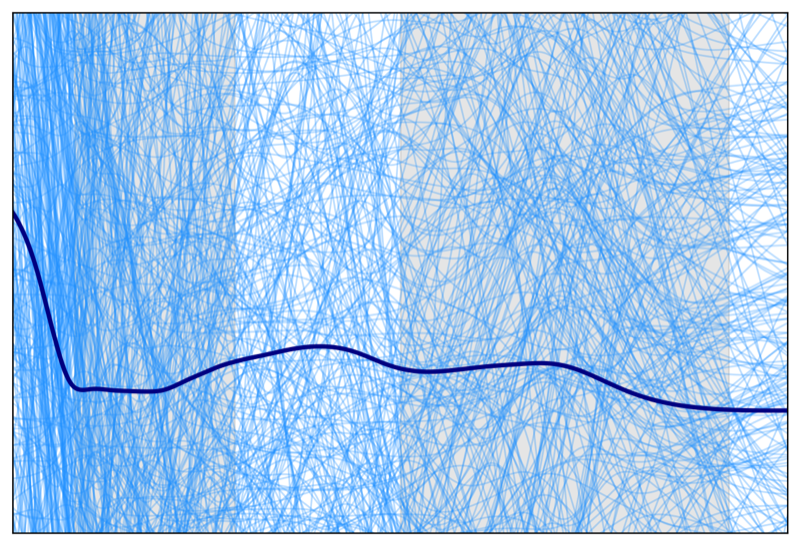}
    \end{tabular}
    
    \caption{Neural network regression with \euclaabbrv. Left: complete data; right: gap data.}
    \label{fig:nn-reg-ela}
\end{figure*}

\subsubsection{Numerical stability}

In Figure~\ref{fig:nn-reg}, we purposefully reported runs which resulted in relatively well-behaved predictive distributions for the methods under comparison. \ourabbrv\ is numerically stable but the other methods occasionally fail. This is demonstrated in Figures \ref{fig:failure-bergamin} and \ref{fig:failure-monge} that show (examples of) bad runs where both \bergaminabbrv \ and \logbergaminabbrv \ have extremely bad and isolated samples that influence even the mean prediction significantly. For completeness, Figure~\ref{fig:failure-fisher} shows the only run among the $10$ independent runs where \ourabbrv\ has any issues, in form of two minor outlier samples that do not have notable influence on the whole predictive distribution.

\begin{figure}[t]
    \centering
    \begin{tabular}{cc}
        \includegraphics[width=0.22\textwidth]{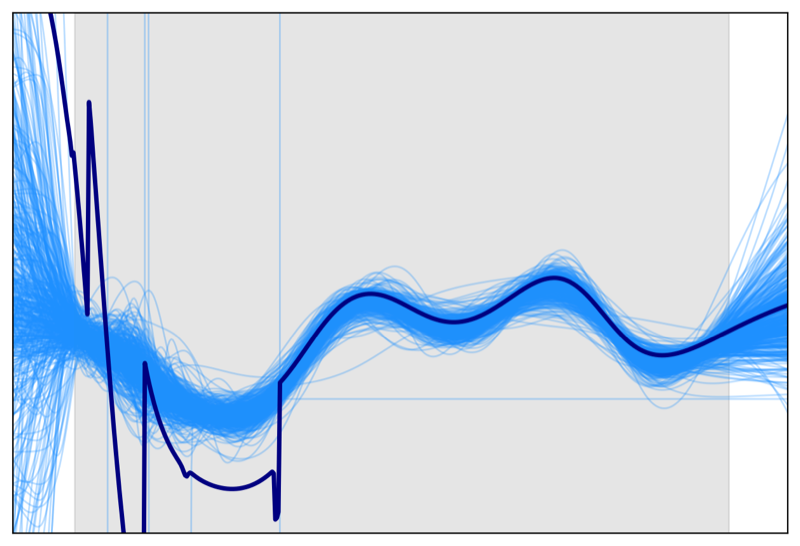} & 
        \includegraphics[width=0.22\textwidth]{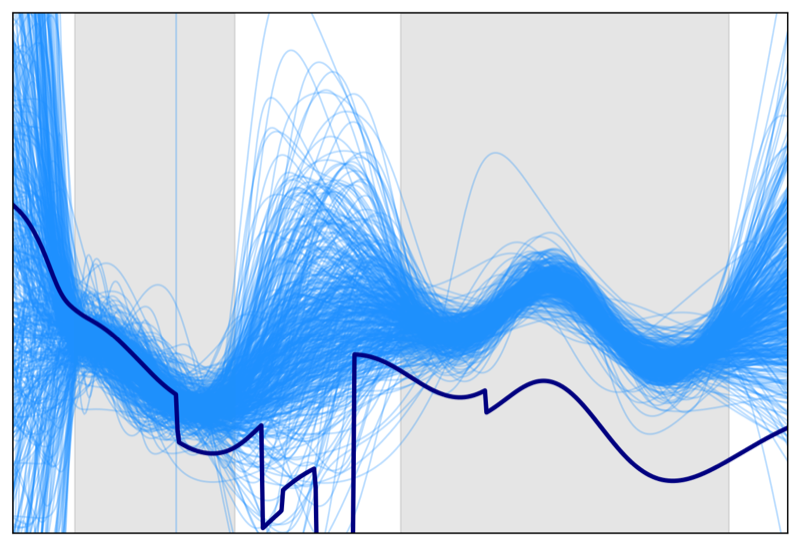}
    \end{tabular}
    
    \caption{Bad runs with \bergaminabbrv \ on neural networks. The left shows the complete case, and the right shows the gap case.}
    \label{fig:failure-bergamin}
\end{figure}
\begin{figure}[t]
    \centering
    \begin{tabular}{cc}
        \includegraphics[width=0.22\textwidth]{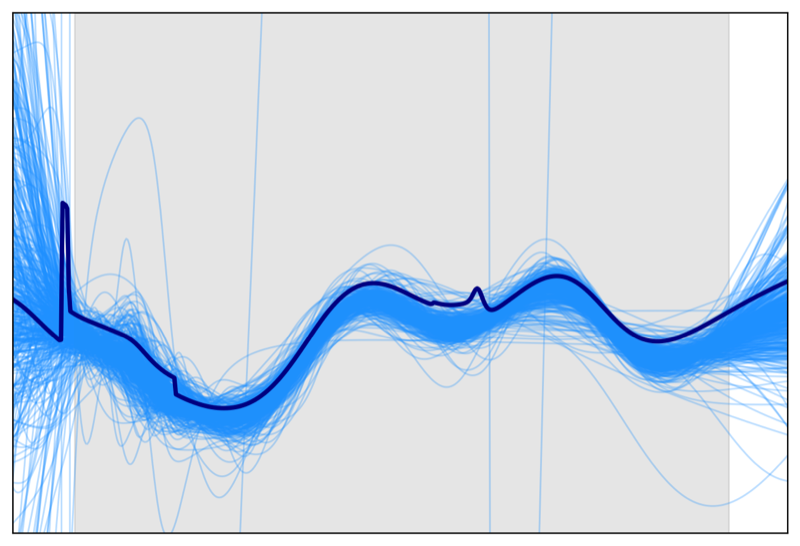} & 
        \includegraphics[width=0.22\textwidth]{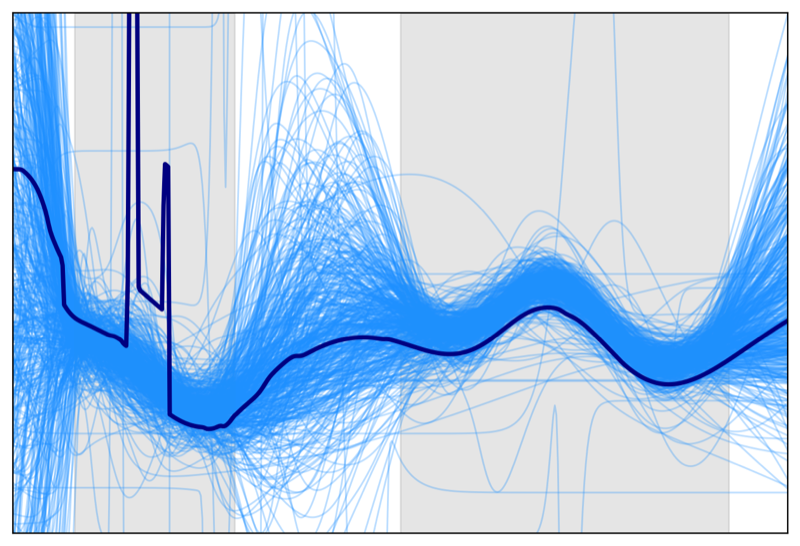}
    \end{tabular}
    
    \caption{Bad runs with \logbergaminabbrv \ on neural networks. The left shows the complete case, and the right shows the gap case.}
    \label{fig:failure-monge}
\end{figure}

\begin{figure}[t]
    \centering
    \includegraphics[width=0.22\textwidth]{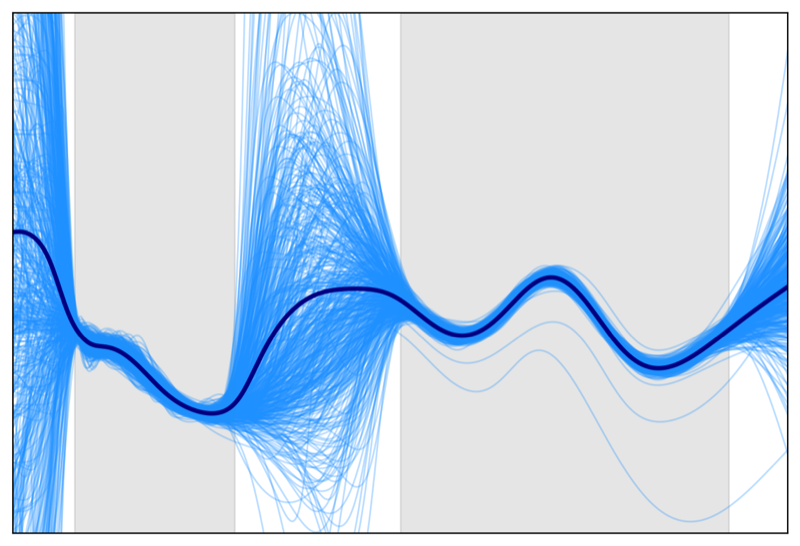}
    
    \caption{Bad run with \ourabbrv \ on neural networks, for the gap case.}
    \label{fig:failure-fisher}
\end{figure}

\subsubsection{Standardized data}

In the neural network regression experiment from the main paper, we directly run the neural network on the original dataset from \citet{Snelson2005}. We later observed that using standardized data may lead to reduced running times for all methods while reducing possible numerical issues, and we report the results in Table~\ref{tbl:nn-reg-std} and Figure~\ref{fig:nn-reg-std}. For finding the MAP estimate, we set $weight\_decay=1e-4$, with the other settings the same as for non-standardized data.

Interestingly, while in the main experiment \bergaminabbrv\ was demonstrated to lead to predictions wider than predictions based on NUTS samples, with standardized data it leads to narrower predictions. \logbergaminabbrv\ reduces the narrowness, while being noisier. \ourabbrv\ samples are similar to NUTS samples, while needing significantly fewer function evaluations and being faster in terms of running times.

\begin{table*}
	\begin{center}
		\begin{tabular}{|l|l|l|l|l|l|l|l|l|}
			\hline
			\multicolumn{1}{|c|}{} & \multicolumn{4}{|c|}{Complete} & \multicolumn{4}{|c|}{Gap} \\
			\hline
			method & MSE & NLL & $T$ & time & MSE & NLL & $T$ & time \\
			\hline
			\euclaabbrv & [0.718, 0.231] & [2.538, 0.041] & N/A & N/A & [1.497, 0.221] & [2.659, 0.032] & N/A & N/A \\
			\hline
			\bergaminabbrv & [0.072, 0.0] & [\textbf{0.11}, 0.001] & 2065.4 & 10.4 & [0.149, 0.011] & [\textbf{0.599}, 0.022] & 2109.5 & 10.6 \\
			\hline
			\logbergaminabbrv & [0.072, 0.0] & [0.121, 0.002] & 1428.1 & 10.2 & [0.107, 0.008] & [0.706, 0.028] & 1415.3 & 10.0 \\
			\hline
			\ourabbrv & [\textbf{0.071}, 0.0] & [0.115, 0.001] & \textbf{77.7} & \textbf{0.3} & [0.122, 0.003] & [0.805, 0.01] & \textbf{88.7} & \textbf{0.3} \\
			\hline
			NUTS & [0.073, 0.0] & [0.122, 0.001] & N/A & N/A & [\textbf{0.093}, 0.003] & [0.684, 0.017] & N/A & N/A \\
			\hline
		\end{tabular}
	\end{center}
	\caption{NN regression results with standardized data as $[\text{mean},\text{std}]$. $T$ indicates the average number of function evaluations for one sample while \textit{time} indicates the average time for one sample. Bold font indicates the best method. For all evaluation metrics smaller is better.}
	\label{tbl:nn-reg-std}
\end{table*}

\begin{figure*}[t]
    \centering
    \begin{tabular}{cccc}
        \includegraphics[width=0.22\textwidth]{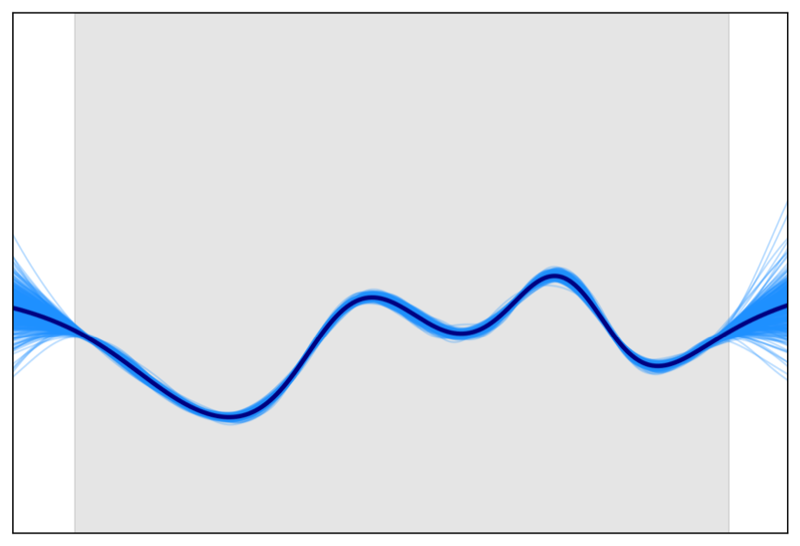} &
        \includegraphics[width=0.22\textwidth]{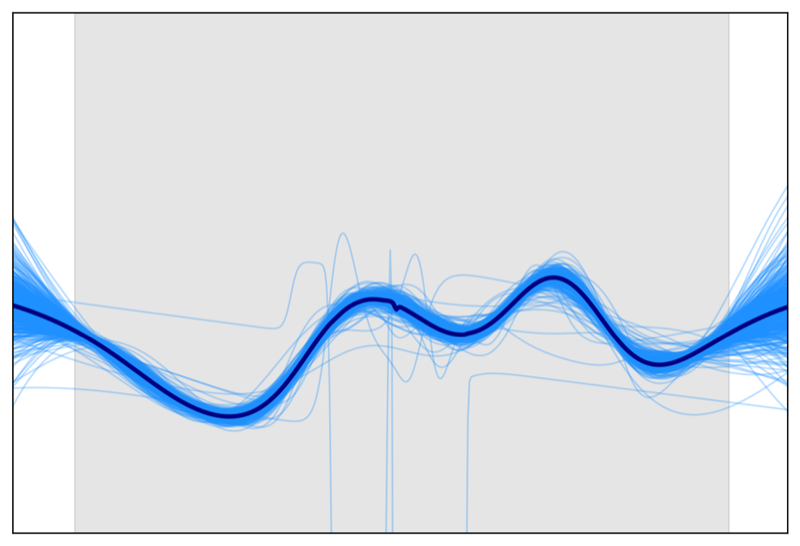} &
        \includegraphics[width=0.22\textwidth]{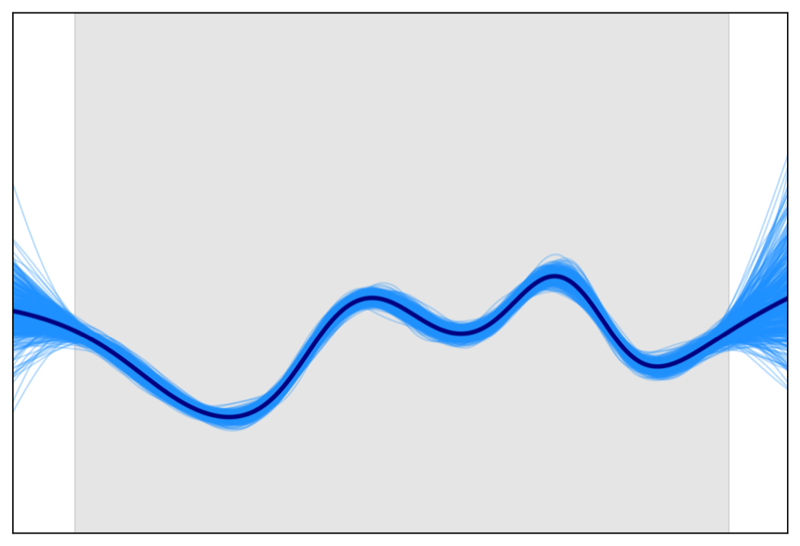} &
        \includegraphics[width=0.22\textwidth]{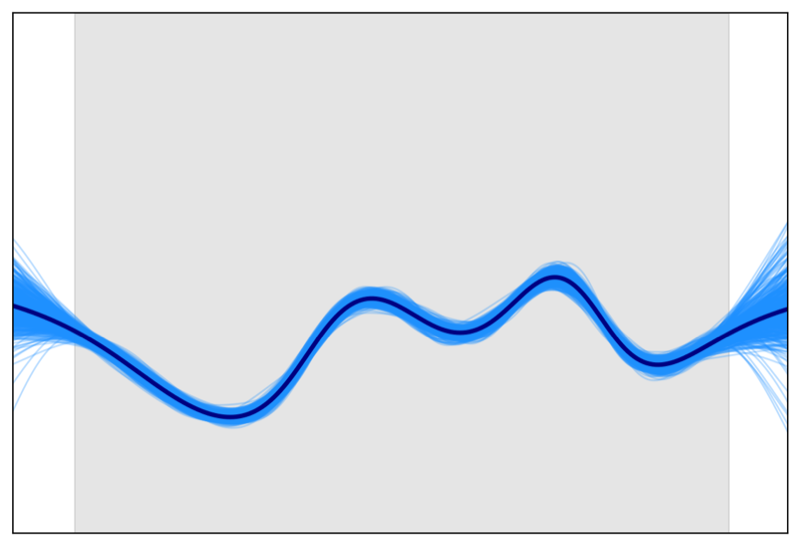} \\
        \includegraphics[width=0.22\textwidth]{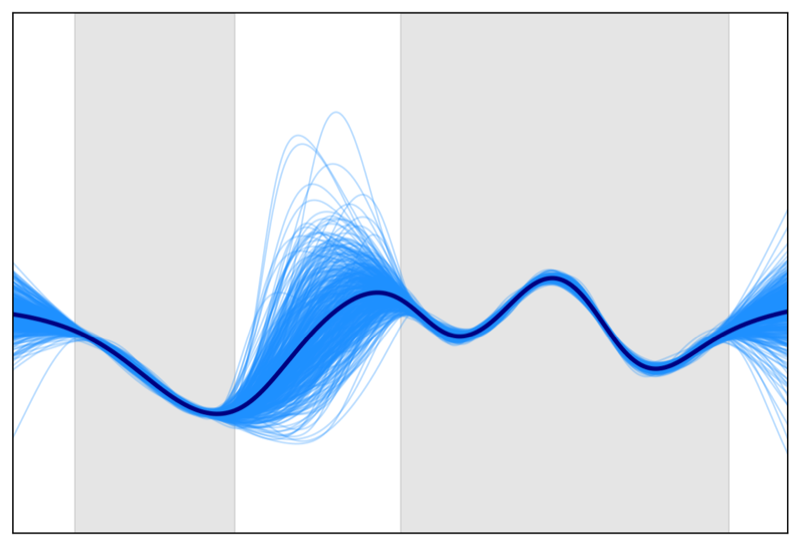} &
        \includegraphics[width=0.22\textwidth]{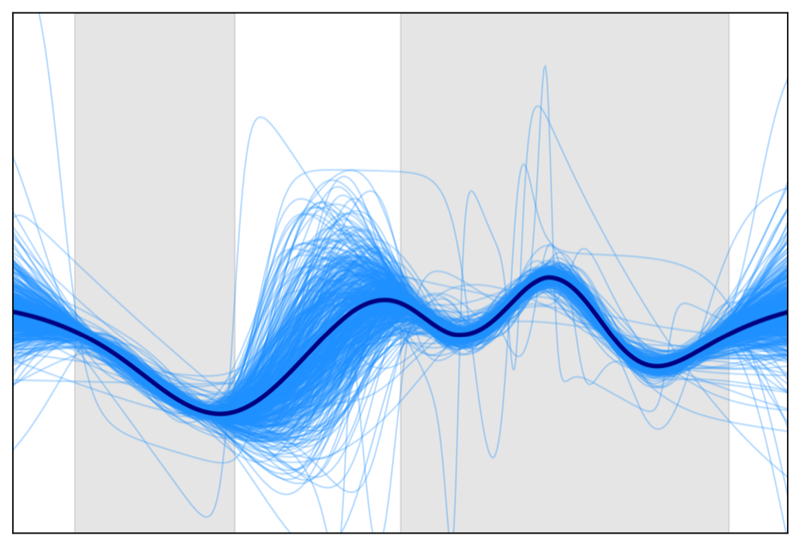} &
        \includegraphics[width=0.22\textwidth]{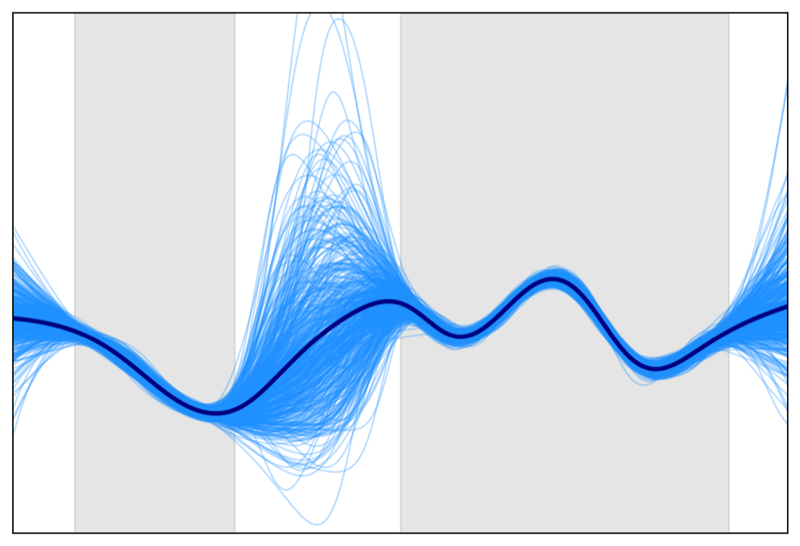} &
        \includegraphics[width=0.22\textwidth]{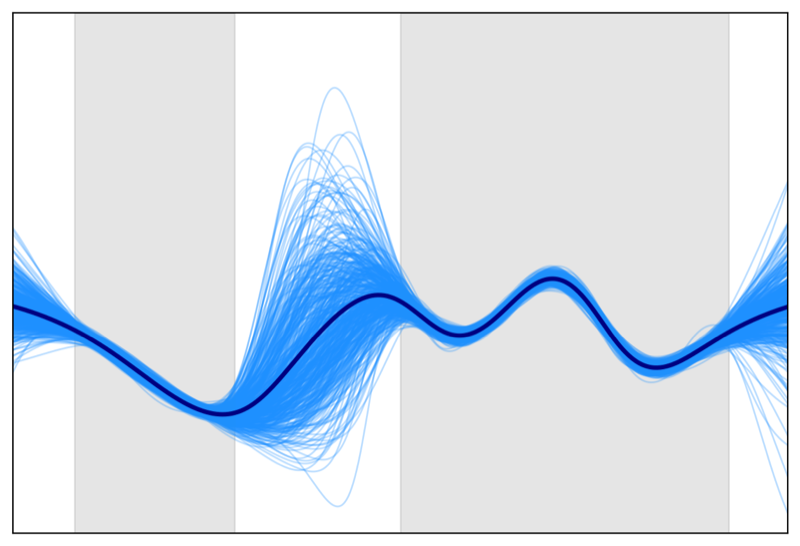}
    \end{tabular}
    
    \caption{NN regression with complete (top) and gap (bottom) standardized training data. Methods from left to right:
    \bergaminabbrv; \logbergaminabbrv; \ourabbrv; NUTS. Gray shading denotes the part of $x$-axis with training data, dark line is the mean prediction, and blue lines are samples.}
    \label{fig:nn-reg-std}
\end{figure*}

\subsubsection{Scalability experiment}
\label{sec:scalability}

We next consider a larger scale experiment similar to \citep{Bergamin2023}. We generate synthetic regression datasets with varying numbers of data points with the data generation mechanism based on \citet{Daxberger2021} before applying standardizations, and consider fully connected neural networks of size $[1, N, N, 1]$ with varying $N$. We find the MAP estimates using $lr=1e-3$, $weight\_decay=1e-4$ for $50000$ epochs with the Adam \citep{Kingma2015} optimizer. The prior precision and noise std are optimized post hoc following standard procedure from \citet{Daxberger2021}.

An example of the generated dataset and the average time to obtain one sample in seconds of \bergaminabbrv\ and \ourabbrv\ under different number of data points and different number of parameters are shown in Figure~\ref{fig:scalability}. Interestingly, \ourabbrv\ can be faster than \bergaminabbrv\ for $D$ up to $1000$. With large number of data points, \ourabbrv\ becomes slower due to needing per-sample gradients. However, this may be resolved by mini-batching.

\begin{figure}[ht]
    \centering
    \includegraphics[width=0.8\textwidth]{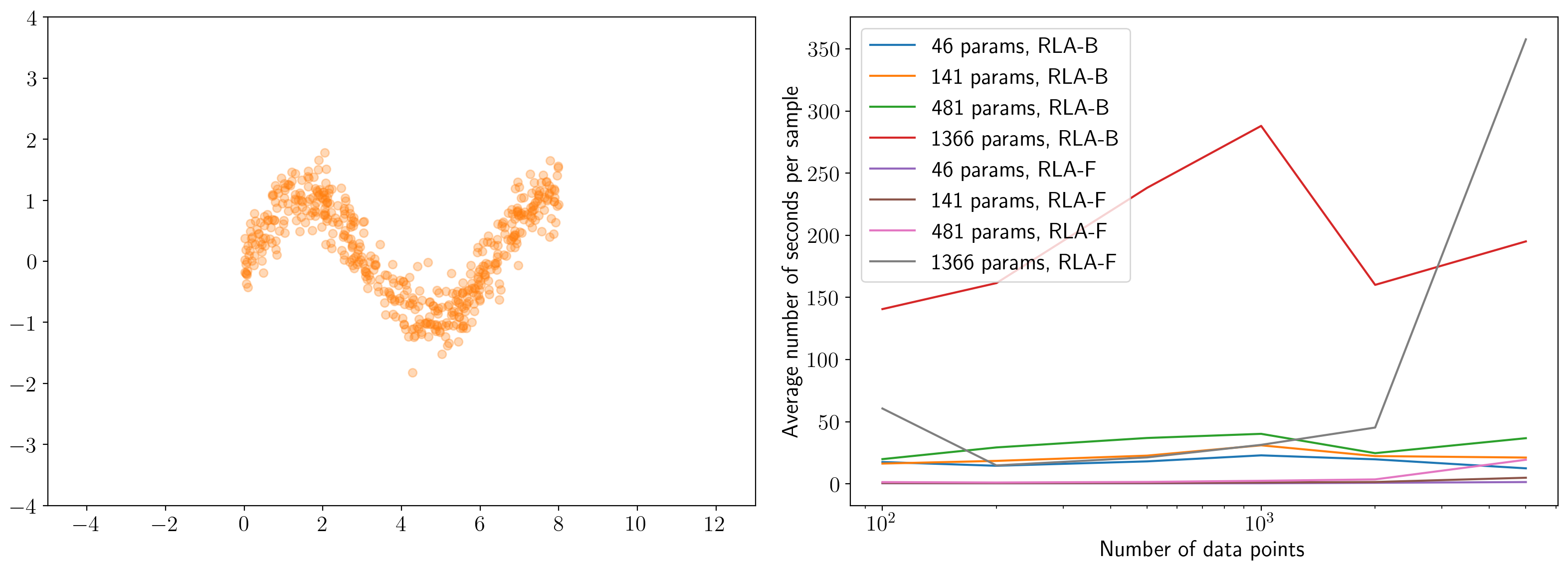}
    
    \caption{Scalability experiment. Left: An example dataset. Right: Running times of \bergaminabbrv\ and \ourabbrv\ under different settings.}
    \label{fig:scalability}
\end{figure}

\subsection{Empirical Fisher}

Some works have considered an alternative Riemannian metric known as the Empirical Fisher Information Matrix, given by
\begin{equation*}
\sum_{n=1}^{N}\nabla \log\pi(y_{n}|\btheta) \nabla \log\pi(y_{n}|\btheta)^{\top}.
\end{equation*}
There have been arguments that it may not lead to desirable behaviors \citep{Kunstner2019}. \citet{Lan2015} also considers another Empirical Fisher Information Matrix whose estimates are more accurate than the previous ones. We tried replacing the FIM with the above empirical version in the Fisher metric, terming the resulting metric as Empirical Fisher metric, and performed a range of preliminary experiments to evaluate whether the Empirical Fisher could also be used inside the metric in RLA, always using numerical Christoffel symbols. 

While it works reasonably well for sampling from banana distribution, yielding a Wasserstein distance of $[0.799,0.008]$ with Euclidean MAP and $[0.141, 0.003]$ with Hausdorff MAP (reported as $[\text{mean}, \text{std}]$), it becomes much worse for Bayesian logistic regression, as shown in Table~\ref{tbl:lr-ef}. Especially, for the experiment on Heart dataset without standardization, it is much worse than Euclidean. We therefore did not explore it further.
\begin{table*}[b]
	\begin{center}
		\begin{tabular}{|l|l|l|l|}
			\hline
			 & data & $\mathcal{W}$ & $T$ \\
			\hline
			\multirow{3}{1em}{\centering\rotatebox[origin=c]{90}{stand.}}&Ripl & [0.106, 0.004] & 12.0 \\
			\cline{2-4}
			&Pima & [0.148, 0.0] & 12.1 \\
			\cline{2-4}
			&Hear & [0.53, 0.0] & 13.8 \\
			\hline
			\multirow{3}{1em}{\centering\rotatebox[origin=c]{90}{raw}}&Ripl & [0.426, 0.011] & 12.1 \\
			\cline{2-4}
			&Pima & [0.199, 0.001] & 12.5 \\
			\cline{2-4}
			&Hear & [0.942, 0.02] & 17.0 \\
			\hline
		\end{tabular}
	\end{center}
	\caption{Logistic regression results with Empirical Fisher metric, reported as $[\text{mean}, \text{std}]$. $\mathcal{W}$ indicates Wasserstein distance to NUTS samples while $T$ indicates the average number of function evaluations for one sample. For all evaluation metrics smaller is better.}
	\label{tbl:lr-ef}
\end{table*}

\end{document}